\documentclass{article}
\PassOptionsToPackage{table}{xcolor}
\usepackage{iclr2027_conference,times}

\usepackage{amsmath,amsfonts,bm}

\newcommand{\stdcell}[2]{%
  \shortstack[c]{#1\\{\scriptsize $\pm\,#2$}}%
}

\newcommand{\beststdcell}[2]{%
  \shortstack[c]{\textbf{#1}\\{\scriptsize \textbf{$\pm\,#2$}}}%
}

\newcommand{\understdcell}[2]{%
  \shortstack[c]{\underline{#1}\\{\scriptsize $\pm\,#2$}}%
}

\def\eqref#1{equation~\ref{#1}}

\def\1{\bm{1}}

\DeclareMathAlphabet{\mathsfit}{\encodingdefault}{\sfdefault}{m}{sl}
\SetMathAlphabet{\mathsfit}{bold}{\encodingdefault}{\sfdefault}{bx}{n}

\usepackage[colorlinks=true, citecolor=blue, linkcolor=blue, urlcolor=blue]{hyperref}
\usepackage{url}
\usepackage{siunitx}
\DeclareSIUnit\angstrom{\text{Å}}
\usepackage{graphicx}
\usepackage{booktabs}
\usepackage{array,tabularx,longtable}
\usepackage{multirow}
\usepackage{bbm}
\usepackage{enumitem}
\usepackage{amsthm}
\usepackage[table]{xcolor}
\usepackage{mdframed}
\usepackage{wrapfig}
\usepackage{comment}
\usepackage{amssymb}
\DeclareMathOperator{\logit}{logit}

\theoremstyle{definition}

\newtheorem{assumption}{Assumption}

\theoremstyle{plain}
\newtheorem{proposition}{Proposition}

\newtheorem{corollary}{Corollary}

\mdfdefinestyle{statementbox}{
  linewidth=0.8pt,
  linecolor=gray!45,
  backgroundcolor=gray!6,
  roundcorner=3pt,
  innertopmargin=6pt,
  innerbottommargin=6pt,
  innerleftmargin=8pt,
  innerrightmargin=8pt,
  skipabove=8pt,
  skipbelow=8pt
}
\surroundwithmdframed[style=statementbox]{definition}
\surroundwithmdframed[style=statementbox]{assumption}
\surroundwithmdframed[style=statementbox]{proposition}
\surroundwithmdframed[style=statementbox]{remark}
\surroundwithmdframed[style=statementbox]{lemma}
\surroundwithmdframed[style=statementbox]{theorem}
\surroundwithmdframed[style=statementbox]{corollary}

\title{Contextualizing Biological Language Models across Modalities via Logit-Space Contrastive Alignment}

\author{%
\hspace{-0.27cm} Yanjun Shao$^{1}$,\hspace{0.2cm}
Yundi Chen$^{1}$,\hspace{0.2cm}
Yashvi Patel$^{1}$,\hspace{0.2cm}
\bf
Aurelien Pelissier$^{2,1,}$\thanks{A.P. and M.R.M. jointly supervised this work and contributed equally as senior authors. \\ \mbox{} \hspace{0.4cm} Correspondence to: \href{mailto.pelissier.38@gmail.com}{aurelien.pelissier.38@gmail.com}; \href{mailto.rodriguezmartinez@yale.edu}{maria.rodriguezmartinez@yale.edu}},\hspace{0.2cm}
Mar\'{i}a R. Mart\'{i}nez$^{1,}$\footnotemark[1]\\[.1922cm]
\small $^{1}$Biomedical Informatics and Data Science, Yale School of Medicine, United States\\
\small $^{2}$Institute of Computational Life Sciences, Zürich University of Applied Sciences (ZHAW), Switzerland\
}

\iclrfinalcopy 

\DeclareRobustCommand{\genmark}{\ensuremath{^{\dagger}}}
\DeclareRobustCommand{\structmark}{\ensuremath{^{\ddagger}}}
\newcommand{\methodcell}[2]{%
  \shortstack[l]{#1\\[-0.25ex]{\scriptsize\textcolor{gray!70!black}{#2}}}%
}

\begin{document}

\maketitle

\vspace{-0.3cm}

\begin{abstract}

\vspace{-0.3cm}

Pretrained biological language models encode rich probabilistic structure through native token distributions, but multimodal adaptation typically aligns hidden representations or replaces the native likelihood interface with task-specific heads. We introduce \textbf{\textsc{LogiCA}} (\emph{Logit-space Contrastive Alignment}), whose central methodological change is to align the \emph{context-conditioned native token distributions themselves} rather than pooled latent representations.
Using native-head-preserving cross-modal adapters, \textsc{LogiCA} couples models with \textbf{distinct vocabularies and tokenization schemes} without requiring a shared tokenizer, decoder, or embedding space. Contrastive learning operates directly through each model's original token-likelihood interface, enabling mutation-local ranking, cross-modal dependency mapping, and context-guided sequence design without task-specific heads.
Across protein--ligand, TCR--peptide, and drug-resistance tasks, \textsc{LogiCA} outperforms matched latent-contrastive and conditional-MLM baselines, particularly in low-homology, mutation-local, and data-scarce regimes. These results support native-logit alignment as a general strategy for making pretrained biological token distributions context-sensitive without sacrificing their probabilistic semantics.
Code is available at \url{https://github.com/AI-SysBio/LogiCA}.

\end{abstract}

\vspace{-0.4cm}

\begin{figure}[!htbp]
    \centering
    \includegraphics[width=.95\linewidth]{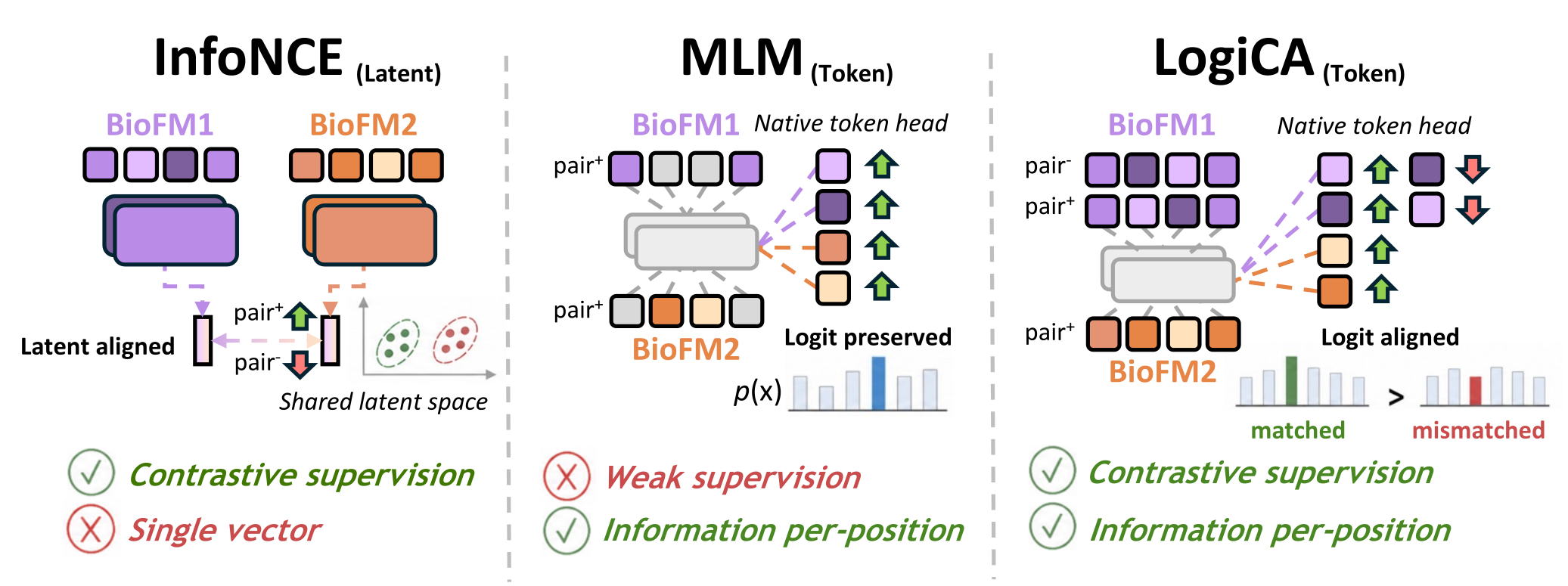}
    \label{fig:graphical_abstract}
\end{figure}

\vspace{-0.5cm}

\section{Introduction}
\label{sec:intro}

\vspace{-0.1cm}

Pretrained biological language models (BioFMs) have become foundational tools for
modeling proteins, genomes, and molecules because they define normalized token
distributions over biological vocabularies. These position-wise probabilities
support zero-shot mutation scoring and evolutionary analysis
~\citep{meier2021language,lin2023evolutionary}, as well as sequence design and
generation~\citep{madani2023large,hie2024efficient}. They also provide a natural
interface for token-level interpretation~\citep{zhang2024interactingmotifs}.
This position-level formalism is especially valuable in biology, where functional
changes often arise from small perturbations localized to individual residues,
nucleotides, or molecular tokens.

Many therapeutic prediction problems, however, are inherently contextual. A
T-cell receptor should be scored under a peptide, a ligand under a protein
target, and a resistance mutation under a drug. In such settings, the question is
not only whether a sequence is plausible, but whether it is plausible in a
specific biological context. Generating plausible
protein sequences is not sufficient for protein engineering: the central
challenge is to identify sequences with improved experimentally measured
function, as emphasized by recent functional-alignment studies
~\citep{blalock2026functional,widatalla2026proteindpo}. Here, we address
this challenge under explicit molecular-partner conditioning.
Standard pretrained language-model logits capture
broad evolutionary, structural, or chemical regularities, but are not trained to
directly encode compatibility with a particular binding partner, peptide, or
treatment condition.

Most existing approaches address contextual matching in one of two ways
(Appendix~\ref{app:related}). The first class leaves the language model's token
interface behind, learning scalar compatibility scores from pooled
representations, contrastive embedding alignment, or task-specific prediction
heads~\citep{huang2021moltrans,singh2023contrastive,jia2026drugclip}.
Such scores are effective for retrieval and binary matching,
but are no longer token likelihoods, and therefore do not naturally localize to
mutated residues, support likelihood-based generation, or expose the
position-wise probabilities that make pretrained biological language models
useful.
The second class preserves the token interface by fine-tuning with conditional
masked-language-modeling objectives~\citep{meynard2024tulip,chen2025pepmlm,
ullanat2026mint}. However, paired
biological datasets are often sparse, noisy, and context-specific;
reconstruction losses encourage token recovery but do not directly separate
matched from mismatched contexts. This is especially limiting when functional
signal is localized to small sequence perturbations: pooled objectives can
dilute the effect, while reconstruction losses may fail to distinguish the
correct context from plausible alternatives.

\begin{figure}[h]
    \centering
    \includegraphics[width=1\linewidth]{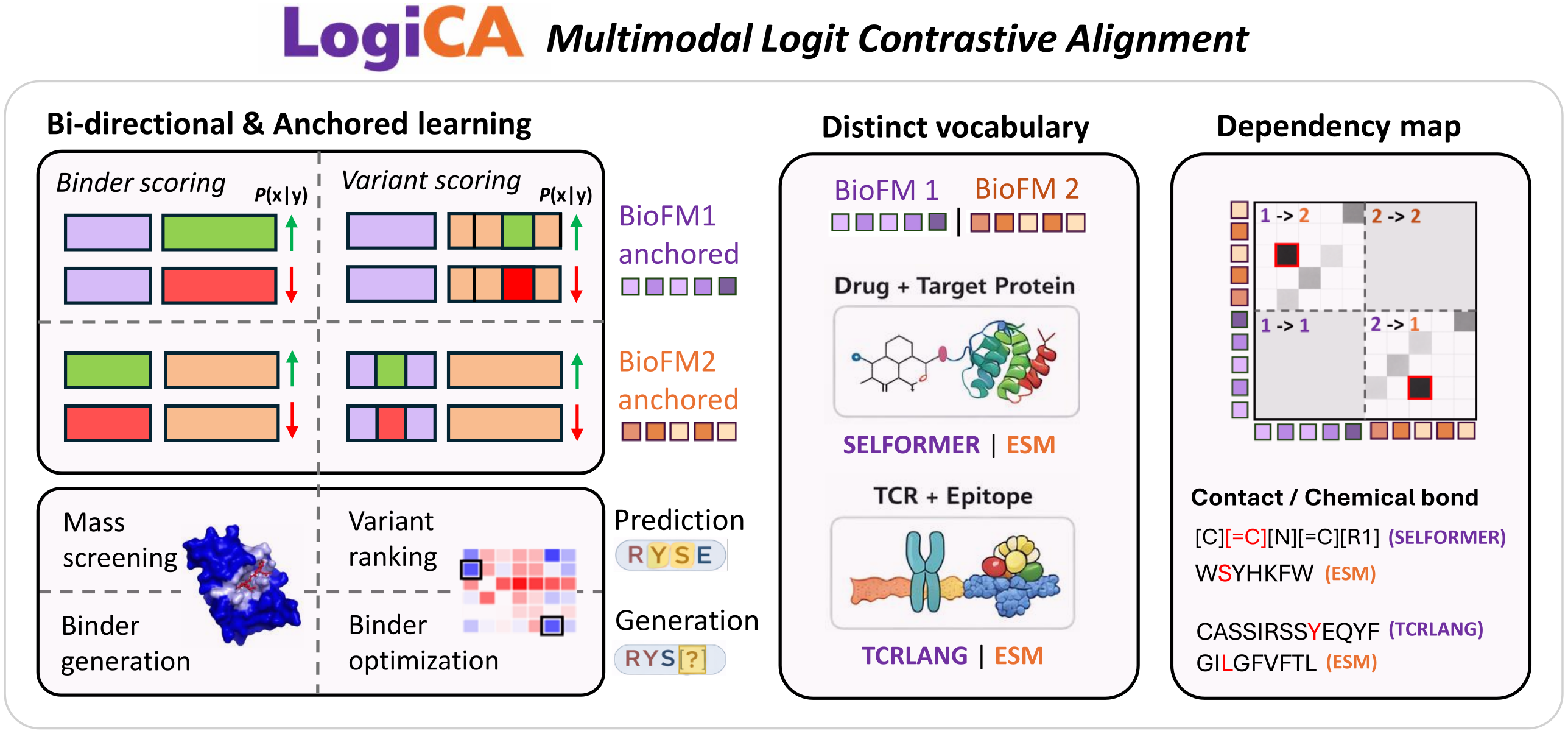}
    \caption{Overview of \textsc{LogiCA}. Native-head-preserving cross-modal adapters contrastively align context-conditioned token distributions across heterogeneous BioFMs.}
    \label{fig:logiCA}
\end{figure}

\textbf{Our contribution.} The central methodological contribution of \textbf{\textsc{LogiCA}} is a change in the \emph{object of multimodal alignment}: rather than aligning hidden representations, \textsc{LogiCA} aligns the context-conditioned native token distributions produced by pretrained BioFMs. Figure~\ref{fig:logiCA} summarizes how this principle is realized. Heterogeneous BioFMs with distinct vocabularies and tokenization schemes exchange information through cross-modal adapters, but each update is returned to the model's native hidden space and scored through its original pretrained token head. Contrastive supervision therefore acts on normalized token likelihoods in each model's own vocabulary, not on a shared latent embedding or task-specific classifier. 
This interface is especially useful for \textbf{mutation-local variant ranking}: when variants share a wild-type reference and mutation set, shared wild-type terms cancel exactly, reducing pairwise comparisons to context-conditioned mutant-token likelihoods at the perturbed sites. Empirically, we couple ESM-2~\citep{lin2023evolutionary}, TCRLang~\citep{raybould2024observed}, and SELFormer~\citep{yuksel2023selformer}. \textsc{LogiCA} generalizes beyond mutation scoring on a low-homology protein-family split, while its largest gains occur in finite-candidate and mutation-local settings where fine-grained token probabilities matter. Pair-specific contextualization makes it better suited to reranking and variant scoring than exhaustive first-stage retrieval.

\section{\textsc{LogiCA}: Logit-space Contrastive Alignment}
\label{sec:logica}

\subsection{Contextual ranking in logit space}
\label{sec:logica-ranking}

Let \(x=(x_1,\ldots,x_L)\) be a sequence to be scored, let \(y\) be an
external biological context, and let \(A\subseteq[L]\) denote the positions at
which the score is evaluated. A contextualized biological language model
defines, for each position \(i\), a token distribution
\[
  \pi_\theta(\cdot\mid x_{\setminus i},y)
\]
over the native vocabulary of \(x\). Our goal is to use these conditional token
probabilities not only for reconstruction, but as compatibility scores in a
contextual ranking problem.
Given an anchor \(a\), a candidate set \(\mathcal C\), and any scalar
compatibility score \(s(a,c)\), the induced ranking distribution is
\begin{equation}
  p(c\mid a,\mathcal C)
  =
  \frac{\exp(s(a,c)/\tau)}
       {\sum_{c'\in\mathcal C}\exp(s(a,c')/\tau)} .
  \label{eq:logica-ranking}
\end{equation}
The anchor and candidates may be either biological sequences or external
contexts, so this form covers both context retrieval and variant ranking. With
multiple candidates, Eq.~\ref{eq:logica-ranking} yields the InfoNCE objective
\citep{oord2018representation}; with two candidates, it reduces to the
Bradley--Terry preference loss
\citep{bradley1952rank,burges2005learning,rafailov2023direct}. Thus, the main
modeling choice is not the ranking loss itself, but the biological quantity
used to define \(s(a,c)\).
Latent-space contrastive models typically choose a pooled-representation
score, such as
\(s_z(x,y)=\langle f_\theta(x),g_\phi(y)\rangle\).
This CLIP-style objective can align modalities effectively
~\citep{radford2021learning}, but the resulting compatibility score is
detached from the model's original residue-level likelihoods.

\textsc{LogiCA} keeps the same contextual ranking template, but moves the
score into the language model's native logit space. Thus, the embeddings are
not trained to be directly comparable across modalities; they are trained
only insofar as they induce useful context-conditioned token probabilities
through the native output heads. We define compatibility by the site-averaged
conditional log-likelihood
\begin{equation}
  s_{\textsc{LogiCA}}(x,y;A)
  =
  \ell_A(x\mid y)
  =
  \frac{1}{|A|}
  \sum_{i\in A}
  \log \pi_\theta(x_i\mid x_{\setminus i},y).
  \label{eq:logica-score}
\end{equation}
For sequence-level matching we set \(A=[L]\). For localized tasks, \(A\) can
be restricted to mutation sites, binding-interface residues, or other
biologically meaningful subsets. Importantly, our variant-ranking experiments
require no functional-site annotations: mutation positions are determined
directly by comparing each candidate with its reference sequence. The same
scalar score can therefore be used inside Eq.~\ref{eq:logica-ranking}, while
remaining decomposable into position-level token probabilities.

\subsection{Mutation-local variant ranking}
\label{sec:logica-variant-ranking}

For variant comparison, the token-level probabilistic interface preserved by
\textsc{LogiCA} allows scores to focus on perturbed positions. Let
\(x^{\mathrm{wt}}\) be a reference sequence and
\(\mathcal M(x,x^{\mathrm{wt}})=\{i:x_i\neq x^{\mathrm{wt}}_i\}\) denote the
mutated sites. We consider position-aligned substitutions, excluding insertions
and deletions. We define the context-conditioned mutation score
\begin{equation}
  s_{\mathcal M}(x,y;x^{\mathrm{wt}})
  =
  \frac{1}{|\mathcal M|}
  \sum_{i\in\mathcal M}
  \log
  \frac{
    \pi_\theta(x_i\mid x_{\setminus i},y)
  }{
    \pi_\theta(x^{\mathrm{wt}}_i\mid x^{\mathrm{wt}}_{\setminus i},y)
  } .
  \label{eq:logica-mut-score}
\end{equation}
The score measures the average context-conditioned likelihood advantage of
substituted residues relative to wild type.
Reference normalization accounts for differences in baseline residue conservation across mutation sites,
while averaging prevents scores from scaling directly with mutation count.

\begin{proposition}[Mutation-local cancellation]
\label{prop:mutation-local-cancellation}
Let \(x^A\) and \(x^B\) be variants of the same wild-type sequence
\(x^{\mathrm{wt}}\) under context \(y\), with mutation sets
\(\mathcal M_A\) and \(\mathcal M_B\). On their nonempty union
\(\mathcal U=\mathcal M_A\cup\mathcal M_B\), define
\[
  s_{\mathcal U}(x,y;x^{\mathrm{wt}})
  =
  \frac{1}{|\mathcal U|}
  \sum_{i\in\mathcal U}
  \log
  \frac{
    \pi_\theta(x_i\mid x_{\setminus i},y)
  }{
    \pi_\theta(x_i^{\mathrm{wt}}\mid
    x_{\setminus i}^{\mathrm{wt}},y)
  } .
\]
Then the wild-type reference terms cancel exactly:
\begin{equation}
  \begin{aligned}
  \Delta_{\mathcal U}
  &=
  s_{\mathcal U}(x^A,y;x^{\mathrm{wt}})
  -
  s_{\mathcal U}(x^B,y;x^{\mathrm{wt}})\\
  &=
  \frac{1}{|\mathcal U|}
  \sum_{i\in\mathcal U}
  \log
  \frac{
    \pi_\theta(x_i^A\mid x^A_{\setminus i},y)
  }{
    \pi_\theta(x_i^B\mid x^B_{\setminus i},y)
  } .
  \end{aligned}
  \label{eq:logica-union-cancellation}
\end{equation}
For the corresponding Bradley--Terry model,
\[
  \Pr(x^A\succ x^B\mid y,\{x^A,x^B\})
  =
  \sigma(\Delta_{\mathcal U}/\tau).
\]
When \(\mathcal M_A=\mathcal M_B\), this reduces exactly to the
difference of the mutation scores in Eq.~\ref{eq:logica-mut-score}.
\end{proposition}

Here, mutation positions are determined from sequence differences and require
no functional-site annotations. Thus, mutation-local cancellation extends to variants with different
mutation sets when evaluated on their common union support. In practice,
\textsc{LogiCA} ranks variants using their individually normalized scores
in Eq.~\ref{eq:logica-mut-score}; for unequal mutation sets, these need not
equal the common-support comparison because mutations can affect likelihoods
at other positions. This formulation preserves contextual and epistatic
effects while isolating explicit mutation-site contributions, unlike pooled
latent similarities.
Proposition~\ref{prop:mutation-local-cancellation}
is an algebraic characterization of mutation-local scoring, rather than a
guarantee of predictive performance. It clarifies reference normalization
and contextual dependencies; the benefits for ranking and sequence design
are evaluated empirically in Section~\ref{sec:exp-dti}.
Appendix~\ref{app:proofs} provides the proof, gradient
analysis, and multi-site concentration bound, and Appendix~\ref{app:logica-efficiency} compares exact position-masked evaluation of Eq.~\ref{eq:logica-mut-score} against the single-pass wild-type marginal approximation.

\subsection{Bidirectional logit-space matching}
\label{sec:logica-bidirectional}

When both modalities have pretrained token heads, \textsc{LogiCA} can evaluate
compatibility in both conditional directions. This bidirectionality is natural
for pairwise biological interactions: compatibility is a property of the pair,
even though token likelihoods are directional. The score \(\ell_{A_x}(x\mid y)\)
asks whether context \(y\) makes the evaluated tokens of \(x\) likely, whereas
\(\ell_{A_y}(y\mid x)\) asks the reciprocal question. We therefore combine them with a mixture:
\begin{equation}
  s(x,y)
  = \frac{1}{2} \left[ \ell_{A_x}(x\mid y)
  +
  \ell_{A_y}(y\mid x)
  \right]
  \label{eq:logica-bidirectional-score}
\end{equation}
Here \(A_x\) and \(A_y\) denote the evaluated token positions in the two
modalities. Because these directional
likelihoods may differ in sequence length, vocabulary size, entropy scale, and
pretrained head calibration, we assessed unbalanced mixing strategies in Appendix~\ref{app:alpha-ablation} but it did not materially change performance.

Substituting \(s(x,y)\) into Eq.~\ref{eq:logica-ranking} yields the
\textsc{LogiCA} contrastive objective. Alignment is therefore performed
directly in logit space: positives and mismatched partners are separated by the
likelihoods assigned by each model's original token head, without requiring a
shared vocabulary, shared decoder, pooled-latent similarity, or separate
pair-classification head. At inference, pair-level binding prediction uses the unconditional-reference-adjusted bidirectional score $\widetilde{s}_{1/2}(x,y)$, whereas variant ranking uses the wild-type-anchored mutation score $s_{\mathcal M}(x,y;x^{\mathrm{wt}})$ (Appendix~\ref{sec:logica-reference-adjustment}).

\subsection{Native-head-preserving cross-modal adapters}
\label{sec:logica-architecture}

\textsc{LogiCA} contextualizes pretrained biological language models by
introducing native-head-preserving cross-modal adapters (Figure~\ref{fig:logica-architecture}). Cross-attention is a
standard mechanism for coupling pretrained encoders across modalities
\citep{tan2019lxmert,garau2024multi}; here, we use it to produce contextual
updates that remain compatible with each backbone's original token head.

Let \(H^x\) and \(H^y\) be token-level hidden states extracted from pretrained
encoders for the two modalities. The adapter projects both streams into a shared
interaction width, yielding \(Z^x_0\) and \(Z^y_0\), applies \(N\)
bidirectional cross-attention blocks to obtain \(Z^x_N\) and \(Z^y_N\), and
returns the accumulated update to each native hidden space through a gated
residual:
\begin{equation}
  H^m_c
  =
  H^m
  +
  \sigma(g_m)\,
  \phi_m\!\left(Z^m_N-Z^m_0\right),
  \qquad m\in\{x,y\}.
  \label{eq:logica-gated-residual}
\end{equation}
Here, \(Z^m_N-Z^m_0\) is the adapter update for modality \(m\), and \(\phi_m\) maps it back to the corresponding native hidden dimension. The protein--ligand instantiation uses the near-zero scalar gate shown above to keep \(H^m_c\) close to \(H^m\) at initialization; the TCR--peptide instantiation omits this scalar gate while retaining the same native-space residual principle. In both cases, \(H^m_c\) is scored by the original language-model head, preserving probabilities over the native vocabulary.

\subsection{Training with anchored negatives}
\label{sec:logica-sampling}

The ranking objective in Eq.~\ref{eq:logica-ranking} is defined by the training
candidate set, making pair construction a central design choice in
\textsc{LogiCA}. For each matched pair \((x,y)\), training constructs anchored
negatives by holding one modality fixed and replacing the other with corrupted,
mutated, or mismatched alternatives.
Alternating the anchor provides supervision for both \(\ell_{A_x}(x\mid y)\)
and \(\ell_{A_y}(y\mid x)\), supporting the bidirectional formulation in
Eq.~\ref{eq:logica-bidirectional-score}. Effective training uses a negative
pool that combines local single- or few-token perturbations, which emphasize
mutation-sensitive positions, with more distant negatives that preserve global
pair-level discrimination.

\subsection{Emergent capabilities of trained token logits}
\label{sec:logica-downstream}

The central consequence of native-logit alignment is a \emph{reusable,
context-conditioned probabilistic interface}, rather than a task-specific
scalar prediction head. Three capabilities arise as direct readouts of the
same contrastively trained token distributions: native mutation ranking, cross-modal
dependency mapping, and context-guided sequence generation. No separate
ranking head, attribution model, or generative decoder is required
(Appendix~\ref{app:logica-downstream}).

\paragraph{Cross-modal dependency maps.} To interpret cross-modal interactions, we probe how perturbing one token
changes the conditional distribution at another. For a context-token
substitution \(y_j\to a\), define
\begin{equation}
  D^{y\to x}_{ij}
  =
  \frac{1}{|\mathcal A^y_j|}
  \sum_{a\in\mathcal A^y_j}
  \left\|
    \pi_{\theta,i}(\cdot\mid x_{\setminus i},y^{(j\to a)})
    -
    \pi_{\theta,i}(\cdot\mid x_{\setminus i},y)
  \right\|_2 .
  \label{eq:logica-cross-dependency-main}
\end{equation}
Here, \(\mathcal A^y_j\) denotes valid substitutions at context position \(j\).
Large \(D^{y\to x}_{ij}\) identifies context tokens that strongly influence
the distribution at scored position \(i\), yielding dependency maps without
contact-label supervision.

\paragraph{Context-conditioned Gibbs design.}
For a fixed context \(y\) and designable sites \(A\), we can also sample
directly from these native conditionals. Starting from \(x^{(0)}\), select
\(i\in A\) and update
\begin{equation}
  x_i^{(t+1)} \sim
  \pi_\theta\!\left(\cdot\mid x^{(t)}_{\setminus i},y\right),
  \qquad x_j^{(t+1)}=x_j^{(t)}\ \text{for }j\ne i.
  \label{eq:logica-gibbs-main}
\end{equation}
Because independently learned masked-token conditionals need not define a
consistent joint distribution, this is a \emph{pseudo-Gibbs} update, not a
claim of exact joint sampling. Score-guided coordinate MCMC further evaluates
proposals with Eq.~\ref{eq:logica-mut-score}. Sampling protocols, biological
constraints, and retrospective DMS evaluations are detailed in
Appendix~\ref{app:logica-generation}. Neither generation nor dependency
localization is separately optimized during contrastive alignment.

\section{Experiments}

\subsection{Protein--ligand \textsc{LogiCA} for binding, variant ranking, and generation}
\label{sec:exp-dti}

We first assess whether native-logit alignment improves protein--ligand binding
prediction on unseen protein families, then test how the resulting model
transfers to mutation-level function. We subsequently examine sequence design
and contact localization as additional readouts of the same token probabilities,
before studying the role of functional supervision. Proteins are represented by
ESM-2~\citep{lin2023evolutionary} and ligands by SELFIES-encoded
SELFormer~\citep{yuksel2023selformer}, both retaining their native token heads. The score is evaluated over the
full sequence for binding and over the substituted sites for mutation ranking.
Unless otherwise specified, the principal binding and mutation-ranking
comparisons use ESM-2 150M; backbone scaling is examined separately.

\begin{wraptable}{r}{0.51\linewidth}

    \vspace{-0.75cm}

    \centering
    \caption{Low-homology BindingDB. MMseqs2 protein families (30\% identity, 80\% coverage) are disjoint across train/validation/test. Cells are five-seed mean $\pm$ s.d.; $\ddagger$ denotes structure/pocket input. \textbf{Bold}/\underline{underline}: best/second-best; *: $p<0.01$ vs.\ second-best by paired $t$-test.}
    \label{tab:dti-main}
    \small
    \renewcommand{\arraystretch}{1.06}
    \setlength{\tabcolsep}{7pt}
    \resizebox{0.51\textwidth}{!}{%
    \begin{tabular}{llcc}
    \toprule
    \textbf{Method} & \textbf{Input} & \textbf{AUROC} & \textbf{AUPR} \\
    \midrule
    \rowcolor{blue!6}
    \multicolumn{4}{l}{\textit{Matched contextualized backbones}} \\
    \midrule
    \methodcell{w/ \textsc{LogiMLM}\genmark (150M)}{logit-level MLM} & sequence
      & \stdcell{0.547}{0.004} & \stdcell{0.531}{0.005} \\
    \methodcell{w/ LatentCA (150M)}{latent contrastive} & sequence
      & \stdcell{0.604}{0.005} & \stdcell{0.580}{0.004} \\
    \rowcolor{blue!10}
    \methodcell{\textbf{w/ \textsc{LogiCA}\genmark (150M)}}{\textbf{native-logit contrastive}} & sequence
      & \beststdcell{0.630*}{0.007} & \beststdcell{0.597}{0.009} \\
    \midrule
    \rowcolor{gray!8}
    \multicolumn{4}{l}{\textit{Sequence-only external baselines}} \\
    \midrule
    \methodcell{MolTrans}{end-to-end classifier} & sequence & \stdcell{0.602}{0.005} & \stdcell{0.577}{0.008} \\
    \methodcell{DrugBAN}{end-to-end classifier} & sequence & \stdcell{0.599}{0.008} & \stdcell{0.573}{0.010} \\
    \methodcell{ConPLex}{latent contrastive} & sequence & \stdcell{0.593}{0.008} & \stdcell{0.579}{0.008} \\
    \midrule
    \rowcolor{orange!8}
    \multicolumn{4}{l}{\textit{Structure-informed external baselines}} \\
    \midrule
    \rowcolor{orange!4}
    \methodcell{3DProtDTA\structmark}{end-to-end classifier} & structure & \stdcell{0.614}{0.010} & \understdcell{0.594}{0.012} \\
    \rowcolor{orange!4}
    \methodcell{GS-DTI\structmark}{latent contrastive} & structure & \stdcell{0.603}{0.012} & \stdcell{0.585}{0.013} \\
    \rowcolor{orange!4}
    \methodcell{DrugCLIP\structmark}{latent contrastive} & pocket & \understdcell{0.621}{0.004} & \beststdcell{0.597}{0.003} \\
    \bottomrule
    \end{tabular}%
    }
\vspace{-0.3cm}
\end{wraptable}
\renewcommand{\arraystretch}{1.0}


\paragraph{Low-homology binding prediction.}
We begin with sequence-level binder matching under a protein-family shift.
BindingDB proteins are clustered with MMseqs2 at 30\% sequence identity and
80\% coverage, with complete clusters assigned to disjoint training,
validation, and test sets. Importantly, all proteins from the held-out clusters,
and every BindingDB interaction containing them, are excluded \emph{before}
contextual pretraining (Appendix~\ref{app:bindingdb}). The evaluation therefore
tests generalization to protein families unseen during both contextual
pretraining and downstream fitting. 
Two matched controls distinguish the effects of the training objective and
scoring interface. \textsc{LogiMLM} preserves the contextualized native heads but
learns from positive-pair reconstruction; LatentCA uses contrastive training
with a pooled-latent similarity score. Sequence-only \textsc{LogiCA} obtains an
AUROC of $0.630\pm0.007$, compared with $0.604\pm0.005$ for LatentCA and
$0.547\pm0.004$ for \textsc{LogiMLM}, and matches the highest AUPR among the
reported methods (Table~\ref{tab:dti-main}). Thus, the benefit of native-logit
contrastive alignment is not confined to mutation-local scoring. Conventional
DTI benchmarks and backbone-size results are reported in
Appendix~\ref{app:standard-dti}; external-baseline protocols are detailed in
Appendix~\ref{app:external-baselines}.

\paragraph{Zero-shot mutation ranking and functional adaptation.}
We next ask whether binding-pretrained token probabilities can rank functional
mutations without DMS labels. We evaluate 103 mutational DMS cohorts spanning
42 genes and 32 homology-disjoint protein clusters (Appendix~\ref{app:drug-resistance}). The evaluation
separately holds out protein clusters (LOCO-30\%), ligand scaffolds (LOLO), or
oncology resistance targets (Oncology LOO; Table~\ref{tab:variant-drug}).
The BindingDB-pretrained \textsc{LogiCA} checkpoint transfers directly to
mutation-local ranking: without DMS supervision, it improves over unconditioned
ESM-2 on both LOCO-30\% and LOLO, reaching $3.38\times$ versus $2.73\times$
top-1\% enrichment on unseen protein clusters. The limitation of binding-only
transfer becomes apparent on oncology resistance screens, where zero-shot
enrichment is only $0.68\times$. Resistance mutations need not increase the
binding-compatible function learned from BindingDB; they can instead reduce
inhibitor sensitivity while preserving protein activity. Fine-tuning on
held-in DMS cohorts addresses this mismatch without using labels from the
held-out evaluation units. After adaptation, \textsc{LogiCA} reaches Spearman
$\rho=0.313$ on LOCO-30\%, $0.354$ on LOLO, and $0.276$ on Oncology LOO.
Matched ablations further distinguish scoring from supervision. SiteFuse uses
mutation-site latent representations rather than native token likelihoods and
is weaker on oncology resistance ranking ($\rho=0.118$ versus $0.276$).
Positive-only \textsc{LogiMLM} retains the native head but lacks contrastive
variant comparisons; its LOCO-30\% top-1\% enrichment is $2.58\times$, versus
$4.06\times$ for \textsc{LogiCA}. The ranking gain remains significant after
aggregating by gene ($\bar{\Delta}\rho_g=+0.0234$,
$p=6.98\times10^{-4}$; Appendix~\ref{app:scaling-regime}). Together, these
controls support complementary contributions from native-token scoring and
contrastive alignment, which we examine further under restricted training data
below. DrugCLIP nearly matches binding AUROC ($0.621$
versus $0.630$; Table~\ref{tab:dti-main}), yet its evaluated mutation-ranking
adaptation reaches only $1.04\times$ versus $4.06\times$ top-1\% enrichment
on unseen clusters (Table~\ref{tab:variant-drug}): screening performance
need not transfer to sequence engineering.

\begin{table}[!t]
    \centering
    \caption{\textbf{Protein--ligand mutation ranking, sequence design, and contact localization.} Zero-shot and unseen-LOO evaluation across 103 mutational DMS cohorts (LOCO-30\%; 42 genes, 32 clusters), 68 scaffold-disjoint cohorts (LOLO), and 29 oncology resistance cohorts (Oncology LOO). Metrics include Spearman $\rho\pm$ s.d., top-$1\%$ prioritization ($\mathrm{P@1\%}$), 12-step coordinate-search enrichment (Gibbs), and contact AUROC across 175 BioLiP structures. $\dagger$: native generative token head; $\ddagger$: 3D input; \texttt{--}: unavailable readout; \textbf{bold}/\underline{underline}: best/second-best; $*$/$^\diamond$: $p<0.01$/$p<0.05$ vs.\ \textsc{LogiMLM}. Additional details: Appendices~\ref{app:drug-resistance}, \ref{app:biolip-dependency}, and Table~\ref{tab:statistical-audit}.}
    \label{tab:variant-drug}
    \small
    \renewcommand{\arraystretch}{1.04}
    \setlength{\tabcolsep}{3.6pt}
\resizebox{\textwidth}{!}{%
\begin{tabular}{lccccccccc}
\toprule
\multirow{2}{*}{\textbf{Method}}
    & \multicolumn{3}{c}{\shortstack{\textbf{LOCO-30\% (Unseen Cluster)}\\ {\scriptsize 103 DMS ($32$ clusters, $42$ genes)}}}
    & \multicolumn{3}{c}{\shortstack{\textbf{LOLO (Unseen Ligand)}\\ {\scriptsize 68 multi-ligand DMS}}}
    & \multicolumn{2}{c}{\shortstack{\textbf{Oncology LOO}\\ {\scriptsize 29 resistance DMS}}}
    & \shortstack{\textbf{BioLiP Dep.}\\ {\scriptsize $175$ structs}} \\
\cmidrule(lr){2-4} \cmidrule(lr){5-7} \cmidrule(lr){8-9} \cmidrule(lr){10-10}
     & $\rho$ & $\mathrm{P@1\%}$ & $\mathrm{Gibbs}$ & $\rho$ & $\mathrm{P@1\%}$ & $\mathrm{Gibbs}$ & $\rho$ & $\mathrm{P@1\%}$ & $\mathrm{AUROC}$ \\
\midrule
\rowcolor{blue!4}
\multicolumn{10}{l}{\textit{Contextualized backbone (BindingDB pretrained only, zero-shot)}} \\
\midrule
\rowcolor{blue!5}
\methodcell{w/ \textsc{LogiCA}\genmark (150M, zero-shot)}{BindingDB pretrained only}
    & \stdcell{0.285}{0.172} & \underline{$3.38\times$} & $1.53\times$ & \stdcell{0.271}{0.168} & \underline{$3.22\times$} & $1.48\times$ & \stdcell{0.084}{0.071} & $0.68\times$ & \stdcell{0.585}{0.218} \\
\midrule
\rowcolor{blue!6}
\multicolumn{10}{l}{\textit{Contextualized backbone (unseen LOO fine-tuned)}} \\
\midrule
\methodcell{w/ LatentFuse (150M)}{concat embeddings + MLP}
    & \stdcell{0.089}{0.074} & $1.18\times$ & -- & \stdcell{0.094}{0.079} & $1.22\times$ & -- & \stdcell{0.059}{0.061} & $0.58\times$ & -- \\
\methodcell{w/ SiteFuse (150M)}{mutation-site latent + MLP}
    & \stdcell{0.289}{0.165} & $2.62\times$ & $1.17\times$ & \stdcell{0.282}{0.158} & $2.75\times$ & $1.19\times$ & \stdcell{0.118}{0.082} & $0.62\times$ & \stdcell{0.522}{0.198} \\
\methodcell{w/ \textsc{LogiMLM}\genmark (150M)}{logit-level MLM (pos.\ only)}
    & \understdcell{0.293}{0.171} & $2.58\times$ & \underline{$1.55\times$} & \understdcell{0.312}{0.162} & $2.90\times$ & \underline{$1.52\times$} & \understdcell{0.258}{0.060} & $0.44\times$ & \stdcell{0.553}{0.206} \\
\rowcolor{blue!10}
\methodcell{\textbf{w/ \textsc{LogiCA}\genmark (150M)}}{\textbf{native-logit contrastive}}
    & \beststdcell{0.313*}{0.164} & \textbf{4.06}$^\diamond\times$ & \textbf{1.85}$^*\times$ & \beststdcell{0.354*}{0.156} & \textbf{4.62}$^*\times$ & {\textbf{2.12}$^*\times$} & \beststdcell{0.276*}{0.062} & \textbf{1.42}$^*\times$ & \beststdcell{0.614*}{0.195} \\
\midrule
\rowcolor{gray!8}
\multicolumn{10}{l}{\textcolor{gray!70!black}{\textit{Unconditional sequence \& structure baselines (zero-shot)}}} \\
\midrule
\methodcell{ESM-1v}{masked LM ensemble}
    & \stdcell{0.258}{0.174} & $2.41\times$ & $1.38\times$ & \stdcell{0.244}{0.168} & $2.36\times$ & $1.34\times$ & \stdcell{0.076}{0.079} & $0.51\times$ & \stdcell{0.566}{0.219} \\
\methodcell{ESM-2\genmark (150M)}{masked LM}
    & \stdcell{0.273}{0.178} & $2.73\times$ & $1.43\times$ & \stdcell{0.259}{0.172} & $2.65\times$ & $1.39\times$ & \stdcell{0.052}{0.055} & $0.48\times$ & \understdcell{0.607}{0.225} \\
\methodcell{EVE}{MSA VAE}
    & \stdcell{0.264}{0.182} & $2.54\times$ & -- & \stdcell{0.251}{0.176} & $2.48\times$ & -- & \stdcell{0.114}{0.155} & \underline{$0.71\times$} & -- \\
\methodcell{Tranception}{retrieval AR LM}
    & \stdcell{0.261}{0.169} & $2.49\times$ & -- & \stdcell{0.248}{0.165} & $2.42\times$ & -- & \stdcell{0.041}{0.056} & $0.52\times$ & -- \\
\methodcell{Thermo / LigandMPNN\structmark}{3D inverse folding}
    & \stdcell{0.236}{0.162} & $2.18\times$ & $1.29\times$ & \stdcell{0.228}{0.159} & $2.14\times$ & $1.26\times$ & \stdcell{0.038}{0.064} & $0.49\times$ & \stdcell{0.557}{0.183} \\
\midrule
\rowcolor{gray!8}
\multicolumn{10}{l}{\textcolor{gray!70!black}{\textit{Contextualized DTI baselines (unseen LOO fine-tuned)}}} \\
\midrule
\methodcell{DrugBAN}{DTI classifier + MLP}
    & \stdcell{0.042}{0.058} & $1.08\times$ & -- & \stdcell{0.051}{0.062} & $1.12\times$ & -- & \stdcell{0.014}{0.018} & $0.52\times$ & -- \\
\methodcell{Boltz-2\structmark}{frozen 3D trunk + MLP}
    & \stdcell{0.038}{0.051} & $1.06\times$ & -- & \stdcell{0.044}{0.055} & $1.09\times$ & -- & \stdcell{0.015}{0.033} & $0.51\times$ & -- \\
\methodcell{DrugCLIP\structmark}{DTI contrastive + MLP}
    & \stdcell{0.031}{0.049} & $1.04\times$ & -- & \stdcell{0.039}{0.053} & $1.07\times$ & -- & \stdcell{0.001}{0.034} & $0.50\times$ & -- \\
\bottomrule
\end{tabular}%
}
\vspace{-0.3cm}
\end{table}
\renewcommand{\arraystretch}{1.0}


\paragraph{Context-conditioned sequence design.}
Mutation ranking uses only the likelihoods assigned to observed substitutions.
The same native head also provides a distribution over all 20 amino acids,
allowing score-guided coordinate proposals without a separately trained
generative decoder. Across 103 mutational DMS cohorts (500 chains, 12 steps),
retrospective top-10\% enrichment increases from $1.43\times$ for unconditioned
ESM-2 to $1.53\times$ for zero-shot \textsc{LogiCA}, and reaches $1.85\times$
after LOCO-30\% adaptation and $2.12\times$ under LOLO (Table~\ref{tab:variant-drug}).
The resulting preferences also support multi-step combinatorial optimization: on the TtgR
combinatorial landscape~\citep{nishikawa2024allosteric}, 15-step search reaches
mean experimental fitness $+1.289\sigma$ ($1.92\times$ top-10\% enrichment), compared with $+0.674\sigma$ ($1.56\times$) for ESM-2 and $+0.184\sigma$ ($1.15\times$) for
SiteFuse. These retrospective experiments test whether the learned
probabilities support useful proposals, rather than claiming experimentally
validated de novo designs (Appendix~\ref{app:logica-generation-results};
Figure~\ref{fig:app-ttgr-gibbs}).

\paragraph{Cross-modal dependency maps.}
We also test whether changes in native token probabilities reveal where a
ligand influences its protein partner. Without contact-label supervision,
perturbation-derived dependency maps are evaluated against 175 held-out,
protein-disjoint BioLiP structures (44,337 residue observations).
At the 150M scale, contact AUROC improves from $0.585$ zero-shot to $0.614$
after DMS-LOO adaptation, compared with $0.553$ for \textsc{LogiMLM} and
$0.522$ for SiteFuse (Table~\ref{tab:variant-drug};
Appendix~\ref{app:biolip-dependency}). Thus, ranking, sequence design, and
contact localization use the same contextualized token interface rather than
separately optimized prediction heads.

\paragraph{Contrastive alignment under scarce DMS supervision.}
Finally, we examine when contrastive supervision contributes most strongly to
mutation ranking. We vary the amount of \emph{held-in} DMS training data while
keeping each evaluation cohort entirely unseen (zero target-cohort labels).
This is distinct from adapting to labeled variants of the test target. Even with
only $1\%$--$5\%$ of held-in supervision, \textsc{LogiCA} exceeds
\textsc{LogiMLM} trained on the full held-in dataset in both Spearman correlation
and top-1\% enrichment (Table~\ref{tab:contrastive-data-scale}).
The advantage changes with training scale (Figure~\ref{fig:fewshot-gap-convergence-wide}).
As more held-in data become available, differences in overall ranking and
classification narrow, whereas the gap in top-1\% enrichment persists or
increases. This pattern is consistent with the distinct objectives:
positive-only reconstruction can learn broad sequence preferences given
sufficient examples, while contrastive training directly compares functional
variants, including difficult upper-tail pairs. The main benefit of contrastive
alignment is therefore not simply improved average correlation, but better
prioritization of high-functioning candidates when experimental
screening capacity is limited (Appendix~\ref{app:scaling-regime}).
\begin{figure}[!h]
    \vspace{-0.15cm}
    \centering
    \includegraphics[width=\linewidth]{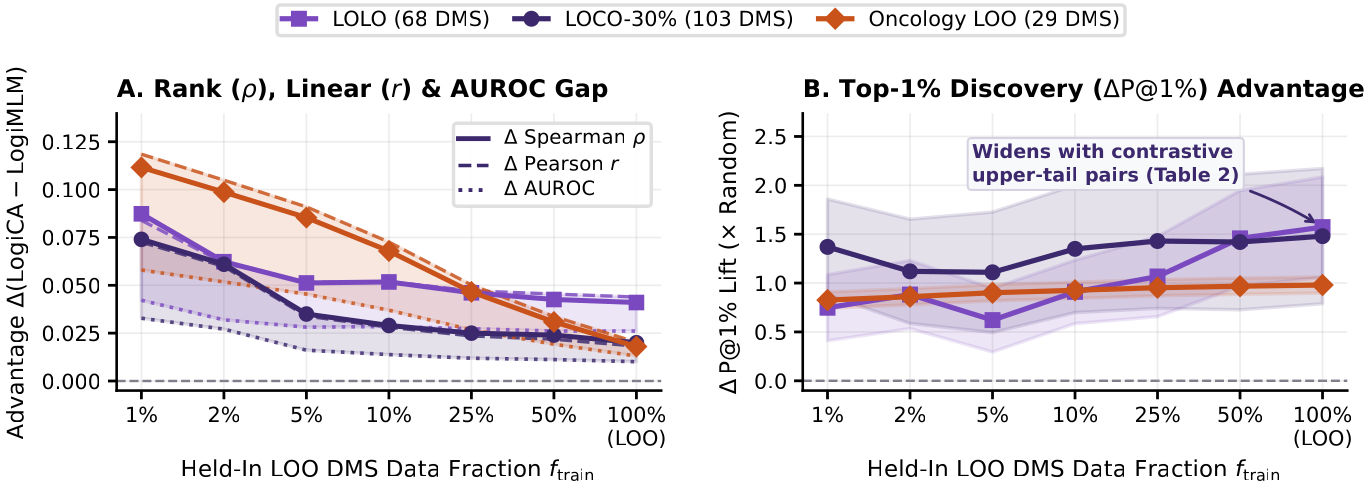}
    \vspace{-0.45cm}
    \caption{\textbf{Contrastive alignment preferentially improves top-hit prioritization under scarce DMS supervision.} Performance gap between \textsc{LogiCA} and positive-only \textsc{LogiMLM} as held-in training data ($f_{\mathrm{train}}$) increases from $1\%$ to $100\%$, with all evaluation cohorts held out ($0\%$ target labels). Results cover LOCO-30\% (103 DMS), LOLO (68 DMS), and Oncology LOO (29 DMS). \textbf{(A)} Ranking and classification gaps (Spearman $\rho$, Pearson $r$, AUROC) narrow with additional supervision. \textbf{(B)} Top-$1\%$ enrichment gains ($\Delta\mathrm{P@1\%}$) persist or increase, indicating sustained advantages in prioritizing high-functioning variants. Shaded bands: $\pm 1$ s.e.m.}
    \label{fig:fewshot-gap-convergence-wide}
    \vspace{-0.2cm}
\end{figure}

\subsection{TCR--peptide \textsc{LogiCA}: mutation ranking, contacts, and design}
\label{sec:exp-tcr-epitope}

Building on the protein--ligand results, we next evaluate
\textsc{LogiCA} on TCR--peptide recognition, a more
data-constrained and imbalanced setting where experimentally
verified nonbinding pairs are particularly scarce.
This provides a stringent test of whether contrastive native-logit
alignment can predict the effects of individual substitutions
without mutation-activity supervision.
We couple TCRLang (44.8M) for CDR3$\alpha$--CDR3$\beta$ sequences
~\citep{raybould2024observed} with ESM-2 35M for peptides
~\citep{lin2023evolutionary}.
Training uses approximately 260,000 curated binding pairs from
IEDB and paired-chain resources
~\citep{vita2019iedb,zhang2024epitope,kwee2023stapler}
in two stages: initial contrastive alignment against online
single-residue perturbations, followed by binding fine-tuning
with shuffled peptides, a standard negative-sampling strategy
~\citep{dens2023pitfalls}
(Appendix~\ref{app:tcr-twostage}).
These constructed negatives are not assumed to be experimentally
confirmed nonbinders, and downstream mutation-activity labels
are excluded throughout (Appendix~\ref{app:tcrbinding}).

\begin{table}[!h]

    \vspace{-0.2cm}
    
    \centering

    \caption{\textbf{Zero-shot TCR--peptide mutation ranking, sequence design, and contact localization.} Evaluation covers 71 peptide-mutation cohorts (26 ePytope, 35 BATCAVE, 10 ATLAS-PEP) and 23 TCR-mutation cohorts (21 affinity sets and 2 868-Z11 saturation DMS; 703 variants), without mutation-activity supervision. Metrics include Spearman $\rho\pm$ s.d., top-$10\%$ prioritization ($\mathrm{P@10\%}$), 12-step coordinate-search enrichment (Gibbs), and peptide--CDR3$\beta$ contact AUROC (251 TCR3d structures). $\dagger$: native generative token head; \texttt{-}: unavailable scalar/unidirectional readout; \textbf{bold}/\underline{underline}: best/second-best; *: $p<0.01$ vs.\ \textsc{LogiMLM}. Additional results: Tables~\ref{tab:mutation-affinity-correlation}, \ref{tab:dependency-map-interchain}, \ref{tab:tcr-top5}, and~\ref{tab:tcr-2x2-factorial}.}

    \label{tab:tcr-main-compact}
    \scriptsize
    \renewcommand{\arraystretch}{0.99}
    \setlength{\tabcolsep}{4.2pt}
    \resizebox{\textwidth}{!}{%
    \begin{tabular}{lccccccc}
    \toprule
    \multirow{2}{*}{\textbf{Method}}
      & \multicolumn{3}{c}{\shortstack{\textbf{Peptide Mutations}\\ {\scriptsize $71$ DMS ($9$ index peptides)}}}
      & \multicolumn{3}{c}{\shortstack{\textbf{TCR Mutations}\\ {\scriptsize $23$ DMS ($18$ index TCRs)}}}
      & \shortstack{\textbf{CDR3$\beta$-PEP Dep.}\\ {\scriptsize $251$ structs}} \\
    \cmidrule(lr){2-4} \cmidrule(lr){5-7} \cmidrule(lr){8-8}
      & $\rho$ & $\mathrm{P@10\%}$ & $\mathrm{Gibbs}$ & $\rho$ & $\mathrm{P@10\%}$ & $\mathrm{Gibbs}$ & $\mathrm{AUROC}$ \\
    \midrule
    \rowcolor{blue!6}
    \multicolumn{8}{l}{\textit{Matched contextualized backbone}} \\
    \methodcell{w/ \textsc{LogiMLM}\genmark}{logit-level MLM}
      & \stdcell{0.060}{0.301} & $1.82\times$ & $1.42\times$
      & \stdcell{0.044}{0.411} & $1.11\times$ & $0.89\times$
      & \stdcell{0.717}{0.153} \\
    \methodcell{w/ LatentCA-TCR}{latent score}
      & \stdcell{0.090}{0.264} & $1.10\times$ & --
      & \stdcell{-0.059}{0.410} & $1.23\times$ & --
      & -- \\
    \rowcolor{blue!10}
    \methodcell{\textbf{w/ \textsc{LogiCA}-TCR\genmark}}{\textbf{TCR-anchored}}
      & \beststdcell{0.271*}{0.267} & \textbf{2.62}$^*\times$ & $2.09\times$
      & \beststdcell{0.190*}{0.455} & \textbf{1.60}$\times$ & \underline{$1.48^*\times$}
      & \beststdcell{0.743*}{0.131} \\
    \rowcolor{blue!10}
    \methodcell{\textbf{w/ \textsc{LogiCA}-Pep\genmark}}{\textbf{peptide-anchored}}
      & \stdcell{0.179}{0.345} & $1.63\times$ & \textbf{4.51}$^*\times$
      & \stdcell{0.105}{0.409} & \underline{$1.35\times$} & \textbf{1.65}$^*\times$
      & \stdcell{0.712}{0.169} \\
    \rowcolor{blue!10}
    \methodcell{\textbf{w/ \textsc{LogiCA}-Dual\genmark}}{\textbf{dual-anchor}}
      & \understdcell{0.223*}{0.357} & \underline{$2.44^*\times$} & \underline{$4.41^*\times$}
      & \stdcell{0.059}{0.451} & \underline{$1.35\times$} & $1.45^*\times$
      & \understdcell{0.733}{0.166} \\
    \midrule
    \rowcolor{gray!8}
    \multicolumn{8}{l}{\textit{Selected external baselines}} \\
    TULIP-TCR\genmark
      & \stdcell{0.207}{0.249} & $2.44\times$ & $1.40\times$
      & \stdcell{0.100}{0.510} & $1.11\times$ & $1.31\times$
      & \stdcell{0.455}{0.189} \\
    TCR-T5\genmark
      & \stdcell{-0.030}{0.225} & $0.73\times$ & --
      & \stdcell{0.065}{0.355} & $0.86\times$ & $1.10\times$
      & \stdcell{0.648}{0.243} \\
    iTCep
      & \stdcell{0.203}{0.242} & $1.96\times$ & --
      & \stdcell{-0.019}{0.338} & $0.86\times$ & --
      & -- \\
    ImRex
      & \stdcell{0.161}{0.258} & $0.75\times$ & --
      & \understdcell{0.164}{0.388} & $0.86\times$ & --
      & -- \\
    EPACT
      & \stdcell{0.073}{0.233} & $0.93\times$ & --
      & \stdcell{0.100}{0.449} & $0.98\times$ & --
      & -- \\
    \bottomrule
    \end{tabular}%
    }
    \vspace{-0.2cm}
\end{table}
\renewcommand{\arraystretch}{1.0}


\begin{figure}[!t]
    \vspace{-0.15cm}
    \centering
    \includegraphics[width=1\linewidth]{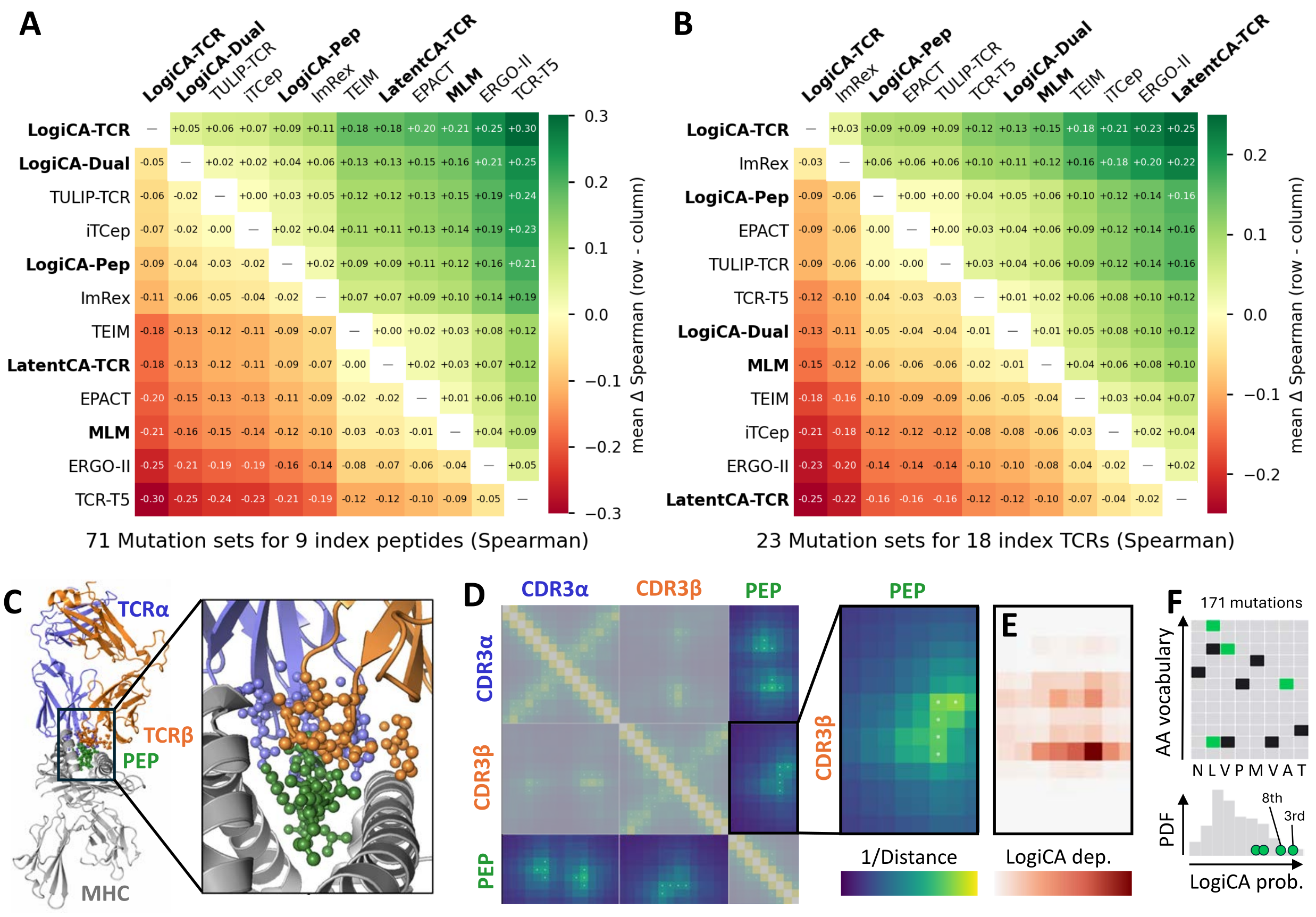}
    \vspace{-0.35cm}
    \caption{\textbf{Zero-shot TCR--peptide mutation ranking and cross-modal dependency localization.} \textbf{(A,B)} Pairwise Spearman correlation advantages of \textsc{LogiCA} over competing methods across peptide- and TCR-mutation benchmarks, respectively. \textbf{(C--E)} Representative TCR--peptide complex (PDB: 5TEZ), experimentally resolved contacts ($<5\,\text{\AA}$), and predicted peptide--CDR3$\beta$ dependencies derived from context-induced changes in native token probabilities, without contact supervision. \textbf{(F)} On the CMV epitope NLVPMVATV, four experimentally validated activity-enhancing substitutions rank 3rd, 8th, 29th, and 48th among 171 candidates~\citep{kula2019t}. Full benchmark results are reported in Table~\ref{tab:mutation-affinity-correlation}.}
    \label{fig:tcr-figure}
    \vspace{-0.2cm}
\end{figure}

\paragraph{Zero-shot peptide and TCR mutation ranking.}
After training on binding pairs, we evaluate mutation-local likelihood
scores without DMS supervision on 71 peptide-mutation cohorts
(10,064 variants) and 23 TCR-mutation cohorts (703 variants).
These benchmarks span ePytope, BATCAVE, ATLAS, TRAIT, SKEMPI~v2,
and the 868-Z11 saturation scans
(Appendix~\ref{app:tcrbinding}). The TCR-anchored \textsc{LogiCA}
configuration achieves the strongest mean Spearman correlations on
both sides of the interface ($\rho=0.271$ for peptide mutations and
$0.190$ for TCR mutations; Table~\ref{tab:tcr-main-compact}).
The matched positive-only \textsc{LogiMLM} control remains close to zero
($\rho=0.060$ and $0.044$, respectively), while the latent-contrastive
control LatentCA-TCR reaches $\rho=0.090$ on peptide mutations.
These comparisons suggest that both contrastive training and scoring
through native token distributions contribute to mutation-sensitive
transfer. The effects of anchor direction and negative count are
examined in Appendix~\ref{app:tcr-extended-ablations}.

\paragraph{Cross-modal contact localization.}
We next ask whether the learned token probabilities carry spatially
meaningful information. Perturbing one partner shifts the predicted
native-token distributions of the other, yielding dependency maps
without contact-label supervision (Figure~\ref{fig:tcr-figure}C--E).
Across 251 TCR3d structures~\citep{gowthaman2019tcr3d,lin2025tcr3d},
\textsc{LogiCA}-TCR reaches an AUROC of $0.743$ for peptide--CDR3$\beta$
contacts, compared with $0.717$ for \textsc{LogiMLM} and $0.648$ for
TCR-T5. The advantage is interface-dependent rather than uniform across
all contact types (Table~\ref{tab:dependency-map-interchain}).

\paragraph{Agonist prioritization and sequence design.}
The same native distributions can also prioritize candidates and guide
iterative sequence search without a separately trained generative head.
For the CMV pp65 epitope NLVPMVATV~\citep{gratama2001tetramer}, four
experimentally validated activity-enhancing substitutions are ranked 3rd, 8th,
29th, and 48th among 171 candidates
(Figure~\ref{fig:tcr-figure}F). For short peptide and CDR3 libraries,
we report top-10\% enrichment to reflect a practical screening budget;
the TCR-anchored model reaches $2.62\times$ enrichment on peptide
mutations and $1.60\times$ on TCR mutations.
Iterative design favors a different configuration: in 12-step
coordinate search, peptide-anchored \textsc{LogiCA} reaches
$4.51\times$ enrichment on peptide variants, compared with
$1.42\times$ for \textsc{LogiMLM}
(Table~\ref{tab:tcr-top5-generation}).
TULIP-TCR reaches $1.40\times$,
highlighting the advantage of partner-conditioned sequence optimization
under the evaluated search protocol.
Thus, the best ranking anchor need not be the best for iterative design.
The same native logits support ranking, contact analysis, and search
without task-specific heads (Appendix~\ref{app:tcr-extended-ablations}).

\vspace{-0.18cm}
\section{Outlook}
\vspace{-0.08cm}

\textsc{LogiCA} provides a general framework for contextualizing heterogeneous biological language models while preserving their native probabilistic interfaces. Rather than collapsing sequences into pooled embeddings or task-specific classifiers, cross-modal alignment directly reshapes each model's token distributions. Low-homology BindingDB results demonstrate generalization to unseen protein families, while the strongest advantages emerge in mutation-local and data-scarce settings, particularly when prioritizing rare high-functioning variants.

Importantly, our expanded experiments establish that the value of preserving native logits extends beyond discriminative prediction. The same context-conditioned distributions support mutation ranking, residue-level dependency mapping, and iterative sequence design without separately trained prediction heads, contact models, or generative decoders. 

These capabilities nevertheless have distinct computational and experimental limitations. Pair-specific contextualization is costlier than cached dual-encoder retrieval, making \textsc{LogiCA} particularly suited to candidate reranking, finite-library prioritization, structural interpretation, and local conditional design rather than exhaustive all-pairs screening. Moreover, the present generative evaluations are retrospective, using experimentally characterized mutational landscapes; prospective experimental validation of newly designed sequences remains an important next step. More broadly, preserving the native likelihood interface offers a modular route to contextualizing other biological foundation models, including DNA, RNA, genomic, and microbial encoders~\citep{ji2021dnabert,penic2025rinalmo,avsec2026alphagenome,wiatrak2025bacformer}.

\vspace{-0.18cm}
\section*{AI Use Statement}
\vspace{-0.08cm}

Generative AI tools were used to assist with drafting and editing portions of the manuscript and with writing, debugging, and executing code used in the research workflow. All AI-assisted text and code were reviewed by the authors, and code outputs used in the study were checked for correctness. Scientific decisions, interpretation of results, and final claims were made and verified by the authors. The authors take full responsibility for the final content of this work, including text, code, analyses, and other artifacts produced with the aid of generative AI.

\bibliography{references,logica_2026_functional_alignment_refs}
\bibliographystyle{iclr2027_conference}

\appendix
\setcounter{figure}{0}
\setcounter{table}{0}
\renewcommand{\thefigure}{S\arabic{figure}}
\renewcommand{\thetable}{S\arabic{table}}

\section{Related Work}
\label{app:related}

\paragraph{Native probabilistic interfaces in biological foundation models.}
Pretrained biological sequence models use position-wise token distributions for mutation scoring, evolutionary analysis, sequence design, generation, and token-level interpretation across proteins, nucleic acids, and molecular representations~\citep{lin2023evolutionary,riesselman2018deep,meier2021language,hie2024efficient,hie2022evolutionary,yuksel2023selformer,dalla2025nucleotide,cui2025towards,brixi2026genome,tomaz2024nucleotide,gong2024evorank,zhang2024interactingmotifs,madani2023large,johnson2021generating,mccarter2026make}. Related foundation models now span DNA/RNA, genomic-state, microbial, and synthetic-biology settings~\citep{dalla2025nucleotide,penic2025rinalmo,ji2021dnabert,avsec2021effective,gao2024epigept,avsec2026alphagenome,zvyagin2023genslms,wiatrak2025bacformer}.

\paragraph{Contextualized biological sequence models.}
Biological foundation models increasingly augment sequence representations with
external biological context, including cellular or tissue state, genomic
neighborhoods, interaction networks, and molecular graphs
\citep{li2024pinnacle,hwang2024glm,avsec2026alphagenome,wiatrak2025bacformer,
wang2024biobridge}. Closest to our setting are paired or target-conditioned
models for peptide design, TCR--peptide binding, interacting proteins,
antibody and TCR chain pairing, and partner-specific sequence modeling~\citep{ullanat2026mint, mizrahi20234m, meynard2024tulip, karthikeyan2025tcrt5, chen2025pepmlm, burbach2024improving, liu2025plminteract, lupo2024pairing}.
Many of these methods condition one biological sequence on another through
co-encoding, cross-attention, adapters, or conditional masked-language-modeling
objectives. They therefore retain, to varying degrees, the token-level interface
needed for residue scoring and generation. However, their matching or
interaction objectives are typically not defined directly on contextualized
token likelihoods. In contrast, \textsc{LogiCA} uses token log-likelihoods themselves as
the contrastive matching scores, preserving compatibility with each pretrained
model's native output head.

\paragraph{Fine-tuning and adapter-based conditioning.}
A common way to adapt pretrained masked language models is to fine-tune the
entire model or insert parameter-efficient modules such as low-rank adapters
\citep{hu2022lora}. Such approaches can be applied to paired biological inputs,
including protein--protein interactions, antibody heavy/light chains,
TCR$\alpha$/$\beta$ chains, and other cross-sequence settings
\citep{lupo2024pairing, zhang2024epitope, meynard2024tulip, burbach2024improving, ullanat2026mint, 
nagano2025contrastive, deutschmann2024domain}. Standard masked
language modeling trains a model to reconstruct masked tokens from their
surrounding context,
\begin{equation}
  \mathcal{L}_{\mathrm{MLM}}
  =
  - \frac{1}{|\Omega|}
  \sum_{i \in \Omega}
  \log p_\theta(x_i \mid x_{\setminus i})
  =
  - \frac{1}{|\Omega|}
  \sum_{i \in \Omega}
  \log
  \frac{\exp(\ell_{i,x_i})}
       {\sum_{a \in \mathcal{A}} \exp(\ell_{i,a})},
  \label{eq:app-mlm}
\end{equation}
where $\Omega$ is the set of masked positions, $\ell_{i,a}$ is the logit for token $a$ at position $i$ and
$\mathcal{A}$ is the model vocabulary. Conditional MLMs extend this objective
by providing an additional context $y$, yielding token probabilities of the
form $p_\theta(x_i \mid x_{\setminus i}, y)$. While this preserves the
probabilistic token interface, reconstruction alone is often a weak signal for
biological compatibility: it encourages recovery of observed tokens, but does
not directly separate matched from mismatched contexts. \textsc{LogiCA} instead uses
paired supervision contrastively in logit space, so that the contextualized
token likelihoods are optimized as compatibility scores.

\paragraph{Latent-space contrastive matching.}
A second line of work learns compatibility in a shared latent space using
CLIP-style contrastive supervision \citep{radford2021learning}. This template
has been applied to drug--target interaction and virtual screening
\citep{huang2021moltrans,bai2023drugban,singh2023contrastive,jia2026drugclip,liu2025spdti,yu2025graph,passaro2025boltz2}, protein retrieval
\citep{su2024protrek}, TCR--antigen recognition
\citep{zhang2024epitope,nagano2025contrastive}, and multimodal biological
integration \citep{gayoso2021totalvi,ashuach2023multivi,zhang2026apollo}.
The standard InfoNCE or NT-Xent objective scores a sequence--context pair using
a pooled-latent similarity, for example
\begin{equation}
  s_z(x,y)
  =
  \langle f_\theta(x),\, g_\phi(y)\rangle,
  \label{eq:app-embed-sim}
\end{equation}
where $f_\theta$ and $g_\phi$ map the sequence and context into a shared
latent space. These scores are effective for retrieval and binary matching,
but they no longer correspond to token likelihoods. Consequently, the
per-position distribution
$p_\theta(x_i \mid x_{\setminus i}, y)$ is not recoverable from
$s_z(x,y)$, making it difficult to localize scores to mutated
residues, perform likelihood-based generation, or analyze position-wise
probabilistic signals. \textsc{LogiCA} keeps the contrastive template but changes the
scored object: rather than contrasting pooled latents, it contrasts
context-conditioned token log-likelihoods.

\paragraph{Likelihood-based variant ranking.}
Variant-effect prediction is naturally comparative: experiments often ask which
variant is more functional, resistant, or compatible under a fixed biological
condition. Protein language models commonly score mutations using
wild-type-normalized log-likelihood ratios at the mutated positions
$\mathcal{M}$,
\begin{equation}
  s^{\circ}_{\mathcal M}(x;\, x^{\mathrm{wt}})
  =
  \frac{1}{|\mathcal M|}
  \sum_{i \in \mathcal M}
  \log
  \frac{p_\theta(x_i \mid x_{\setminus i})}
       {p_\theta(x^{\mathrm{wt}}_i \mid x^{\mathrm{wt}}_{\setminus i})},
  \label{eq:app-uncond-llr}
\end{equation}
and such scores are widely used in protein variant-effect benchmarks
\citep{meier2021language,notin2023proteingym}. Related likelihood-based
approaches incorporate evolutionary context, retrieval-augmented protein
families, or pairwise ranking supervision
\citep{gong2024evorank, notin2022tranception,
pugh2025likelihood, lee2023fine, zhao2024contrastive}. These methods motivate
the use of token likelihoods for ranking, but the conditioning signal is
usually the sequence itself, an evolutionary family, or preference supervision
over variants. \textsc{LogiCA} instead makes the conditioning variable an explicit
external biological context, such as a drug, epitope, ligand, or binding
partner, yielding scores based on
$p_\theta(x_i \mid x_{\setminus i}, y)$.

\paragraph{Preference and logit-space contrastive objectives.}
Preference-based fine-tuning provides another route for adapting pretrained
language models to task-specific comparisons. Direct Preference Optimization
(DPO) \citep{rafailov2023direct} and related methods optimize models so that
preferred outputs receive higher likelihood than dispreferred ones. Extensions
to masked language models and biological sequence models often use
pseudo-log-likelihood or average token likelihood as a sequence-level reward
\citep{lee2023fine,zhao2024contrastive, hawkins2024likelihood}, for example
\begin{equation}
  r_\theta(x)
  =
  \frac{1}{L}
  \sum_{i=1}^{L}
  \log p_\theta(x_i \mid x_{\setminus i}),
  \label{eq:app-pll-reward}
\end{equation}
and optimize pairwise preferences with a soft-margin or Bradley--Terry-style
loss. Recent work has also explored contrastive supervision in log-likelihood
space for autoregressive models, including contrastive preference learning and
contrastive preference optimization
\citep{hejna2023contrastive,xu2024contrastive}. These methods compare outputs
conditioned on a single input, typically for generation or translation.
\textsc{LogiCA} differs in both setting and scoring: it performs contrastive learning
over structured biological pairs and uses contextualized token likelihoods from
masked biological language models as bidirectional compatibility scores.

\paragraph{How \textsc{LogiCA} differs.}
The key distinction is the \emph{object} being aligned. Latent contrastive
models compare pooled representations, while conditional MLMs preserve token
heads but do not directly contrast matched and mismatched partners.
LogiCA uses context-conditioned native token likelihoods as its contrastive
scores. External biological context therefore changes the same position-wise
probabilities subsequently used for mutation ranking, dependency analysis,
and sequence design. This differs from preference-based protein scoring
without explicit external partners, while retaining a mutation-local
likelihood decomposition absent from standard pooled similarities.

%

\clearpage
\section{Mathematical Foundations and Properties of \textsc{LogiCA}}
\label{app:logica-theory}
\label{app:proofs}

This appendix develops the architecture and mathematical properties of the
native token-likelihood interface underlying \textsc{LogiCA}. We first define
the cross-modal adapter and notation, then establish the mutation-local
cancellation identities and their extension to variants with different
mutation supports. We subsequently analyze the score-level derivatives of
the pairwise preference objective and the statistical behavior of averaged
multi-site evidence under an explicit dependence-aware noise model. These
results concern the scoring formulation; empirical evaluations of ranking,
localization, and design are presented separately.

\subsection{Cross-modal adapter architecture and notation}
\label{app:arch-notation}

Figure~\ref{fig:logica-architecture} illustrates how bidirectional cross-attention couples heterogeneous pretrained encoders (such as ESM-2 and SELFormer, or TCRLang and ESM-2) in a shared bottleneck space while returning contextualized residual updates to each model's native hidden dimension prior to its original language-modeling head. Table~\ref{tab:appendix-notation} summarizes all mathematical notation used throughout the main text and theoretical proofs.

\begin{figure}[!htbp]
    \centering
    \includegraphics[width=1\linewidth]{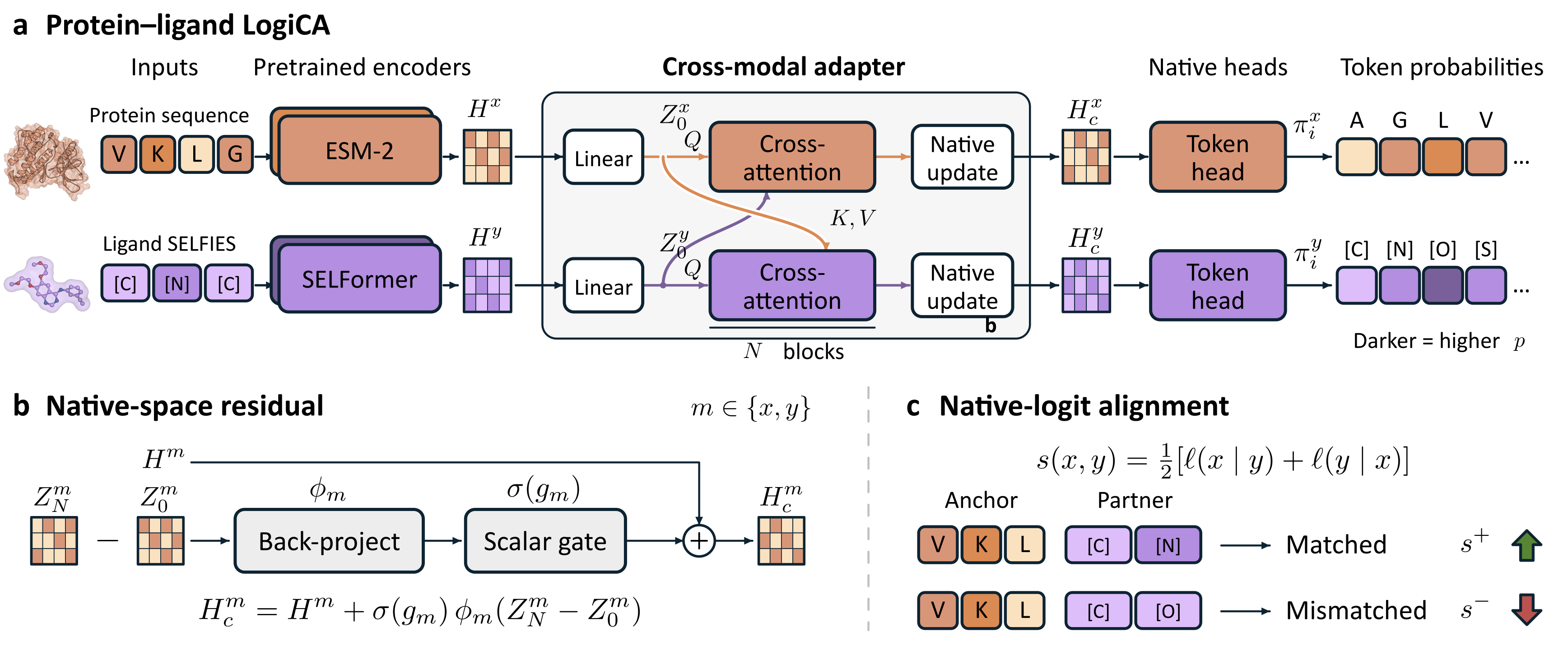}
    \caption{\textbf{Cross-modal adapters preserve each model's native token-likelihood interface.}
\textbf{(a)} Bidirectional cross-attention exchanges information between ESM-2 protein and SELFormer ligand representations without replacing either pretrained token head.
\textbf{(b)} A back-projected residual returns the contextual update to each model's native hidden space; the protein--ligand implementation uses a scalar gate, whereas the TCR--peptide implementation omits it.
\textbf{(c)} The native heads yield directional token log-likelihoods used for contrastive matching. Darker output cells indicate higher schematic token probabilities. The EGFR--erlotinib complex (PDB: 1M17) is illustrative.}
    \label{fig:logica-architecture}
\end{figure}

\begin{table*}[!htbp]
\centering
\caption{Core notation for \textsc{LogiCA} scoring and ranking.}
\label{tab:appendix-notation}
\small
\setlength{\tabcolsep}{5pt}
\renewcommand{\arraystretch}{1.15}
\begin{tabular*}{\textwidth}{@{\extracolsep{\fill}}p{0.24\textwidth}p{0.70\textwidth}@{}}
\toprule
\textbf{Symbol} & \textbf{Description} \\
\midrule

\rowcolor{blue!6}
\multicolumn{2}{@{}l}{\textit{Sequences and contexts}} \\
$x$ & Query sequence in its native vocabulary, such as a protein, TCR, or peptide. \\
$y$ & External context, such as a binding partner, ligand, drug, therapy, or tokenized condition. \\
$x^{\mathrm{wt}}$ & Wild-type reference sequence for variant ranking. \\
$L$ & Sequence length of $x$; $[L]=\{1,\ldots,L\}$. \\
$a,\,\mathcal C$ & Anchor and candidate set in the contextual ranking template (Eq.~\ref{eq:logica-ranking}); $a$ may be a sequence or context, $\mathcal C$ is the set of competing candidates. \\

\midrule
\rowcolor{blue!6}
\multicolumn{2}{@{}l}{\textit{Token likelihoods and site sets}} \\
$\pi_\theta(x_i\mid x_{\setminus i},y)$
  & Contextualized probability assigned to the observed token $x_i$ when site $i$ is masked, under context $y$. \\
$\pi_{\theta,i}(\cdot\mid x_{\setminus i},y)$
  & Full categorical distribution at site $i$ over the native vocabulary; used in the cross-modal dependency map (Eq.~\ref{eq:logica-cross-dependency-main}). \\
$A\subseteq[L]$
  & Scored site set. For full-sequence interaction scoring, $A=[L]$; for variant scoring, $A=\mathcal M$. \\
$A_x,A_y$
  & Per-modality scored token positions used in the bidirectional score $s_\alpha$ (Eq.~\ref{eq:logica-bidirectional-score}). \\
$\ell_A(x\mid y)$
  & Site-averaged log-likelihood over $A$:
    $\frac{1}{|A|}\sum_{i\in A}\log\pi_\theta(x_i\mid x_{\setminus i},y)$. \\
$\mathcal M(x,x^{\mathrm{wt}})$
  & Mutation set $\{i:x_i\neq x^{\mathrm{wt}}_i\}$ for substitution-only variants ($x$ and $x^{\mathrm{wt}}$ share length); throughout the proofs, $m=|\mathcal M|$. \\
$\mathcal A^y_j$
  & Substitution alphabet at context position $j$ used to perturb the context token $y_j$ in the dependency map. \\

\midrule
\rowcolor{blue!6}
\multicolumn{2}{@{}l}{\textit{Task scores and diagnostics}} \\
$s(a,c)$
  & Generic compatibility score in the contextual ranking template; instantiated by $s_\alpha$ for interaction scoring or $s_{\mathcal M}$ for variant scoring. \\
$s_{\textsc{LogiCA}}(x,y;A)$
  & Site-averaged \textsc{LogiCA} score $\ell_A(x\mid y)$ (Eq.~\ref{eq:logica-score}); reduces to the directional log-likelihood under the chosen site set. \\
$s_\alpha(x,y)$
  & Bidirectional interaction score (Eq.~\ref{eq:logica-bidirectional-score}):
    $\alpha\,\ell_{A_x}(x\mid y)+(1-\alpha)\,\ell_{A_y}(y\mid x)$. \\
$s_{\mathcal M}(x,y;x^{\mathrm{wt}})$
  & Mutation-local variant score (Eq.~\ref{eq:logica-mut-score}):
    $\ell_{\mathcal M}(x\mid y)-\ell_{\mathcal M}(x^{\mathrm{wt}}\mid y)$. \\
$\bar\gamma$
  & Symmetric contextualized likelihood margin between matched and sampled-mismatched contexts. \\
$\gamma_{x\mid y},\gamma_{y\mid x}$
  & Directional likelihood margins under the negative-sampling distributions. \\
$\alpha$
  & Learned directional weight in $s_\alpha$. \\
$\tau$
  & Temperature in candidate-choice, contrastive, and pairwise preference losses. \\
$D^{y\to x}_{ij}$
  & Cross-modal dependency map (Eq.~\ref{eq:logica-cross-dependency-main}) measuring how perturbing context token $j$ shifts the predicted distribution at sequence position $i$. \\

\midrule
\rowcolor{blue!6}
\multicolumn{2}{@{}l}{\textit{Proof notation}} \\
$\ell_i^V$
  & Site log-likelihood $\log\pi_\theta(x_{V,i}\mid x_{V,\setminus i},y)$ for $V\in\{A,B,\mathrm{wt}\}$. \\
$d_i$
  & Per-site log-likelihood gap $\ell_i^A-\ell_i^B$. \\
$\Delta$
  & Exact score gap $\frac{1}{m}\sum_{i\in\mathcal M}d_i$ for variants with the same mutation set. \\
$x_A,x_B$
  & Two variants compared under the same context and mutation set. \\
$\mu_i,\bar{\mu}_{\mathcal M}$
  & Site-level mean advantage and its average over the mutation set. \\
$\sigma_i^2,\nu_{\mathcal M}^2$
  & Site-level sub-Gaussian parameters and the averaged parameter
    $\frac{1}{m^2}\sum_{i\in\mathcal M}\sigma_i^2$ used in the misranking bound. \\

\bottomrule
\end{tabular*}
\end{table*}
\renewcommand{\arraystretch}{1.0}

\subsection{Mutation-local likelihood decomposition}
\label{app:mutation-decomposition}

We begin with the exact cancellation of wild-type references for matched
mutation sets. We then relate this identity to common-support comparisons of
variants with different mutation sets, retaining an explicit residual for
context-dependent changes at unmutated positions. Throughout, the conditional
likelihood at a given site is allowed to depend on all other sequence positions
and on the external biological context. No independence or absence of
sequence-level epistasis is assumed for these algebraic results.

\paragraph{Matched mutation sets and exact wild-type cancellation.}

We first consider two variants of the same wild-type sequence
$x^{\mathrm{wt}}$ under a fixed biological context $y$. Let
$x^A$ and $x^B$ perturb the same nonempty set of positions
relative to the wild type:
\[
  \mathcal M
  =
  \mathcal M(x^A,x^{\mathrm{wt}})
  =
  \mathcal M(x^B,x^{\mathrm{wt}}),
  \qquad
  m=|\mathcal M|\ge 1.
\]
The variants share the mutated positions, although their
substituted tokens at those positions may differ.

For $V\in\{A,B,\mathrm{wt}\}$ and $i\in[L]$, define
\[
  \ell_i^V
  :=
  \log \pi_\theta(x_i^V\mid x^V_{\setminus i},y).
\]
The per-site log-likelihood advantage of $x^A$ over $x^B$
is
\[
  d_i := \ell_i^A-\ell_i^B,
\]
and the average advantage over the shared mutation set is
\[
  \Delta
  :=
  \frac{1}{m}\sum_{i\in\mathcal M}d_i.
\]

For a candidate identical to the wild type, we adopt the
convention
$s_{\varnothing}(x^{\mathrm{wt}},y;x^{\mathrm{wt}})=0$,
corresponding to zero log-likelihood change relative to
the reference. The following result concerns nonempty
mutation sets.

\setcounter{proposition}{0}

\begin{proposition}[Mutation-local reduction]
\label{thm:mutation-local}
For two variants $x^A$ and $x^B$ sharing the same
wild-type reference, biological context $y$, and
nonempty mutation set $\mathcal M$,
\[
  s_{\mathcal M}(x^A,y;x^{\mathrm{wt}})
  -
  s_{\mathcal M}(x^B,y;x^{\mathrm{wt}})
  =
  \Delta.
\]
Consequently, for the two-candidate ranking problem
$\mathcal C=\{x^A,x^B\}$, Eq.~\ref{eq:logica-ranking}
gives
\[
  \Pr(x^A\succ x^B\mid y,\mathcal C)
  =
  \sigma(\Delta/\tau).
\]
Thus, additive terms shared by both candidates cancel
from their pairwise score difference, including the
wild-type reference term in
Eq.~\ref{eq:logica-mut-score}.
\end{proposition}

\begin{proof}
By Eq.~\ref{eq:logica-mut-score},
\begin{align*}
  &s_{\mathcal M}(x^A,y;x^{\mathrm{wt}})
  -
  s_{\mathcal M}(x^B,y;x^{\mathrm{wt}})\\
  &=
  \frac{1}{m}\sum_{i\in\mathcal M}
  \log
  \frac{
    \pi_\theta(x_i^A\mid x^A_{\setminus i},y)
  }{
    \pi_\theta(x_i^{\mathrm{wt}}
    \mid x^{\mathrm{wt}}_{\setminus i},y)
  }\\
  &\quad-
  \frac{1}{m}\sum_{i\in\mathcal M}
  \log
  \frac{
    \pi_\theta(x_i^B\mid x^B_{\setminus i},y)
  }{
    \pi_\theta(x_i^{\mathrm{wt}}
    \mid x^{\mathrm{wt}}_{\setminus i},y)
  }\\
  &=
  \frac{1}{m}\sum_{i\in\mathcal M}
  \left(\ell_i^A-\ell_i^B\right)
  =
  \Delta.
\end{align*}
The wild-type reference terms cancel because both
candidates are evaluated over the same positions against
the same reference sequence. Substituting the resulting
score difference into the two-candidate form of
Eq.~\ref{eq:logica-ranking} yields the stated
Bradley--Terry probability.
\end{proof}

The cancellation is exact and does not require
independence between mutated positions or invariance
of their conditional likelihoods to the surrounding
sequence. In particular, contextual and epistatic
effects may influence each $\ell_i^V$ through
$x^V_{\setminus i}$ without affecting the cancellation
of the shared reference term.

\paragraph{Extension to non-matching mutation sets.}
We next consider variants with different mutation
supports,
\[
  \mathcal M_A
  =
  \mathcal M(x^A,x^{\mathrm{wt}}),
  \qquad
  \mathcal M_B
  =
  \mathcal M(x^B,x^{\mathrm{wt}}),
\]
where $\mathcal M_A\neq\mathcal M_B$.
Define their nonempty union
$\mathcal U=\mathcal M_A\cup\mathcal M_B$.
Evaluating both variants on this common support gives
\begin{equation}
\begin{aligned}
  \Delta_{\mathcal U}(x^A,x^B;y)
  &=
  s_{\mathcal U}(x^A,y;x^{\mathrm{wt}})
  -
  s_{\mathcal U}(x^B,y;x^{\mathrm{wt}})\\
  &=
  \frac{1}{|\mathcal U|}
  \sum_{i\in\mathcal U}
  (\ell_i^A-\ell_i^B).
\end{aligned}
\label{eq:app-union-anchor}
\end{equation}
Here $s_{\mathcal U}$ denotes the reference-relative
score evaluated over the common support, as defined
in Proposition~\ref{prop:mutation-local-cancellation}.
The reference cancellation in
Eq.~\ref{eq:app-union-anchor} is exact because both
candidates are evaluated over the same index set.

To relate this pairwise union-support comparison to
the individually normalized scores used for variant
ranking, partition $\mathcal U$ into the disjoint sets
\[
\begin{aligned}
  \mathcal I
  &=\mathcal M_A\cap\mathcal M_B,\\
  \mathcal A
  &=\mathcal M_A\setminus\mathcal M_B,\\
  \mathcal B
  &=\mathcal M_B\setminus\mathcal M_A.
\end{aligned}
\]
At positions $i\in\mathcal A$, variant $x^B$ carries
the wild-type token, whereas at positions
$j\in\mathcal B$, variant $x^A$ carries the wild-type
token. Nevertheless, the corresponding conditional
likelihoods need not equal those of the wild-type
sequence, since the remaining sequence context
may differ.

Adding and subtracting the wild-type log-likelihoods
therefore gives the following exact decomposition:
\begin{equation}
\begin{aligned}
  |\mathcal U|\,
  \Delta_{\mathcal U}(x^A,x^B;y)
  &=
  |\mathcal M_A|\,
  s_{\mathcal M_A}(x^A,y;x^{\mathrm{wt}})\\
  &\quad-
  |\mathcal M_B|\,
  s_{\mathcal M_B}(x^B,y;x^{\mathrm{wt}})\\
  &\quad+R_{AB},
\end{aligned}
\label{eq:app-union-factorization}
\end{equation}
where
\begin{equation}
\begin{aligned}
  R_{AB}
  &=
  \sum_{i\in\mathcal A}
  \left(\ell_i^{\mathrm{wt}}-\ell_i^B\right)\\
  &\quad+
  \sum_{j\in\mathcal B}
  \left(\ell_j^A-\ell_j^{\mathrm{wt}}\right).
\end{aligned}
\label{eq:app-union-residual}
\end{equation}

To verify the identity, expand the individually
normalized scores:
\begin{align*}
  &|\mathcal M_A|s_{\mathcal M_A}(x^A,y;x^{\mathrm{wt}})
  -
  |\mathcal M_B|s_{\mathcal M_B}(x^B,y;x^{\mathrm{wt}})\\
  &=
  \sum_{i\in\mathcal M_A}
  (\ell_i^A-\ell_i^{\mathrm{wt}})
  -
  \sum_{j\in\mathcal M_B}
  (\ell_j^B-\ell_j^{\mathrm{wt}})\\
  &=
  \sum_{i\in\mathcal I}
  (\ell_i^A-\ell_i^B)
  +
  \sum_{i\in\mathcal A}
  (\ell_i^A-\ell_i^{\mathrm{wt}})\\
  &\quad+
  \sum_{j\in\mathcal B}
  (\ell_j^{\mathrm{wt}}-\ell_j^B).
\end{align*}
Comparing this expression with
$\sum_{i\in\mathcal U}(\ell_i^A-\ell_i^B)$
yields the residual in
Eq.~\ref{eq:app-union-residual}.

The residual $R_{AB}$ captures changes in the
conditional likelihood of wild-type tokens at
positions that are mutated in only one candidate.
It is therefore a context-dependent correction,
rather than an additional reference-normalization
term. In particular, $R_{AB}$ vanishes if
\[
\begin{aligned}
  \ell_i^B&=\ell_i^{\mathrm{wt}},
  &&i\in\mathcal A,\\
  \ell_j^A&=\ell_j^{\mathrm{wt}},
  &&j\in\mathcal B.
\end{aligned}
\]
These conditions correspond to invariance of the
wild-type token likelihood at non-overlapping
mutation sites under the sequence-context changes
introduced by the competing variant. Under this
assumption, Eq.~\ref{eq:app-union-factorization}
reduces to
\[
\begin{aligned}
  |\mathcal U|\Delta_{\mathcal U}
  &=
  |\mathcal M_A|s_{\mathcal M_A}(x^A,y;x^{\mathrm{wt}})\\
  &\quad-
  |\mathcal M_B|s_{\mathcal M_B}(x^B,y;x^{\mathrm{wt}}).
\end{aligned}
\]

When the mutation sets have equal cardinality,
$|\mathcal M_A|=|\mathcal M_B|=m$, the general
identity becomes
\[
\begin{aligned}
  &s_{\mathcal M_A}(x^A,y;x^{\mathrm{wt}})
  -
  s_{\mathcal M_B}(x^B,y;x^{\mathrm{wt}})\\
  &\qquad=
  \frac{|\mathcal U|}{m}\Delta_{\mathcal U}
  -
  \frac{R_{AB}}{m}.
\end{aligned}
\]
Thus, the individually normalized score difference
is proportional to the union-support comparison
when $R_{AB}=0$. For example, two single-point
mutants at distinct positions have
$m=1$ and $|\mathcal U|=2$, so their individually
normalized score difference equals
$2\Delta_{\mathcal U}$ when the residual vanishes.

For different mutation counts, the union-support
comparison and the difference of individually
normalized scores are generally distinct even when
$R_{AB}=0$. This distinction is relevant because
\textsc{LogiCA} uses the individual mutation-local
scores in Eq.~\ref{eq:logica-mut-score} for
candidate ranking, rather than recomputing a
different union support for each candidate pair.

\paragraph{Role of wild-type reference subtraction.}
The reference term in Eq.~\ref{eq:logica-mut-score}
expresses the likelihood of each substitution
relative to the wild-type token at the same
position and under the same biological context.
This provides a common reference for mutations
occurring at positions with different baseline
token probabilities. Such differences may reflect
site-specific evolutionary constraints or
sequence-dependent predictive uncertainty.

For matched mutation sets, the reference terms
cancel exactly in pairwise score differences.
For non-matching mutation sets, subtraction
makes the contributions of different sites
comparable relative to their respective wild-type
baselines. Equation~\ref{eq:app-union-factorization}
further identifies the correction required to
relate these individually anchored contributions
to a common-support comparison when sequence
context changes affect unmutated token
likelihoods.

Importantly, reference subtraction does not
remove contextual or epistatic effects from
the mutant likelihoods. These effects remain
encoded in
$\pi_\theta(x_i\mid x_{\setminus i},y)$
and may contribute to the residual
$R_{AB}$ in comparisons involving different
mutation sets.

\paragraph{Role of mutation-count normalization.}
The factor $1/|\mathcal M|$ in
Eq.~\ref{eq:logica-mut-score} converts the
cumulative log-likelihood ratio into an average
contribution per mutated position. Without this
normalization, the magnitude of the score can
increase with mutation count, potentially
introducing systematic differences between
single-site and multi-site variants.

In particular, when the average log-likelihood
change associated with a substitution is negative,
an unnormalized sum tends to assign more negative
scores to variants with larger mutation sets,
even if their mean per-site likelihood advantage
is comparable. Averaging mitigates this
mutation-count dependence and defines a
consistent per-site scoring scale.

This normalization is a modeling choice rather
than a direct consequence of the union-support
identity. It preserves the ordering of candidates
with the same mutation count, but can change
comparisons between variants with different
numbers of substitutions. The resulting score
represents average mutation-local evidence,
rather than a cumulative sequence-level
likelihood change.

\paragraph{Comparison with pooled latent scores.}
The preceding reductions follow from the
additive structure of mutation-local token
log-likelihoods. Standard pooled latent
similarities used in CLIP-style contrastive
learning do not generally possess an analogous
decomposition.

For example, consider
\[
  s_z(x,y)
  =
  \langle f_\theta(x),g_\phi(y)\rangle,
\]
where $f_\theta$ and $g_\phi$ produce
sequence-level representations. The difference
between two variant scores is
\[
\begin{aligned}
  s_z(x^A,y)-s_z(x^B,y)
  &=
  \langle f_\theta(x^A),g_\phi(y)\rangle\\
  &\quad-
  \langle f_\theta(x^B),g_\phi(y)\rangle.
\end{aligned}
\]
Even when the variants differ at only a small
number of positions, their pooled
representations may change throughout the
embedding space. Consequently, this score
difference does not generally admit an exact
decomposition into native token-likelihood
contributions at the mutated sites.

This distinction concerns the scoring interface
rather than the expressive capacity of latent
models. Pooled representations can encode local
sequence changes and support effective pairwise
discrimination, while explicitly mutation-aware
latent models may recover some of the same
local information. However, standard pooled
similarity scores do not themselves provide
the normalized, context-conditioned token
probabilities used in the mutation-local
reduction.

\paragraph{Interpretation.}
Matched mutation sets permit exact cancellation of the wild-type reference.
For unequal mutation sets, the common-support comparison additionally
contains the residual $R_{AB}$, which accounts for context-dependent
likelihood changes at tokens that are wild type in one candidate. Averaging
by mutation count changes comparisons across different numbers of edits.
These statements concern the score decomposition, not assumptions of
independence or absence of epistasis.

\subsection{Preference gradients and mutation-local supervision}

At the level of score variables, the exact cancellation above implies that the
pairwise preference loss is directly supported on the mutant-token likelihoods
at the perturbed sites. This
is formalized by the following score-level derivative calculation.

\begin{corollary}[Mutation-local gradients]
\label{cor:mutation-local-gradient}
Under the pairwise preference loss
\[
  \mathcal{L}_{\mathrm{BT}}
  =
  -\log\sigma(\Delta/\tau),
  \qquad
  \Delta
  =
  \frac{1}{m}
  \sum_{i\in\mathcal M}
  (\ell_i^A-\ell_i^B),
\]
the direct partial derivatives with respect to the site log-likelihood terms are
\[
  \frac{\partial \mathcal{L}_{\mathrm{BT}}}{\partial \ell_i^A}
  =
  -\frac{1}{m\tau}\sigma(-\Delta/\tau),
  \qquad
  \frac{\partial \mathcal{L}_{\mathrm{BT}}}{\partial \ell_i^B}
  =
  \frac{1}{m\tau}\sigma(-\Delta/\tau),
  \qquad i\in\mathcal M,
\]
and these direct partial derivatives are zero for $i\notin\mathcal M$.
\end{corollary}

\begin{proof}
Since
\[
  \frac{d}{dt}\log\sigma(t)=\sigma(-t),
\]
we have
\[
  \frac{\partial \mathcal{L}_{\mathrm{BT}}}{\partial \Delta}
  =
  -\frac{1}{\tau}\sigma(-\Delta/\tau).
\]
For $i\in\mathcal M$,
\[
  \frac{\partial \Delta}{\partial \ell_i^A}
  =
  \frac{1}{m},
  \qquad
  \frac{\partial \Delta}{\partial \ell_i^B}
  =
  -\frac{1}{m}.
\]
The stated derivatives follow by the chain rule. If $i\notin\mathcal M$, then
$\ell_i^A$ and $\ell_i^B$ do not appear as direct terms in $\Delta$, so the
corresponding direct partial derivatives are zero.
\end{proof}

\paragraph{Parameter-level coupling.}
Corollary~\ref{cor:mutation-local-gradient} describes score-level partial
derivatives, not full parameter gradients. Because neural network parameters are
shared across positions, and because masked-token likelihoods at mutant sites
condition on the surrounding sequence, parameter gradients can still depend on
unmutated residues through the model architecture. The key point is that the
preference loss is directly supported on the mutated-site likelihood terms.

\paragraph{Interpretation.}
The loss gives direct positive pressure to increase the likelihood of the
preferred mutant tokens and direct negative pressure to decrease the likelihood
of the dispreferred mutant tokens at the same perturbed positions. Unchanged
positions still shape the conditional distribution through the sequence
context, but they are not explicit score terms in the mutation-local objective.

\subsection{Dependence-aware concentration of multi-site likelihood evidence}

The preceding results are algebraic and hold for any fixed model scores. We now
ask how the averaged gap $\Delta$ behaves when the per-site likelihood
advantages are viewed as noisy measurements of an underlying preference signal.
This captures a statistical setting in which different mutated sites provide
imperfect but positively biased evidence for the same preferred variant.

Unlike the deterministic cancellation argument, this concentration result
requires a probabilistic model for the site-level score gaps. Because masked
language models condition each masked-token prediction on the surrounding
sequence and share parameters across positions, the site-level advantage terms
for multi-site variants need not be statistically independent. We therefore use
a dependence-aware sub-Gaussian model: the noise vector may have cross-site
dependence, summarized by a positive semidefinite dependence proxy.

\begin{assumption}[Sub-Gaussian advantages]
\label{ass:site-advantages}
View $\{d_i\}_{i\in\mathcal M}$ as random site-level advantages induced by a
population of mutation comparisons under a fixed context. Write
\[
  d_i=\mu_i+\epsilon_i,
\]
where $\mu_i$ is the mean advantage at site $i$. Let
\[
  \epsilon_{\mathcal M}
  =
  (\epsilon_i)_{i\in\mathcal M}
  \in \mathbb R^m
\]
denote the vector of centered site-level noise terms. We assume that
$\epsilon_{\mathcal M}$ is jointly sub-Gaussian in the sense that, for every
$a\in\mathbb R^m$,
\[
  \mathbb E\!\left[
    \exp(a^\top \epsilon_{\mathcal M})
  \right]
  \le
  \exp\!\left(
    \frac{1}{2}a^\top \Sigma_{\mathcal M}a
  \right),
\]
for some positive semidefinite matrix
$\Sigma_{\mathcal M}\in\mathbb R^{m\times m}$.

Let
\[
  \bar{\mu}_{\mathcal M}
  =
  \frac{1}{m}\sum_{i\in\mathcal M}\mu_i,
  \qquad
  \nu_{\mathcal M}^2
  =
  \frac{1}{m^2}
  \mathbf 1^\top \Sigma_{\mathcal M}\mathbf 1,
\]
where $\mathbf 1\in\mathbb R^m$ is the all-ones vector.
\end{assumption}

\paragraph{Remark on the modeling assumption.}
Assumption~\ref{ass:site-advantages} is intended as a simple concentration
model for the averaged site-level score gap, not as a complete generative model
of masked language model predictions. The joint sub-Gaussian condition allows
the site-level advantage terms to be statistically dependent. This is important
for multi-site variants, where masked language model predictions are coupled
through the shared sequence context and shared model parameters. The dependence
is summarized through
$\mathbf 1^\top \Sigma_{\mathcal M}\mathbf 1$, the sub-Gaussian proxy for the
averaged noise direction. The deterministic cancellation result in
Proposition~\ref{thm:mutation-local} and the score-level gradient statement in
Corollary~\ref{cor:mutation-local-gradient} do not rely on this probabilistic
assumption.

\begin{corollary}[Misranking bound]
\label{cor:concentration}
Under Assumption~\ref{ass:site-advantages},
$\mathbb E[\Delta]=\bar{\mu}_{\mathcal M}$ and
$\Delta-\bar{\mu}_{\mathcal M}$ is
$\nu_{\mathcal M}^2$-sub-Gaussian. If
$\bar{\mu}_{\mathcal M}>0$, then
\[
  \Pr\!\left[
    s_{\mathcal M}(x^A,y;x^{\mathrm{wt}})
    \le
    s_{\mathcal M}(x^B,y;x^{\mathrm{wt}})
  \right]
  \le
  \exp\!\left(
    -\frac{\bar{\mu}_{\mathcal M}^2}
          {2\nu_{\mathcal M}^2}
  \right).
\]
\end{corollary}

\begin{proof}
By definition,
\[
  \Delta
  =
  \frac{1}{m}\sum_{i\in\mathcal M}d_i
  =
  \frac{1}{m}\sum_{i\in\mathcal M}\mu_i
  +
  \frac{1}{m}\sum_{i\in\mathcal M}\epsilon_i
  =
  \bar{\mu}_{\mathcal M}
  +
  \frac{1}{m}\mathbf 1^\top \epsilon_{\mathcal M}.
\]
Thus $\mathbb E[\Delta]=\bar{\mu}_{\mathcal M}$. Moreover, for any
$\lambda\in\mathbb R$, applying Assumption~\ref{ass:site-advantages} with
$a=(\lambda/m)\mathbf 1$ gives
\begin{align*}
  \mathbb E\!\left[
    \exp\!\left(\lambda(\Delta-\bar{\mu}_{\mathcal M})\right)
  \right]
  &=
  \mathbb E\!\left[
    \exp\!\left(
      \frac{\lambda}{m}
      \mathbf 1^\top \epsilon_{\mathcal M}
    \right)
  \right] \\
  &\le
  \exp\!\left(
    \frac{1}{2}
    \frac{\lambda^2}{m^2}
    \mathbf 1^\top \Sigma_{\mathcal M}\mathbf 1
  \right) \\
  &=
  \exp\!\left(
    \frac{\lambda^2\nu_{\mathcal M}^2}{2}
  \right).
\end{align*}
Therefore $\Delta-\bar{\mu}_{\mathcal M}$ is
$\nu_{\mathcal M}^2$-sub-Gaussian.

The misranking event is
\[
  \{\Delta\le0\}
  =
  \{\Delta-\bar{\mu}_{\mathcal M}\le-\bar{\mu}_{\mathcal M}\}.
\]
Applying the standard one-sided sub-Gaussian tail bound yields
\[
  \Pr[\Delta\le0]
  \le
  \exp\!\left(
    -\frac{\bar{\mu}_{\mathcal M}^2}
          {2\nu_{\mathcal M}^2}
  \right).
\]
Finally, Proposition~\ref{thm:mutation-local} identifies $\Delta$ with the
difference between the two mutation-local scores, so this is the stated
misranking bound.
\end{proof}

\paragraph{Independent-site special case.}
If the site-level noise terms are independent and each $\epsilon_i$ is
$\sigma_i^2$-sub-Gaussian, then
$\Sigma_{\mathcal M}$ may be taken to be diagonal with entries
$\sigma_i^2$. In that case,
\[
  \nu_{\mathcal M}^2
  =
  \frac{1}{m^2}\sum_{i\in\mathcal M}\sigma_i^2,
\]
which recovers the independent-site bound as a special case. If additionally
$\sigma_i^2\le\sigma^2$ for all $i\in\mathcal M$, then
\[
  \nu_{\mathcal M}^2
  \le
  \frac{\sigma^2}{m},
\]
and hence
\[
  \Pr\!\left[
    s_{\mathcal M}(x^A,y;x^{\mathrm{wt}})
    \le
    s_{\mathcal M}(x^B,y;x^{\mathrm{wt}})
  \right]
  \le
  \exp\!\left(
    -\frac{m\bar{\mu}_{\mathcal M}^2}{2\sigma^2}
  \right).
\]

\paragraph{Bounded-correlation interpretation.}
The dependence-aware bound also clarifies how cross-site coupling changes the
benefit of averaging. Suppose, for example, that $0\le\rho\le1$,
$\Sigma_{ii}\le\sigma^2$ and
$\Sigma_{ij}\le\rho\sigma^2$ for $i\ne j$. Then
\[
  \nu_{\mathcal M}^2
  =
  \frac{1}{m^2}
  \mathbf 1^\top \Sigma_{\mathcal M}\mathbf 1
  \le
  \frac{\sigma^2}{m}\{1+(m-1)\rho\}.
\]
Consequently,
\[
  \Pr[\Delta\le0]
  \le
  \exp\!\left(
    -\frac{m\bar{\mu}_{\mathcal M}^2}
          {2\sigma^2\{1+(m-1)\rho\}}
  \right).
\]
When $\rho=0$, this reduces to the independent-site rate. Larger positive cross-site dependence weakens the concentration benefit of averaging, whereas small cross-site dependence preserves much of the stabilizing effect of multi-site evidence.

\paragraph{Implications for mutation-local learning.}
The algebraic cancellation is exact for common mutation supports; the
parameter gradients remain coupled through the full sequence. The
misranking bound adds a separate, explicit sub-Gaussian assumption and
weakens when site-level noise is more strongly correlated. It should
therefore be read as a conditional property of averaged evidence, not a
guarantee of predictive accuracy.

\clearpage
\section{Emergent Capabilities of Context-Conditioned Native Logits}
\label{app:logica-downstream}
\label{app:logica-capabilities}

Contrastive alignment in \textsc{LogiCA} operates through the native token
distributions of pretrained biological language models. The resulting
context-conditioned probabilistic interface can therefore be used beyond
the pair discrimination objective employed during training. We develop three
complementary readouts: likelihood-based variant ranking, perturbation-based
dependency mapping, and context-guided sequence design. Ranking compares
observed sequences; dependency analysis measures sensitivity of token
predictions to context perturbations; and generation proposes or refines
candidate sequences using the same native token probabilities.

The following subsections define these procedures. Their experimental
validation is presented separately in Appendices~\ref{app:prot-lig-extra}
and~\ref{app:tcr-extended-ablations}.

\subsection{Likelihood-based variant ranking and reference adjustment}
\label{app:logica-ranking}

\paragraph{Mutation-local candidate ranking.}

For a fixed context \(y\), reference sequence \(x^{\mathrm{wt}}\), and variant
\(x\), \textsc{LogiCA} ranks candidates by the mutation-local score introduced
in Eq.~\ref{eq:logica-mut-score}:
\begin{equation}
  r(x;y,x^{\mathrm{wt}})
  =
  s_{\mathcal M}(x,y;x^{\mathrm{wt}}),
  \qquad
  \mathcal M=\{i:x_i\neq x^{\mathrm{wt}}_i\}.
  \label{eq:app-logica-ranking-rule}
\end{equation}
Given a candidate set
\(\mathcal X=\{x^{(1)},\ldots,x^{(K)}\}\), variants are sorted in decreasing
order of \(r(x^{(k)};y,x^{\mathrm{wt}})\). Equivalently, the scores define a
soft ranking distribution
\begin{equation}
  p(x^{(k)}\mid y,x^{\mathrm{wt}},\mathcal X)
  =
  \frac{
    \exp\!\left(
      s_{\mathcal M_k}(x^{(k)},y;x^{\mathrm{wt}})/\tau
    \right)
  }{
    \sum_{\ell=1}^{K}
    \exp\!\left(
      s_{\mathcal M_\ell}(x^{(\ell)},y;x^{\mathrm{wt}})/\tau
    \right)
  },
  \label{eq:app-logica-variant-softmax}
\end{equation}
where \(\mathcal M_k=\{i:x^{(k)}_i\neq x^{\mathrm{wt}}_i\}\). Thus inference
uses exactly the same likelihood-ratio quantity optimized during preference
training.

\paragraph{Unconditional reference adjustment.}
\label{sec:logica-reference-adjustment}

At evaluation time, raw contextual likelihoods may reflect both
context-specific compatibility and the unconditional plausibility of the tokens
being scored. We therefore use a reference-adjusted score that subtracts the
corresponding unconditional native-head likelihood in each direction:
\begin{equation}
\widetilde{s}_{\alpha}(x,y)
=
\alpha
\left[
\ell_{A_x}(x\mid y)-\ell_{A_x}(x)
\right]
+
(1-\alpha)
\left[
\ell_{A_y}(y\mid x)-\ell_{A_y}(y)
\right],
\qquad \alpha\in[0,1].
\label{eq:logica-reference-adjusted-score}
\end{equation}
Here \(\ell_{A_x}(x)\) and \(\ell_{A_y}(y)\) are computed with the
native token heads without cross-modal conditioning. Thus,
\(\widetilde{s}_{\alpha}(x,y)\) isolates the likelihood gain associated
with the external partner. In contrast, pretraining optimizes contextual
bidirectional likelihoods (Eq.~\ref{eq:logica-bidirectional-score}) through
the ranking objective (Eq.~\ref{eq:logica-ranking}). Inference uses a
scoring rule appropriate to each task.

\noindent\textit{Protein--ligand binding.}
We score full sequences (\(A_x=[L_x]\), \(A_y=[L_y]\)) using the symmetric
unconditional-reference-adjusted score \(\widetilde{s}_{1/2}(x,y)\).

\noindent\textit{Protein--ligand DMS ranking.}
Variants are scored at substituted sites \(\mathcal M\) by the wild-type-relative
mutation score \(s_{\mathcal M}(x,y;x^{\mathrm{wt}})\)
(Eq.~\ref{eq:logica-mut-score}), using single-pass or strided wild-type
marginals (Appendix~\ref{app:logica-efficiency}). Under unseen-target LOO
adaptation, we further subtract the apo reference:
\(\Delta s_{\mathcal M} = s_{\mathcal M}(x,y;x^{\mathrm{wt}})
- s_{\mathcal M}^{\mathrm{ESM2}}(x;x^{\mathrm{wt}})\).

\noindent\textit{TCR--peptide mutation ranking.}
We use exact position-masked mutation scores at peptide or CDR3 mutation
sites, together with the partner-chain likelihood shift
\(\Delta\ell(y\mid x)-\Delta\ell(y\mid x^{\mathrm{wt}})\).
Table~\ref{tab:tcr-2x2-factorial} reports both the reference-adjusted and
pure holo-LLR results.

Wild-type anchoring (Eq.~\ref{eq:logica-mut-score}) compares substituted and
reference tokens under a fixed partner, whereas unconditional-reference
adjustment (Eq.~\ref{eq:logica-reference-adjusted-score}) isolates the partner's
contribution. The two operations are distinct and not algebraically
interchangeable.

\subsection{Cross-modal dependency mapping and residue-level interpretation}
\label{app:logica-dependency}

Because \textsc{LogiCA} preserves normalized token distributions, it can be
probed by perturbing one token and measuring the induced change in another
token's predicted distribution. This gives a dependency map over positions.
Prior token-probability perturbation analyses have shown that such
sensitivities can reveal structural contacts, interacting motifs, and
evolutionary constraints within a single sequence
~\citep{tomaz2024nucleotide,zhang2024interactingmotifs,cornman2024omg}.
\textsc{LogiCA} extends this idea across the conditioning interface.

For within-sequence dependencies, let \(x^{(j\to a)}\) denote the sequence
obtained by replacing token \(x_j\) with \(a\), and let \(\mathcal A^x_j\) be
the set of allowed substitutions at position \(j\). We define
\begin{equation}
  D^{x\to x}_{ij}(y)
  =
  \frac{1}{|\mathcal A^x_j|}
  \sum_{a\in\mathcal A^x_j}
  \left\|
    \pi_{\theta,i}(\cdot\mid x^{(j\to a)}_{\setminus i},y)
    -
    \pi_{\theta,i}(\cdot\mid x_{\setminus i},y)
  \right\|_2 .
  \label{eq:app-logica-within-dependency}
\end{equation}
Large \(D^{x\to x}_{ij}(y)\) indicates that perturbing position \(j\) in the
scored sequence changes the predicted distribution at position \(i\), under the
fixed context \(y\).

For cross-modality dependencies, we instead perturb the context. Let
\(y^{(j\to a)}\) be the context obtained by substituting token \(y_j\) with
\(a\), and let \(\mathcal A^y_j\) be the allowed substitution set for that
context position. We define
\begin{equation}
  D^{y\to x}_{ij}
  =
  \frac{1}{|\mathcal A^y_j|}
  \sum_{a\in\mathcal A^y_j}
  \left\|
    \pi_{\theta,i}(\cdot\mid x_{\setminus i},y^{(j\to a)})
    -
    \pi_{\theta,i}(\cdot\mid x_{\setminus i},y)
  \right\|_2 .
  \label{eq:app-logica-cross-dependency}
\end{equation}
This quantity measures how strongly token \(j\) in the context affects the
model's predicted distribution at token \(i\) in the scored sequence. When both
directions are available, the reverse map \(D^{x\to y}_{ji}\) is computed
analogously by perturbing \(x\) and measuring changes in the predicted token
distribution over \(y\). Together, these maps provide a token-level view of the
intermodal dependencies learned by the model.

\paragraph{Ligand-dependent site sensitivity.}
A complementary one-dimensional summary compares native token distributions
at the same sequence position under two alternative contexts $y_1$ and $y_2$.
We use the symmetrized Kullback--Leibler divergence
\begin{equation}
  D_{\mathrm{KL}}^{(i)}(y_1 \parallel y_2)
  =
  \frac{1}{2}\Big[
    \mathrm{KL}\big(\pi_\theta(\cdot\mid x_{\setminus i}, y_1)\,\|\,\pi_\theta(\cdot\mid x_{\setminus i}, y_2)\big)
    +
    \mathrm{KL}\big(\pi_\theta(\cdot\mid x_{\setminus i}, y_2)\,\|\,\pi_\theta(\cdot\mid x_{\setminus i}, y_1)\big)
  \Big]
  \label{eq:app-kl-site-selection}
\end{equation}
This score measures the sensitivity of the model's residue distribution to the
choice of external context at position $i$. It is a diagnostic of learned
context dependence and does not, by itself, establish a physical contact or
causal interaction; these interpretations require separate validation.

\subsection{Context-conditioned sequence generation and optimization}
\label{app:logica-generation}

Since \textsc{LogiCA} remains a conditional language model, it can also be used
as a generative model under a fixed context. Let \(A\subseteq[L]\) be a set of
designable positions and let \(x_{\setminus A}\) denote the fixed sequence
background. The model defines the pseudo-likelihood score
\begin{equation}
  \widetilde p_\theta(x_A\mid x_{\setminus A},y)
  \propto
  \prod_{i\in A}
  \pi_\theta(x_i\mid x_{\setminus i},y).
  \label{eq:app-logica-pseudolikelihood}
\end{equation}
This product is a scoring surrogate and need not define a globally consistent
joint distribution because masked-token conditionals are learned independently.
We therefore use the conditionals in a pseudo-Gibbs generation heuristic.
Starting from an initial sequence \(x^{(0)}\), each step selects a design
position \(i\in A\) and resamples
\begin{equation}
  x_i^{(t+1)}
  \sim
  \pi_\theta\!\left(
    \cdot \mid x^{(t)}_{\setminus i}, y
  \right),
  \qquad
  x_j^{(t+1)}=x_j^{(t)}
  \ \text{for } j\neq i .
  \label{eq:app-logica-gibbs}
\end{equation}
The sampler can be constrained to valid biological tokens, fixed motif
positions, interface residues, or a mutation budget around a reference
sequence. When a reference \(x^{\mathrm{wt}}\) is available, proposals can also
be ranked or filtered by the change in the \textsc{LogiCA} score,
\begin{equation}
  \Delta s
  =
  s_{\mathcal M_{\mathrm{new}}}(x^{\mathrm{new}},y;x^{\mathrm{wt}})
  -
  s_{\mathcal M_{\mathrm{old}}}(x^{\mathrm{old}},y;x^{\mathrm{wt}}),
  \label{eq:app-logica-generation-delta}
\end{equation}
where $\mathcal M_{\mathrm{new}}$ and $\mathcal M_{\mathrm{old}}$ denote
the respective mutation sets relative to the same wild-type reference.
Thus, the logits used for ranking can also propose and refine
context-compatible sequences.

{

\paragraph{Sampling and search protocols.}
The same native probabilities support three distinct procedures. A
within-site conditional draws a non-wild-type residue from a masked
reference position ($\tau=0.30$). A finite-library Boltzmann rule samples
complete variants from a fixed candidate set with probabilities
proportional to $\exp(s_{\mathcal M}/\tau)$ ($\tau=0.25$). Finally, coordinate
search proposes successive substitutions along 12-step trajectories (500
chains per evaluation), using a score-difference acceptance rule with
$\tau_{\mathrm{MCMC}}=0.20$. These procedures need not target the same
distribution and should not be conflated.

\begin{table}[!htbp]
\centering
\caption{Distinct sequence-selection protocols evaluated in the supplement.
Only the coordinate method updates the sequence context across steps.}
\label{tab:app-generation-protocols}
\small
\renewcommand{\arraystretch}{1.12}
\begin{tabularx}{\textwidth}{@{}>{\raggedright\arraybackslash}p{0.30\textwidth}>{\raggedright\arraybackslash}Xl@{}}
\toprule
\textbf{Protocol} & \textbf{Scored object} & \textbf{Temperature} \\
\midrule
Within-site conditional & One substitution at a masked reference site & $0.30$ \\
Finite-library Boltzmann & Complete variants in a fixed enumerated library & $0.25$ \\
Coordinate score-guided search & Sequential sequence edits; 12 steps, 500 chains & $0.20$ \\
\bottomrule
\end{tabularx}
\end{table}

For a proposal kernel $q(x'\mid x)$, an exact Metropolis--Hastings rule
targeting the Boltzmann score would accept with probability
\begin{equation}
  a(x,x')=\min\!\left\{1,
  \exp\!\left(\frac{r(x')-r(x)}{\tau_{\mathrm{MCMC}}}\right)
  \frac{q(x\mid x')}{q(x'\mid x)}\right\},
  \label{eq:app-mh-proposal-ratio}
\end{equation}
where $r(x)=s_{\mathcal M(x,x^{\mathrm{wt}})}(x,y;x^{\mathrm{wt}})$.
The reported score-difference-only acceptance is exact for symmetric
proposals; otherwise it is a Metropolis-type heuristic, with no claim
of exact stationary Boltzmann sampling. Retrospective evaluations are
reported in Appendix~\ref{app:logica-generation-results}.

\paragraph{Contrastive context guidance.}
For applications that favor one external context over another, we additionally
form a contrastively steered token score for candidate residue $a$,
\begin{equation}
  \tilde{\ell}_i^{(\alpha)}(a \mid x_{\setminus i}, y_{\mathrm{target}}, y_{\mathrm{off}})
  =
  \ell_\theta(a \mid x_{\setminus i}, y_{\mathrm{target}})
  -
  \alpha\,\ell_\theta(a \mid x_{\setminus i}, y_{\mathrm{off}}).
  \label{eq:app-contrastive-gibbs}
\end{equation}
The context-specific score can be normalized over valid tokens before
sampling. In this formulation $\alpha$ controls the strength of the off-target
contrast; it is distinct from the directional mixing weight used for
bidirectional matching.

}

\clearpage

{

\section{Computational Complexity and Inference Efficiency}
\label{app:task-complexity}
\label{app:logica-efficiency}

The cost of native-logit scoring depends on the task. A pair-specific
cross-attention model is less suitable for exhaustive retrieval than a
factorized dual encoder, but its token head can reuse a single contextualized
sequence representation when many substitutions share the same background.
We distinguish these regimes below, using measured inference times rather
than treating a single per-pair benchmark as representative of every use case.

\subsection{Single-pass wild-type marginals and exact masked scoring}
\label{app:single-pass-vs-lpass}

The native token head outputs distributions at every sequence position in one
contextualized forward pass. For single-site variant libraries, these outputs
can be used as \emph{wild-type marginals}: all candidate substitutions are scored
against the same unmasked sequence context. This approximation differs from
exact position-masked likelihood scoring, which masks the evaluated site and
requires separate passes for distinct masking templates.

On an NVIDIA H100 GPU, with ESM-2 150M and SELFormer at protein and ligand
lengths $L_x=450$ and $L_y=64$, single-pass wild-type marginals take
$4.55\text{\,ms}$ per pair, versus $2{,}047.5\text{\,ms}$ for exhaustive
$L$-pass masking. The corresponding LOCO-30\% variant-ranking correlations are
nearly identical ($\rho=0.313$ and $0.312$, respectively). Subsampling scored
positions within a single pass is also possible: evaluating $15\%$ of sites
gives BindingDB AUROC $0.628$ instead of $0.630$ with all sites, and
wild-type enzyme--substrate ranking $0.402$ instead of $0.408$.
These are empirical approximations, not consequences of the concentration
bound in Corollary~\ref{cor:concentration}.

No functional-site annotations are needed for the principal benchmarks.
Sequence-level matching uses $A=[L]$, whereas mutation ranking obtains
$A=\mathcal M(x,x^{\mathrm{wt}})$ directly from the sequence differences.
Crystal contacts enter only as evaluation labels for the dependency maps.

\subsection{Architectures and benchmarking assumptions}
\label{app:task-by-task-complexity}

Three scoring interfaces determine the comparisons. A \emph{factorized dual
encoder} computes protein and partner embeddings independently and compares
cached vectors by a cosine product. A \emph{contextualized latent model}
(\textsc{LatentCA}, \textsc{LatentFuse}, or \textsc{SiteFuse}) uses the same
two-layer bidirectional cross-attention interaction tower as \textsc{LogiCA},
but outputs a pooled or mutation-site scalar score. \textsc{LogiCA} uses the
interaction tower and the backbones' native token heads to obtain normalized
conditional probabilities (matched implementations are described in
Appendix~\ref{app:protein-controlled-baselines}).

The distinction between \emph{cold-pair inference} and \emph{cached scoring}
is central. A cold pair requires unimodal encoding, which accounts for more
than $96\%$ of its compute: ESM-2 150M uses $148.26$ GFLOPs, SELFormer
$5.51$ GFLOPs, and the shared interaction tower $3.12$ GFLOPs. Cold-pair
H100 latencies are $4.55\text{\,ms}$ for \textsc{LogiCA} and
$4.42\text{\,ms}$ for \textsc{LatentCA}; the corresponding eight-thread
EPYC CPU measurements are $1{,}002.2$ and $977.0\text{\,ms}$. These values
should not be compared directly with cached dual-encoder matrix multiplication,
which does not run either backbone or cross-attention for each candidate pair.

The dense native head adds a further cost after contextualization. For a
protein hidden width $d_x$, scored-site count $|A|$, and vocabulary
$|\mathcal V_x|$, the leading protein-head operations scale as
$O(|A|d_x^2+|A|d_x|\mathcal V_x|)$. The corresponding ligand-head cost
scales with its width, vocabulary, and evaluated length. At 35M, 150M,
and 650M protein scale the measured full-head costs are $0.376$, $0.542$,
and $1.667$ GFLOPs, respectively. Selecting $H_x^{\mathrm{cond}}[:,A]$
\emph{before} applying \texttt{EsmLMHead} avoids evaluating the head on
unused positions (Figure~\ref{fig:app-task-complexity}A).

\subsection{Efficiency across downstream tasks}

\paragraph{Large-scale virtual screening.}
A factorized dual encoder is the natural first stage when scoring a large
Cartesian product of targets and partners: after encoding each input once,
all pair scores can be obtained by matrix multiplication. Pair-specific
cross-attention cannot be factorized in this way. For a $500\times500$
protein--ligand screen, the dual encoder takes $488.4\text{\,s}$ including
one-time encoding; exhaustive \textsc{LogiCA} with full bidirectional heads
takes $6{,}828.4\text{\,s}$ on the same eight-thread CPU. A two-stage system
that reranks only the top $1\%$ of dual-encoder candidates with \textsc{LogiCA}
reduces the measured total to $551.8\text{\,s}$. The gain is a reduction in
pairwise contextualization, not a claim that \textsc{LogiCA} competes with
cached cosine similarity for initial retrieval. The TCR--peptide
measurements show the same distinction: cached \textsc{LogiCA} costs
$5.70\text{\,ms}$ per pair ($11.39\text{\,ms}$ with apo subtraction),
versus $5.11\text{\,ms}$ for cross-attended \textsc{LatentCA} and
$0.063\text{\,ms}$ for dual-encoder cosine scoring.

\paragraph{Single-site and combinatorial mutagenesis.}
When all candidates are single substitutions of one reference, the same
native distribution provides probabilities for all 20 amino acids at every
site. On a library of $1{,}000$ variants spanning $53$ positions,
\textsc{LogiCA} requires $1.00\text{\,s}$ for wild-type-marginal scoring, or
$49.32\text{\,s}$ for exact per-site masking with shared templates;
\textsc{SiteFuse} takes $929.69\text{\,s}$ because it re-encodes each
mutant. This advantage does not extend unchanged to combinatorial libraries.
For $1{,}000$ distinct triple mutants, simultaneous mutation-site masking
puts \textsc{LogiCA} near \textsc{SiteFuse} ($930.27$ versus
$929.69\text{\,s}$), while exact per-site leave-one-out scoring with an apo
reference costs $2{,}864.90\text{\,s}$. The scoring protocol must therefore
be specified whenever a variant-screening speedup is reported.

\paragraph{Iterative sequence design.}
Mutations accepted during a coordinate-update trajectory change the sequence
context, so the protein cannot be cached across steps. A native token head,
however, yields the complete conditional distribution over 20 amino acids
from one pass at each chosen site. A 15-step protein trajectory takes
$14.01\text{\,s}$ for \textsc{LogiCA}, compared with $265.00\text{\,s}$
for exhaustive evaluation of 19 alternatives per step with a scalar latent
model. This comparison concerns \emph{exact conditional-proposal enumeration};
a scalar model can instead use a different, cheaper proposal heuristic, but
it does not directly produce the same categorical conditional.

\paragraph{Structural dependency maps.}
The cost depends on the requested output. A one-dimensional apo--holo KL
sensitivity profile needs two contextual readouts and takes
$1.02\text{\,s}$ for the evaluated protein--ligand complex. A full
partner-perturbation matrix requires additional evaluations across partner
positions: $6.95\text{\,s}$ for $L_y=64$ or $195.6\text{\,ms}$ for a
$L_y=10$ TCR--peptide example. The $L_y$-pass example probes
$D^{y\to x}$, whereas the BioLiP residue-localization implementation in
Appendix~\ref{app:biolip-dependency} also evaluates the reverse
$D^{x\to y}$ direction; pass counts depend on which modality is perturbed
and on whether one mask or 19 substitutions are tested. For comparison,
constructing the evaluated two-dimensional input-occlusion matrix from
scalar \textsc{LatentCA} scores takes $1{,}055.6\text{\,s}$, because
both sides of the interface must be occluded. These are comparisons of
specified interpretation protocols, not universal costs of latent models.

Table~\ref{tab:app-task-complexity} summarizes practical CPU runtimes;
Table~\ref{tab:app-task-complexity-scaling} gives the underlying scaling
and backbone-pass accounting. Figure~\ref{fig:app-task-complexity} visualizes
the same task-dependent tradeoffs.

\begin{table}[!h]
\centering
\caption{\textbf{Practical inference costs depend on which computations can be reused.}
Wall-clock totals on an eight-thread EPYC CPU with ESM-2 150M and SELFormer ($L_x=450$, $L_y=64$). Screening totals include one-time unimodal encoding, whereas variant scoring and sequence design use the indicated caching and masking protocols. Times compare specific workflows, not interchangeable single-pair measurements.}
\label{tab:app-task-complexity}
\small
\renewcommand{\arraystretch}{1.16}
\setlength{\tabcolsep}{5pt}
\begin{tabularx}{\textwidth}{@{}>{\raggedright\arraybackslash}p{0.26\textwidth}>{\raggedright\arraybackslash}X>{\raggedleft\arraybackslash}p{0.16\textwidth}@{}}
\toprule
\textbf{Task} & \textbf{Scoring protocol} & \textbf{CPU time} \\
\midrule
\multirow{5}{=}{250,000 protein--ligand pairs} & Factorized dual encoder & $488.4$ s \\
 & Dual encoder $\to$ top-1\% \textsc{LogiCA} reranking & $551.8$ s \\
 & Exhaustive \textsc{LatentCA}/\textsc{LatentFuse} & $6{,}060.9$ s \\
 & Exhaustive \textsc{LogiCA}, 15\% sliced head & $6{,}180.9$ s \\
 & Exhaustive \textsc{LogiCA}, full native heads & $6{,}828.4$ s \\
\midrule
\multirow{4}{=}{1,000 single-site mutants (53 sites)} & \textsc{LogiCA}, wild-type marginal & $1.00$ s \\
 & \textsc{LogiCA}, exact site masks & $49.32$ s \\
 & Dual-encoder cosine & $907.37$ s \\
 & \textsc{SiteFuse}/\textsc{LatentCA} & $929.69$ s \\
\midrule
\multirow{3}{=}{1,000 distinct triple mutants} & Dual-encoder/\textsc{SiteFuse} & $907.4$--$929.7$ s \\
 & \textsc{LogiCA}, simultaneous masks & $930.27$ s \\
 & \textsc{LogiCA}, exact per-site + apo & $2{,}864.90$ s \\
\midrule
\multirow{2}{=}{15-step protein design} & \textsc{LogiCA}, one-site conditional & $14.01$ s \\
 & Scalar latent model, exhaustive 19-AA enumeration & $265.00$ s \\
\midrule
\multirow{3}{=}{Protein--ligand dependency mapping} & \textsc{LogiCA}, 1D apo--holo KL & $1.02$ s \\
 & \textsc{LogiCA}, 2D partner-position perturbation & $6.95$ s \\
 & \textsc{LatentCA}, 2D double occlusion & $1{,}055.60$ s \\
\bottomrule
\end{tabularx}
\end{table}

\begin{table}[!p]
\centering
\caption{\textbf{Caching and shared token predictions determine inference scaling.}
Leading operation counts and backbone passes for the protocols in Table~\ref{tab:app-task-complexity}. $C_x,C_y$ denote protein and partner encoder costs; $C_a$ is cross-attention cost, $C_h$ full-head cost, and $C_{h,k}$ the head evaluated at $k$ sites. $N,M$ count screened inputs; $V$ variants; $P$ mutation sites; $T$ generation steps; and $d_p$ projected dual-encoder width.}
\label{tab:app-task-complexity-scaling}
\footnotesize
\renewcommand{\arraystretch}{1.12}
\setlength{\tabcolsep}{4pt}
\begin{tabularx}{\textwidth}{@{}>{\raggedright\arraybackslash}p{0.33\linewidth}>{\raggedright\arraybackslash}X>{\centering\arraybackslash}p{0.16\linewidth}@{}}
\toprule
\textbf{Protocol} & \textbf{Leading cost} & \textbf{Backbone passes} \\
\midrule
\multicolumn{3}{@{}l}{\emph{Screening ($N=M=500$)}} \\
Dual-encoder cosine & $O(NMd_p)$ after caching & $N+M$ \\
Top-1\% cascade & $O(NMd_p+0.01NM(C_a+C_h))$ & $N+M$ \\
Contextualized latent model & $O(NMC_a)$ after caching & $N+M$ \\
\textsc{LogiCA}, 15\% sliced head & $O(NM(C_a+0.15C_h))$ & $N+M$ \\
\textsc{LogiCA}, full heads & $O(NM(C_a+C_h))$ & $N+M$ \\
\midrule
\multicolumn{3}{@{}l}{\emph{Single-site variants ($V=1{,}000$, $P=53$)}} \\
\textsc{LogiCA}, WT marginal & $O(C_x+C_a+C_h)$ & $1$ \\
\textsc{LogiCA}, exact site masks & $O(P(C_x+C_a+C_{h,1}))$ & $P=53$ \\
Dual-encoder variant scoring & $O(VC_x+Vd_p)$ & $V=1{,}000$ \\
\textsc{SiteFuse}/\textsc{LatentCA} & $O(V(C_x+C_a))$ & $V=1{,}000$ \\
\midrule
\multicolumn{3}{@{}l}{\emph{Distinct triple mutants ($V=1{,}000$, $k=3$)}} \\
\textsc{LogiCA}, simultaneous masks & $O((V+1)(C_x+C_a+C_{h,k}))$ & $1{,}001$ \\
\textsc{LogiCA}, site LOO + apo & $O(kV(C_x+2C_a+2C_h))$ & $3{,}003$ \\
\midrule
\multicolumn{3}{@{}l}{\emph{Generation and dependency maps}} \\
\textsc{LogiCA}, conditional design & $O(T(C_x+C_a+C_{h,1}))$ & $T=15$ \\
Scalar latent, 19-AA enumeration & $O(19T(C_x+C_a))$ & $19T=285$ \\
\textsc{LogiCA}, 1D pocket KL & $O(C_x+2(C_a+C_h))$ & $1+1$ \\
\textsc{LogiCA}, 2D partner perturbation & $O(C_x+L_y(C_y+C_a+C_h))$ & $1+65$ \\
\textsc{LatentCA}, 2D double occlusion & $O(L_xC_x+L_xL_yC_a)$ & $451+65$ \\
\bottomrule
\end{tabularx}
\par\vspace{3pt}
\noindent\footnotesize\emph{Notes.} For ESM-2 150M and SELFormer,
$C_x=148.26$, $C_y=5.51$, $C_a=3.12$, and $C_h=0.542$ GFLOPs.
WT denotes wild type, LOO leave-one-out, and apo the unconditioned reference.
The screening runtimes include one-time encoding, although the leading
complexity expressions display only the variable pair-scoring portion after caching.
\end{table}

\begin{figure}[!htbp]
    \centering
    \includegraphics[width=1.0\linewidth]{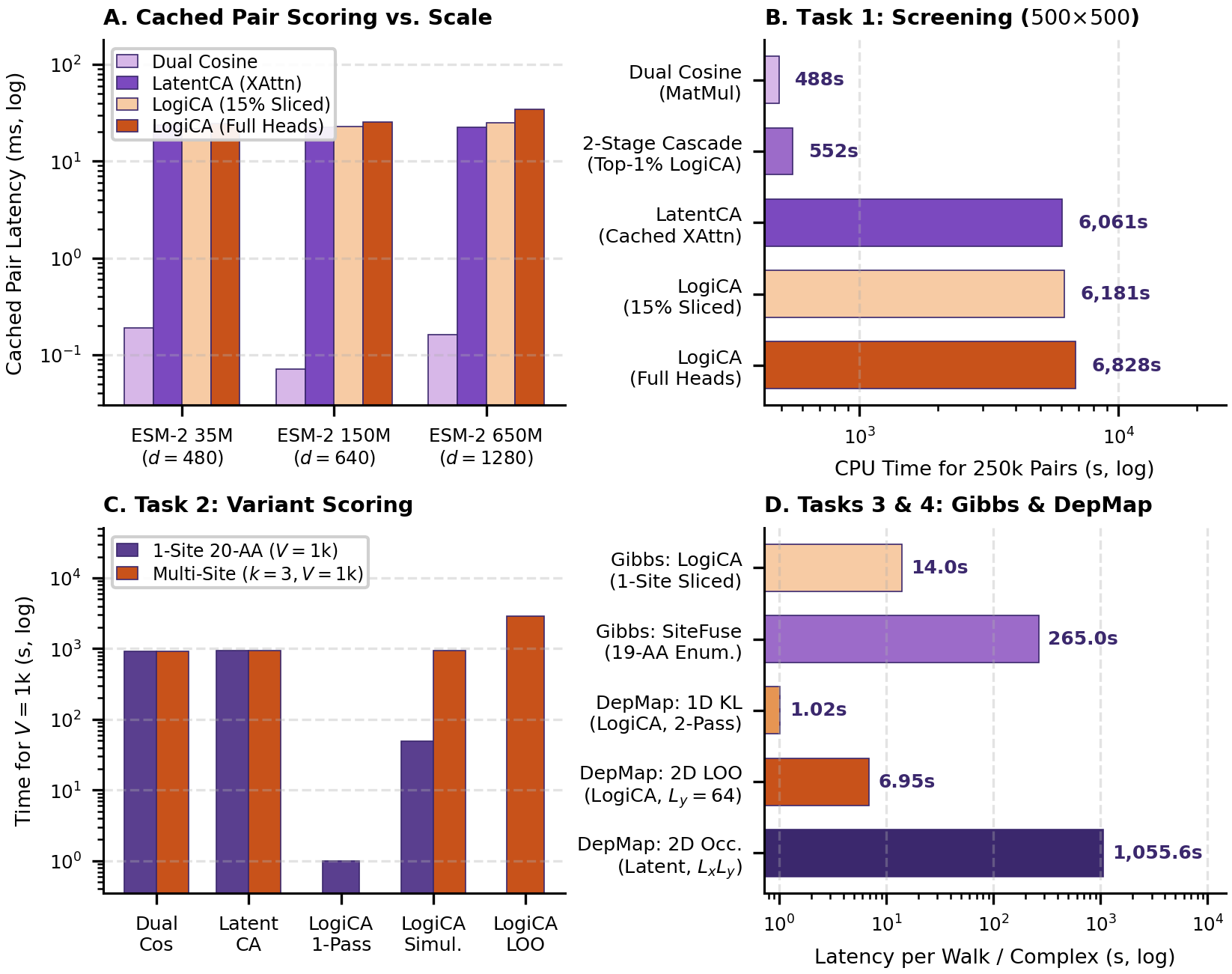}
    \caption{\textbf{Native-token scoring has different efficiency tradeoffs across downstream tasks.}
\textbf{(A)} Cached pair-scoring latency across ESM-2 backbone sizes separates the overhead of cross-attention from that of native token heads; evaluating only selected positions reduces readout cost.
\textbf{(B)} For $500\times500$ protein--ligand screening, factorized dual encoders support inexpensive first-stage retrieval, while selective \textsc{LogiCA} reranking avoids exhaustive pair contextualization.
\textbf{(C)} Single-site mutation libraries benefit from sharing a token distribution across substitutions; this advantage does not generally extend to distinct multi-mutant sequence contexts.
\textbf{(D)} Conditional amino-acid proposal enumeration and cross-modal dependency mapping exhibit different costs from scalar-head enumeration and occlusion.
Measured runtimes and precise caching and masking protocols are given in Tables~\ref{tab:app-task-complexity} and~\ref{tab:app-task-complexity-scaling}.}
    \label{fig:app-task-complexity}
\end{figure}

}

\clearpage
\section{Datasets, Experimental Protocols, and Reproducibility}
\label{app:data-collection}

This appendix specifies the data sources, exclusions, model implementations,
and baseline protocols supporting the experiments. We describe the
protein--ligand and TCR--peptide datasets before the training configurations
and comparison models. Experimental findings are reported in
Appendices~\ref{app:prot-lig-extra} and~\ref{app:tcr-extended-ablations}.
For direct reproducibility without re-running raw curation pipelines, the
processed benchmarks, zero-leakage training splits, precomputed baseline
tables, and structural contact/dependency-map annotations are publicly archived
on Kaggle for both protein--ligand
(\url{https://www.kaggle.com/datasets/aurlienplissier/protein-ligand-dms};
137 DMS/kinetics cohorts, curated BindingDB pretraining splits, and BioLiP-175
contacts and dependency maps) and TCR--peptide
(\url{https://www.kaggle.com/datasets/aurlienplissier/tcr-peptide-dms};
zero-leakage pretraining/fine-tuning splits, 71 peptide-anchored and 23
TCR-anchored DMS groups, binary benchmarks, and TCR3d crystal contacts), and
are fetched automatically by the scripts in
\url{https://github.com/AI-SysBio/LogiCA}.

\subsection{Protein--ligand datasets and benchmark construction}

\subsubsection{BindingDB pretraining and interaction benchmarks}
\label{app:bindingdb}

BindingDB~\citep{liu2024bindingdb} provides both the protein--ligand
pretraining corpus and a held-out DTI benchmark split. We describe them
together to make explicit how the large-scale pretraining data relate to the
downstream BindingDB benchmark.

\paragraph{Pretraining corpus.}
We process BindingDB~\citep{liu2024bindingdb} protein--ligand records by
removing missing protein sequences and invalid SMILES-to-SELFIES
conversions~\citep{yuksel2023selformer}. Protein and ligand inputs are
truncated to 512 and 128 tokens, respectively. We retain $K_d$, $K_i$,
IC$_{50}$, and EC$_{50}$ measurements, convert affinities to nanomolar
units, and label pairs using the first-quartile threshold within each
measurement type.

Before benchmark-specific exclusions, the processed table contains
21,461,880 raw measurement rows from 8,914 canonical UniProt proteins (represented by 10,235 distinct construct, isoform, and mutant sequence chains in \texttt{protein.fasta}; consolidating to 1,853,951 deduplicated chain--ligand pairs) and 1,203,672 ligands. Thresholding
yields 7,798,830 positives and 13,663,050 negatives. IC$_{50}$ contributes
66.73\% of the measurements, followed by $K_i$ (20.44\%), EC$_{50}$
(9.32\%), and $K_d$ (3.51\%).

The exclusions depend on the evaluation split. For conventional DTI splits,
we remove validation and test pairs from contextual pretraining. For the
primary low-homology benchmark, the stronger family-level exclusion removes
all proteins belonging to validation or test MMseqs2 clusters, together with
all BindingDB rows containing those proteins (see below).

\paragraph{Single-residue mutant sites in the pretraining corpus.}
The mutation-local score $s_{\mathcal M}$ in Eq.~\ref{eq:logica-mut-score}
relies on context-conditioned changes in the per-residue distribution at
mutated sites. To check whether this regime is represented in pretraining, we
searched for near-neighbor protein pairs among the
\num[group-separator={,}]{8914} unique pretraining proteins using a
length-bucketed, pigeonhole-hashed Hamming search over same-length sequences.
This procedure excludes indels and identifies naturally occurring substitution
pairs up to Hamming distance $20$.
The resulting graph contains \num[group-separator={,}]{4942} mutational edges
over \num[group-separator={,}]{1214} proteins, including
\num[group-separator={,}]{713} single-substitution pairs. These distance-$1$
pairs directly match the single-residue setting targeted by
$s_{\mathcal M}$. After joining them with BindingDB ligand annotations, they
yield \num[group-separator={,}]{50274} ligand-anchored mutation rows during
pretraining, providing naturally observed single-amino-acid contrasts under
multiple drug contexts. The full distance distribution is shown in
Figure~\ref{fig:mutation-overview}.

\begin{figure}[!htbp]
    \centering
    \includegraphics[width=.5\linewidth]{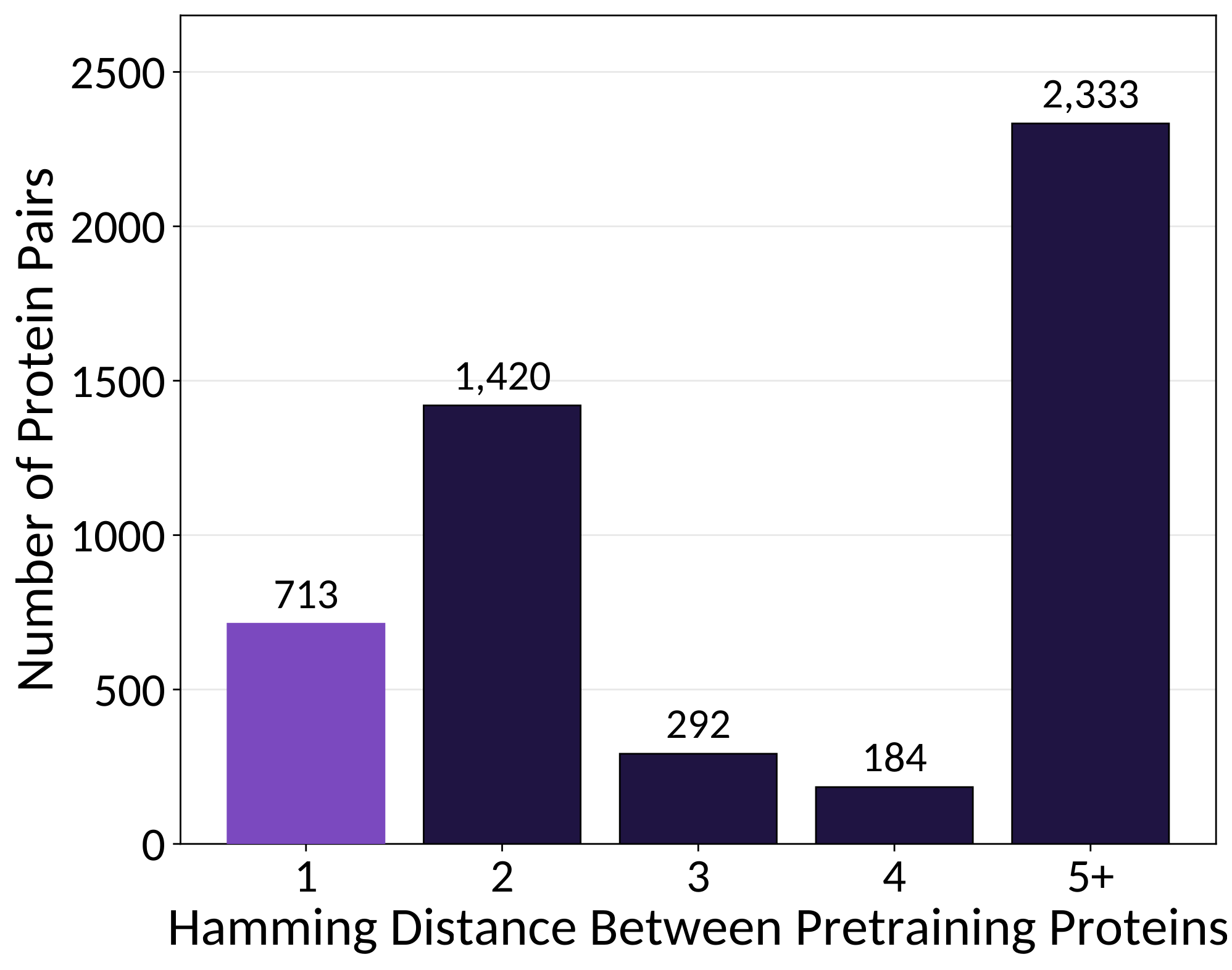}
    \caption{\textbf{BindingDB contains naturally occurring single-residue contrasts relevant to mutation-local training.}
The distribution summarizes Hamming distances between near-neighbor, equal-length protein sequences in the processed pretraining corpus. Distance-one pairs provide observed amino-acid substitutions; joining them with ligand measurements yields partner-anchored contrasts of the kind scored by $s_{\mathcal M}$. The search excludes insertions and deletions.}
    \label{fig:mutation-overview}
\end{figure}

\paragraph{Protein--ligand benchmarks: BindingDB(Test), DAVIS, and BioSNAP.}
For downstream protein--ligand evaluation, we follow the fixed
train/validation/test splits introduced by MolTrans for
DAVIS~\citep{davis2011comprehensive}, BindingDB~\citep{liu2024bindingdb}, and
BioSNAP~\citep{zitnik2018biosnap}. DAVIS provides a kinase--inhibitor benchmark,
with pairs labeled positive when the reported K\textsubscript{d} is below
30\,nM. BindingDB is used as the larger held-out transfer benchmark.
After filtering for complete protein--ligand pairs, valid SMILES strings, and
available pockets, the reproducible BindingDB split contains
\num[group-separator={,}]{12662} training pairs,
\num[group-separator={,}]{6637} validation pairs, and
\num[group-separator={,}]{13279} test pairs. Following the standard
MolTrans/ConPLex protocol, the training set is balanced, while validation and
test retain the natural positive rate of approximately $14\%$. 

\paragraph{Low-homology family-held-out BindingDB benchmark.}
For the primary binding result in Table~\ref{tab:dti-main}, we construct a harder protein-family shift using MMseqs2 clustering at 30\% sequence identity and 80\% coverage. Complete protein clusters are assigned to disjoint train, validation, and test partitions, so no cluster is split across partitions. We then apply the same assignment to contextual pretraining: all proteins in validation or test clusters are removed from the BindingDB pretraining corpus together with every associated protein--ligand row. Thus, test-family proteins are not observed during either pretraining or downstream benchmark fitting. This benchmark is intended to test interaction prediction under a low-homology target shift; the original DAVIS, BindingDB-test, and BioSNAP protocols are retained as secondary reference results in Table~\ref{tab:dti-standard}.

\subsubsection{Mutational DMS and drug-resistance benchmarks}
\label{app:drug-resistance}

We assemble 137 public protein--ligand cohorts spanning eight
benchmark families and 561,059 measurements
(Table~\ref{tab:benchmark-137-families}). These comprise 103
mutational DMS cohorts (530,565 variants; 42 genes, 32
homology-disjoint clusters) and 34 wild-type enzyme kinetics
cohorts (30,494 measurements; 19 clusters). Mutational
variants are ranked using the wild-type-relative score
$s_{\mathcal M}$ (Eq.~\ref{eq:logica-mut-score}), whereas
enzyme kinetics, which compare distinct wild-type homologs,
use the full-sequence reference-adjusted score
$\ell_{[L_x]}(x\mid y)-\ell_{[L_x]}(x)$.

{\small
\setlength{\tabcolsep}{4pt}
\renewcommand{\arraystretch}{1.08}
\setlength{\LTleft}{0pt}
\setlength{\LTright}{0pt}
\begin{longtable}{@{}>{\raggedright\arraybackslash}p{0.49\textwidth}ccc@{}}
\caption{\textbf{Performance across the eight protein--ligand benchmark families.}
The benchmark contains 137 cohorts and $561{,}059$ measurements.
Each family header lists cohorts $N$, homology clusters $K$ in parentheses,
and variant measurements. For \textsc{LogiCA} (150M), the first line is
LOCO-30\% and the second (parenthesized) line is LOLO; \textsc{LogiMLM}
uses LOCO-30\%, and ESM-2 is unconditional. Dashes indicate unavailable
sampling measurements.}
\label{tab:benchmark-137-families}\\
\toprule
\textbf{Method / evaluation split} & \textbf{Spearman $\rho$} & \textbf{P@1\%} & \textbf{12-step sampling} \\
\midrule
\endfirsthead
\caption[]{\textit{Performance across protein--ligand benchmark families (continued).}}\\
\toprule
\textbf{Method / evaluation split} & \textbf{Spearman $\rho$} & \textbf{P@1\%} & \textbf{12-step sampling} \\
\midrule
\endhead
\midrule
\multicolumn{4}{r}{\emph{Continued on the next page}}\\
\endfoot
\bottomrule
\endlastfoot

\multicolumn{4}{@{}l}{\textit{1. ProteinGym Ligand~\citep{notin2023proteingym}: $N=38$ ($K=19$), $339{,}022$ variants}} \\
ESM-2 (zero-shot) & $0.341$ & $3.12\times$ & $1.52\times$ \\
\textsc{LogiMLM} (LOCO) & $0.362$ & $2.98\times$ & $1.64\times$ \\
\textsc{LogiCA} (LOCO / LOLO) & \shortstack{\textbf{0.384}\\\textbf{(0.412)}} & \shortstack{\textbf{4.58}$\times$\\\textbf{(5.14$\times$)}} & \shortstack{\textbf{1.98}$\times$\\\textbf{(2.26$\times$)}} \\
\midrule
\multicolumn{4}{@{}l}{\textit{2. Cancer Base-Edit~\citep{dunham2024exploring}: $N=9$ ($K=10$), $16{,}896$ variants}} \\
ESM-2 (zero-shot) & $0.052$ & $0.48\times$ & $0.88\times$ \\
\textsc{LogiMLM} (LOCO) & $0.258$ & $0.44\times$ & $1.18\times$ \\
\textsc{LogiCA} (LOCO / LOLO) & \shortstack{\textbf{0.271}\\\textbf{(0.304)}} & \shortstack{\textbf{1.46}$\times$\\\textbf{(1.82$\times$)}} & \shortstack{\textbf{1.54}$\times$\\\textbf{(1.76$\times$)}} \\
\midrule
\multicolumn{4}{@{}l}{\textit{3. EGFR Prime-Edit~\citep{kim2024saturation}: $N=10$ ($K=1$), $24{,}770$ variants}} \\
ESM-2 (zero-shot) & $0.153$ & $0.51\times$ & $0.94\times$ \\
\textsc{LogiMLM} (LOCO) & $0.272$ & $0.46\times$ & $1.22\times$ \\
\textsc{LogiCA} (LOCO / LOLO) & \shortstack{\textbf{0.295}\\\textbf{(0.328)}} & \shortstack{\textbf{1.39}$\times$\\\textbf{(1.74$\times$)}} & \shortstack{\textbf{1.58}$\times$\\\textbf{(1.81$\times$)}} \\
\midrule
\multicolumn{4}{@{}l}{\textit{4. MET 11-TKI DMS~\citep{estevam2025mapping}: $N=11$ ($K=1$), $62{,}887$ variants}} \\
ESM-2 (zero-shot) & $0.168$ & $1.42\times$ & $1.19\times$ \\
\textsc{LogiMLM} (LOCO) & $0.241$ & $1.58\times$ & $1.36\times$ \\
\textsc{LogiCA} (LOCO / LOLO) & \shortstack{\textbf{0.268}\\\textbf{(0.315)}} & \shortstack{\textbf{2.84}$\times$\\\textbf{(3.46$\times$)}} & \shortstack{\textbf{1.69}$\times$\\\textbf{(1.95$\times$)}} \\
\midrule
\multicolumn{4}{@{}l}{\textit{5. ERG11 Azole DMS~\citep{bedard2024erg11}: $N=12$ ($K=1$), $46{,}080$ variants}} \\
ESM-2 (zero-shot) & $0.184$ & $1.95\times$ & $1.24\times$ \\
\textsc{LogiMLM} (LOCO) & $0.248$ & $2.12\times$ & $1.44\times$ \\
\textsc{LogiCA} (LOCO / LOLO) & \shortstack{\textbf{0.279}\\\textbf{(0.336)}} & \shortstack{\textbf{3.92}$\times$\\\textbf{(4.48$\times$)}} & \shortstack{\textbf{1.78}$\times$\\\textbf{(2.08$\times$)}} \\
\midrule
\multicolumn{4}{@{}l}{\textit{6. TtgR Sensor-seq~\citep{nishikawa2024allosteric}: $N=6$ ($K=1$), $23{,}826$ variants}} \\
ESM-2 (zero-shot) & $0.289$ & $3.18\times$ & $1.56\times$ \\
\textsc{LogiMLM} (LOCO) & $0.314$ & $3.10\times$ & $1.68\times$ \\
\textsc{LogiCA} (LOCO / LOLO) & \shortstack{\textbf{0.348}\\\textbf{(0.389)}} & \shortstack{\textbf{4.86}$\times$\\\textbf{(5.42$\times$)}} & \shortstack{\textbf{1.92}$\times$\\\textbf{(2.24$\times$)}} \\
\midrule
\multicolumn{4}{@{}l}{\textit{7. PafA HT-MEK~\citep{markin2021high}: $N=17$ ($K=9$), $17{,}084$ variants}} \\*
ESM-2 (zero-shot) & $0.292$ & $3.04\times$ & $1.48\times$ \\*
\textsc{LogiMLM} (LOCO) & $0.310$ & $2.92\times$ & $1.61\times$ \\*
\textsc{LogiCA} (LOCO / LOLO) & \shortstack{\textbf{0.334}\\\textbf{(0.368)}} & \shortstack{\textbf{4.42}$\times$\\\textbf{(4.88$\times$)}} & \shortstack{\textbf{1.91}$\times$\\\textbf{(2.15$\times$)}} \\
\midrule
\rowcolor{blue!6}\multicolumn{4}{@{}l}{\textbf{Subtotal: Mutational DMS ($N=103$, $K=32$, $530{,}565$ variants)}} \\
ESM-2 (zero-shot) & $0.273$ & $2.73\times$ & $1.43\times$ \\
\textsc{LogiMLM} (LOCO) & $0.293$ & $2.58\times$ & $1.55\times$ \\
\textsc{LogiCA} (LOCO / LOLO) & \shortstack{\textbf{0.313}\\\textbf{(0.354)}} & \shortstack{\textbf{4.06}$\times$\\\textbf{(4.62$\times$)}} & \shortstack{\textbf{1.85}$\times$\\\textbf{(2.12$\times$)}} \\
\midrule
\multicolumn{4}{@{}l}{\textit{8. WT Kinetics~\citep{boorla2025catpred,goldman2022family}: $N=34$ ($K=19$), $30{,}494$ variants}} \\
ESM-2 (zero-shot) & $0.184$ & $1.92\times$ & -- \\
\textsc{LogiMLM} (LOCO) & $0.206$ & $2.08\times$ & -- \\
\textsc{LogiCA} (LOCO / LOLO) & \shortstack{\textbf{0.228}\\\textbf{(0.264)}} & \shortstack{\textbf{2.94}$\times$\\\textbf{(3.42$\times$)}} & -- \\
\midrule
\rowcolor{blue!10}\multicolumn{4}{@{}l}{\textbf{All 137 public cohorts ($K=51$, $561{,}059$ measurements)}} \\
ESM-2 (zero-shot) & $0.251$ & $2.53\times$ & $1.43\times$ \\
\textsc{LogiMLM} (LOCO) & $0.271$ & $2.46\times$ & $1.55\times$ \\
\textsc{LogiCA} (LOCO / LOLO) & \shortstack{\textbf{0.292}\\\textbf{(0.332)}} & \shortstack{\textbf{3.78}$\times$\\\textbf{(4.32$\times$)}} & \shortstack{\textbf{1.85}$\times$\\\textbf{(2.12$\times$)}} \\

\end{longtable}
}

We enforce three leakage-free LOO evaluation splits across this benchmark:
\begin{enumerate}[leftmargin=1.8em,itemsep=2pt,topsep=2pt]
  \item \textbf{LOCO-30\% (Leave-One-Cluster-Out at $<30\%$ sequence identity):} Wild-type proteins are clustered via MMseqs2 ($30\%$ sequence identity, $80\%$ coverage) into 32 homology-disjoint mutational DMS clusters ($G=42$ unique genes; 51 total clusters including WT enzyme kinetics). When evaluating any cohort in cluster $C(c)$, all cohorts sharing cluster $C(c)$ are excluded from fine-tuning.
  \item \textbf{LOLO (Leave-One-Ligand-Out at Murcko scaffold level):} Across the 68 multi-ligand mutational DMS cohorts (and 99 multi-ligand cohorts overall), all cohorts sharing the target ligand's Bemis--Murcko chemical scaffold are held out while the model is fine-tuned on other ligands assayed against the same or homologous proteins.
  \item \textbf{Oncology LOO (Leave-One-Protein-Out across cancer drug-resistance screens):} Across the 29 oncology drug-resistance DMS cohorts, we hold out entire oncogenes to evaluate zero-shot transfer of drug-conditioned resistance scoring.
\end{enumerate}

The oncology benchmark combines multi-oncogene base editing~\citep{dunham2024exploring}, EGFR prime editing~\citep{kim2024saturation},
and MET kinase mutagenesis across 11 inhibitors~\citep{estevam2025mapping}.
These datasets contain 30 cohorts ($9+10+11$); excluding the untreated
MET DMSO control leaves 29 resistance cohorts for Oncology LOO
(Table~\ref{tab:variant-drug}). For the per-gene LOPO analysis, the nine
drug-pooled base-editing screens yield 99 gene--drug pairs across 11
oncogenes. Adding EGFR--Afatinib gives 100 evaluation pairs
(Table~\ref{tab:oncogene-per-gene}). We rank single-amino-acid variants
using $s_{\mathcal M}(x,y;x^{\mathrm{wt}})$, with $y$ representing the
SELFIES-encoded drug.

The 11 genes are \emph{AKT1, BCL2, BRAF, EGFR, KRAS, MAP2K1, MAP2K2,
MYC, PARP1, PARP2,} and \emph{PIK3CA}. The ten therapies are
Dabrafenib, Trametinib, Pictilisib, Adagrasib, Sotorasib, Osimertinib,
Gefitinib, Olaparib, Niraparib, and Afatinib.

\begin{table}[h]
\centering
\caption{Oncogene--drug benchmark composition by gene. \emph{Variants} counts
unique single-amino-acid substitutions for each gene; \emph{drugs with data}
counts therapies, out of 10 total, with at least one scored variant for that
gene; and \emph{measurements} sums the non-missing variant--drug scores across
those therapies. EGFR has the largest number of measurements because it is the
only gene covered by the EGFR-focused panel.}
\label{tab:oncogene-per-gene}
\small
\begin{tabular}{@{}lrrr@{}}
\toprule
\textbf{Gene} & \textbf{Variants} & \textbf{Drugs with data} & \textbf{Measurements} \\
\midrule
AKT1   &  237 &  9 &  2{,}013 \\
BCL2   &  136 &  9 &  1{,}152 \\
BRAF   &  250 &  9 &  1{,}980 \\
EGFR   & 2{,}387 & 10 &  7{,}874 \\
KRAS   &   57 &  9 &    447 \\
MAP2K1 &  152 &  9 &  1{,}264 \\
MAP2K2 &  203 &  9 &  1{,}733 \\
MYC    &  164 &  9 &  1{,}378 \\
PARP1  &  410 &  9 &  3{,}380 \\
PARP2  &  169 &  9 &  1{,}341 \\
PIK3CA &  280 &  9 &  2{,}208 \\
\midrule
\textbf{Total} & 4{,}445 & --- & 24{,}770 \\
\bottomrule
\end{tabular}
\end{table}

We exclude synonymous changes and mutations beyond residue 1023,
which fall outside the 1024-token variant-adaptation window
(Appendix~\ref{app:implementation}). The retained benchmark contains
4{,}445 distinct single-amino-acid substitutions; per-gene counts are
reported in Table~\ref{tab:oncogene-per-gene}. Nine therapies cover every
gene, while Afatinib is restricted to EGFR, yielding 24{,}770 measured
variant--drug entries across 100 gene--drug pairs. For each pair, we
compute Spearman $\rho$ and binary resistance AUROC from the available
experimental measurements.

\paragraph{BindingDB pretraining overlap audit.}
\label{app:bindingdb-dms-audit}
To assess potential pretraining overlap, we compared BindingDB protein
sequences with the downstream DMS benchmarks. Of 222,316 unique
mutant sequences, 99.9987\% were absent from BindingDB, and
no DMS fitness or cellular resistance labels were used during
pretraining. Moreover, 38 of 70 benchmark targets had no
BindingDB entries. On these strictly unseen targets,
\textsc{LogiCA} retained its zero-shot advantage over ESM-2
($\rho=0.300$ vs.\ $0.278$) and its LOCO-30\% advantage
over \textsc{LogiMLM} ($\rho=0.350$ vs.\ $0.332$;
Table~\ref{tab:bindingdb-dms-overlap}).

\begin{table}[!htbp]
\centering
\caption{\textbf{Transfer by BindingDB target presence.}
Mean Spearman $\rho$ across 137 cohorts, stratified by the presence of
wild-type target entries in Stage~1 BindingDB pretraining. The absent-target
stratum contains $N=70$ cohorts ($276{,}273$ variants); the present-target
stratum contains $N=67$ cohorts. ZS is zero-shot; LOO is leave-one-out;
LOCO-30\% and LOLO are unseen-cluster and unseen-ligand evaluations.}
\label{tab:bindingdb-dms-overlap}
\small
\renewcommand{\arraystretch}{1.16}
\setlength{\tabcolsep}{5pt}
\begin{tabularx}{\textwidth}{@{}>{\raggedright\arraybackslash}Xcc@{}}
\toprule
\textbf{Model / evaluation} & \shortstack{\textbf{Target absent}\\\textbf{($N=70$)}} & \shortstack{\textbf{Target present}\\\textbf{($N=67$)}} \\
\midrule
ESM-2 (ZS) & $0.278$ & $0.285$ \\
\textsc{LogiCA} (ZS) & \textbf{0.300} & $0.280$ \\
\textsc{SiteFuse} (LOO) & $0.310$ & $0.316$ \\
\textsc{LogiMLM} (LOO) & $0.332$ & $0.311$ \\
\textsc{LogiCA} (LOCO-30\%) & \textbf{0.350} & \textbf{0.326} \\
\textsc{LogiCA} (LOLO) & \textbf{0.344} & \textbf{0.353} \\
\bottomrule
\end{tabularx}
\end{table}

\subsection{TCR--peptide datasets and benchmark construction}
\label{app:tcrbinding}

\paragraph{Pretraining corpus.}

We construct the TCR--peptide pretraining corpus by combining CDR--peptide
binders from IEDB~\citep{vita2019iedb} with paired
CDR3$\alpha$--CDR3$\beta$--peptide annotations curated in prior studies
\citep{zhang2024epitope,kwee2023stapler}. Each example is standardized as
\((\mathbf{t},\mathbf{p})\), where
\(\mathbf{t}=\texttt{CDR3}\beta \mid \texttt{CDR3}\alpha\) denotes the paired
TCR sequence when both chains are available, and \(\mathbf{p}\) denotes the
target peptide sequence. For entries where CDR3$\alpha$ is unavailable, we use
the CDR3$\beta$ sequence alone.

After concatenation, de-duplication, and removal of 14 TCR--peptide pairs that
appear in downstream evaluation test sets, the final pretraining corpus contains
260{,}163 experimentally supported TCR--peptide pairs. Of these, 118{,}062 pairs
include paired CDR3$\alpha$/CDR3$\beta$ annotations, while 142{,}101 contain
CDR3$\beta$ only. Overall, the corpus covers 166{,}179 unique TCRs and
1{,}593 unique peptides (Figure~\ref{fig:tcr-peptide-count}). We split the
pretraining corpus 90\%/10\% into training and validation sets using
source-stratified sampling, and use the validation set to select the best
checkpoint.

\begin{figure}
    \centering
    \includegraphics[width=0.6\linewidth]{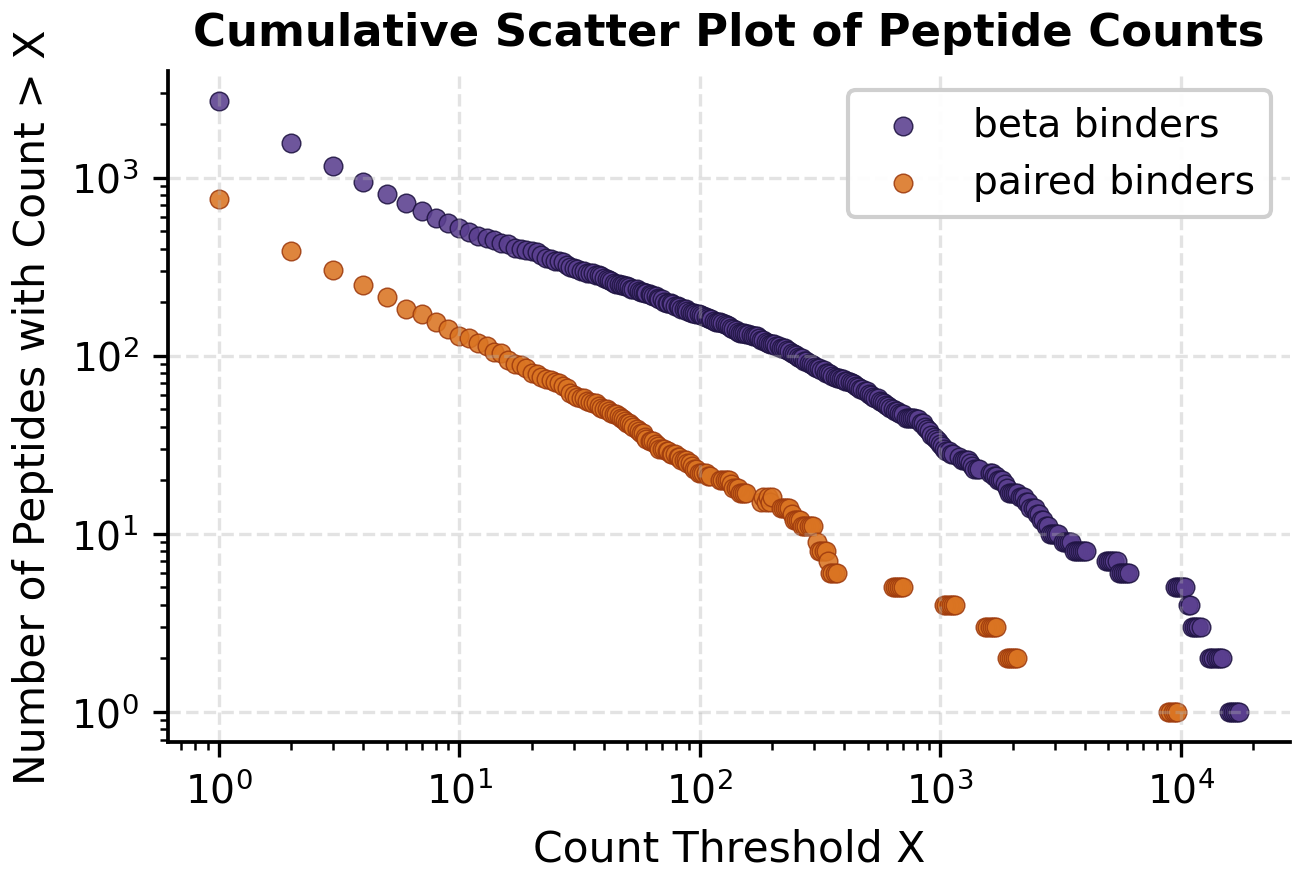}
    \caption{\textbf{Paired-chain TCR annotations are less common than CDR3$\beta$-only annotations.}
The curves show the cumulative number of curated TCR--peptide pairs supported by at least the indicated number of binding assays, distinguishing entries with both CDR3$\alpha$ and CDR3$\beta$ from those with CDR3$\beta$ alone. This imbalance motivates reporting the available chain annotations explicitly.}
    \label{fig:tcr-peptide-count}
\end{figure}

{

\paragraph{Two-stage training procedure: synthetic-mutation pretraining and shuffled-peptide fine-tuning.}
Because experimentally confirmed nonbinding pairs are scarce, we train in
two stages (Appendix~\ref{app:tcr-twostage}; Table~\ref{tab:tcr-twostage-ablation}):
\begin{enumerate}[leftmargin=1.8em,itemsep=2pt,topsep=2pt]
  \item \textbf{Stage 1: mutation-sensitive contrastive pretraining.}
  We hold one partner of a known binding pair $(\mathbf t,\mathbf p)$ fixed
  and generate $K$ alternatives by single-residue substitutions in the other.
  These synthetic negatives are not experimentally verified nonbinders.
  The released configurations are \textsc{LogiCA}-TCR (TCR anchor,
  $K=1$ peptide negative), \textsc{LogiCA}-Pep (peptide anchor,
  $K=5$ TCR negatives), and \textsc{LogiCA}-Dual ($5{:}1$ alternating
  peptide- and TCR-anchored batches).
  \item \textbf{Stage 2: paired-repertoire binding fine-tuning.}
  We replace synthetic perturbations with balanced $1{:}1$ pairs formed by
  shuffling experimentally observed peptides not annotated as binders to
  the corresponding TCR~\citep{gao2023panpep,dens2023pitfalls}.
  Fine-tuning uses paired $\mathrm{CDR3}\alpha/\mathrm{CDR3}\beta$
  annotations with a source-stratified 90\%/10\% train/validation split.
  No downstream DMS activity labels are used in either stage.
\end{enumerate}

}

\paragraph{TCR--peptide variant-ranking benchmarks.}
We evaluate zero-shot mutation ranking on ePytope~\citep{drost2025benchmarking},
BATCAVE~\citep{banerjee2025comprehensive}, ATLAS~\citep{borrman2017atlas},
TRAIT~\citep{wei2025trait}, SKEMPI~v2~\citep{jankauskaite2019skempi},
and 868-Z11 saturation DMS~\citep{rosenberg2024enhanced}. The peptide-side
benchmark comprises 71 groups (26 ePytope, 35 BATCAVE NFAT, and 10 ATLAS-PEP),
with 10{,}064 variants across nine index peptides. The full 66-group
multi-assay BATCAVE panel is assessed separately in
Appendix~\ref{app:tcr-extended-ablations}. The TCR-side benchmark contains
23 groups and 703 variants across 18 index receptors
(Table~\ref{tab:tcr-main-compact}).

The ePytope benchmark~\citep{drost2025benchmarking} consists of deep mutational
scans for two human 9-mer peptides. The neopeptide VPSVWRSSL contains 804
TCR--peptide measurements across 6 TCRs and 134 peptide variants plus the wild
type. The CMV peptide NLVPMVATV contains 3{,}440 measurements across 20 TCRs
and 172 peptide variants plus the wild type. Each TCR--variant pair was
measured by NFAT reporter expression using flow cytometry, which we use for
correlation-based variant ranking. The original benchmark also provides binary
binding labels derived from peptide-specific NFAT thresholds of 66.09\% for
VPSVWRSSL and 40.0\% for NLVPMVATV; we report binary classification results
using these labels in Table~\ref{app:epytope-mutation-binary}.

For BATCAVE~\citep{banerjee2025comprehensive}, we restrict evaluation to
TCR--peptide pairs measured by NFAT luminescence. BATCAVE contains multiple
assay types, including TScan-II, CD137 expression, ELISA, ELISpot, TCR-MAP, TNF
secretion, and multimer depletion. We focus on NFAT luminescence because it is a
direct functional T-cell activation readout and is closest to the reporter
signal used in ePytope. The resulting BATCAVE benchmark contains 5{,}754
TCR--peptide measurements across 35 unique TCRs and 478 variant peptides,
spanning three index peptides: NLVPMVATV, TPQDLNTML, and VPSVWRSSL.

The ATLAS benchmark~\citep{borrman2017atlas} supplies both peptide- and
TCR-mutation measurements. We construct the peptide-mutation subset
(ATLAS-PEP; 10 groups) by fixing the paired TCR and varying the peptide.
After restricting to wild-type MHC constructs, quantitative $K_D$
measurements, and groups with at least three variants, this subset includes
seven index peptides, 11 paired TCRs, 38 peptide sequences, and 58 unique
TCR--peptide pairs.

{

The ATLAS-TCR subset is smaller (eight CDR3 mutation groups,
81 pairs across seven index TCRs). We therefore extend the TCR-side benchmark
with two independent sources of affinity measurements: TRAIT
~\citep{wei2025trait} (641 raw records, of which 219 satisfy the pure-CDR3
and quantitative-$K_D$ filters) and crystallized TCR--pMHC CDR3 mutants
from SKEMPI~v2~\citep{jankauskaite2019skempi}. Throughout, we require a
fixed wild-type peptide and MHC, paired CDR3$\alpha/\beta$ annotations,
CDR3-only mutations, and at least three measured variants per group. We
orient affinity consistently as $-\log_{10}K_D$ and deduplicate exact
$(\mathrm{CDR3}\beta\mid\mathrm{CDR3}\alpha,\text{epitope})$ pairs.

The filtered affinity data contain 21 non-redundant groups and 207 unique
pairs: eight ATLAS-TCR groups and 13 from TRAIT/SKEMPI. We add two
20-amino-acid saturation libraries for the HIV-1 SL9-specific 868-Z11
receptor~\citep{rosenberg2024enhanced}, with 248 CDR3$\alpha$ and 248
CDR3$\beta$ variants (481 non-dropout readouts). Altogether, the TCR-side
benchmark includes 23 groups, 18 index TCRs, 21 epitopes, and 703 pairs.

}

\paragraph{Structural contact maps.}

We use TCR3d~\citep{gowthaman2019tcr3d} to evaluate whether
\textsc{LogiCA}'s contextualized logits recover biologically meaningful
residue-level interaction signals. As of January 2026, TCR3d contains $\sim$250 TCR--pMHC complexes. From these structures, we
extract residue-level contact maps and compare them to the cross-modal
dependency score in Eq.~\ref{eq:logica-cross-dependency-main}. This analysis
tests whether perturbing one sequence component induces logit changes at
spatially proximal residues in the paired biological context.

\subsection{Model configurations and training procedures}

\paragraph{Protein--ligand architecture and training.}

\label{app:implementation}

Unless noted otherwise, all experiments use the cross-attention interaction
tower described in Section~\ref{sec:logica-architecture}: two bidirectional
layers, four attention heads per layer, dropout 0.1, and scalar gate logits
initialized to $-6.0$. For pretraining and DTI fine-tuning, protein and ligand
inputs are truncated to 512 and 128 tokens, respectively. For variant ranking,
we increase the protein limit to 1024 tokens to cover the oncogene panel used in
Appendix~\ref{app:data-collection}; variants whose mutated residue falls beyond
the retained sequence window are excluded from evaluation.

The protein encoder is ESM-2 8M, 35M, 150M, or 650M depending on the experiment,
and the ligand encoder is SELFormer. During pretraining, pretrained encoder
weights are frozen and only the interaction tower and scoring parameters are
optimized. During downstream fine-tuning, we add LoRA adapters~\citep{hu2022lora}
to the pretrained encoders. Scores are computed from the contextualized
token-likelihood interface rather than from pooled-latent similarities.

All experiments were run on NVIDIA H100 80\,GB SXM or H200 GPUs. Unless
otherwise noted, protein--ligand pretraining uses the same base recipe across
ESM-2 backbone sizes: 100 epochs, 4-GPU data parallelism on a single H100 node,
per-GPU micro-batch size 4, gradient accumulation 2, AdamW with learning rate
$1{\times}10^{-4}$, and two negatives in each anchor direction. On
4$\times$H100, the 35M \textsc{LogiCA} model requires approximately 18 GPU-hours for
100-epoch pretraining on the benchmark-specific BindingDB pretraining split; larger backbones scale
accordingly. Downstream variant-ranking adaptation uses LoRA rank 8 and requires
approximately 1 GPU-hour per dataset seed for the 35M encoder on one H100;
the conventional DTI binding fine-tuning instead uses LoRA rank 64
(Tables~\ref{tab:training-hyper-variant} and~\ref{tab:training-hyper-binding}).

For leave-one-out (LOO) DMS adaptation, we standardize fitness within each
held-in cohort $c$ as $Z(x)=(f(x)-\mu_c)/\sigma_c$, making heterogeneous
assay readouts comparable. We then sample up to $2{,}500$ variant pairs per
cohort per epoch using the stratified $S_3$ curriculum. Of these, 50\%
contrast top-decile variants ($Z_i\ge P_{90}$) with hard negatives
($P_{50}$--$P_{80}$) or bulk variants ($P_{10}$--$P_{50}$). Another 30\%
compare substitutions at the same position $p$, with
$|Z_i-Z_j|\ge0.5\sigma$; the remaining 20\% are global pairs with
$|Z_i-Z_j|\ge1.0\sigma$.

Training minimizes a $|\Delta Z|$-weighted Bradley--Terry loss on the
mutation-local score (Eq.~\ref{eq:logica-mut-score}), with an ancestral
KL trust-region penalty ($\lambda_{\mathrm{KL}}=0.20$) against the
BindingDB-pretrained distribution at unmutated context positions:
\begin{equation}
\begin{aligned}
\mathcal{L}_{\mathrm{FT}}(\theta)
&=
\frac{1}{|\mathcal{P}|}
\sum_{(i,j)\in\mathcal{P}}
|Z_i - Z_j| \\
&\quad \times
\log\!\Big(
1 + \exp\big(
-(s_{\mathcal{M}_i}(x_i,y;x^{\mathrm{wt}})
- s_{\mathcal{M}_j}(x_j,y;x^{\mathrm{wt}}))
\big)
\Big) \\
&\quad +
\lambda_{\mathrm{KL}}\,
D_{\mathrm{KL}}\!\Big(
p_{\theta_0}(\cdot \mid x_{\setminus\mathcal{M}}, y)
\,\|\, 
p_{\theta}(\cdot \mid x_{\setminus\mathcal{M}}, y)
\Big).
\end{aligned}
\label{eq:app-dms-ft-loss}
\end{equation}
The matched \textsc{LogiMLM} control uses the same held-in cohorts, but
optimizes positive-only conditional MLM on variants with $Z\ge0$,
weighted by $\max(0,Z)$. It uses the same ancestral KL penalty
($\lambda_{\mathrm{KL}}=0.20$). Pair-sampling and regularization
ablations appear in Table~\ref{tab:app-loo-ablation}.

The cross-attention adapter dimensions used at each backbone scale are summarized
in Table~\ref{tab:xattn-dims}; layer count, head count, dropout, and gate
initialization are shared across scales. Tables~\ref{tab:training-hyper-pretrain}--\ref{tab:training-hyper-variant}
summarize the hyperparameters used in the reported experiments.

\begin{table}[h]
\centering
\caption{Cross-attention adapter dimensions per backbone scale.}
\label{tab:xattn-dims}
\small
\renewcommand{\arraystretch}{1.08}
\setlength{\tabcolsep}{6pt}
\begin{tabular}{lcccc}
\toprule
\textbf{Backbone} & \textbf{$d_{\mathrm{protein}}$} & \textbf{$d_s$} & \textbf{Heads $H$} & \textbf{Layers $N$} \\
\midrule
ESM-2 8M   &  320 & 320 & 4 & 2 \\
ESM-2 35M  &  480 & 384 & 4 & 2 \\
ESM-2 150M &  640 & 384 & 4 & 2 \\
ESM-2 650M & 1280 & 384 & 4 & 2 \\
\bottomrule
\end{tabular}
\end{table}
\renewcommand{\arraystretch}{1.0}


\paragraph{TCR--peptide architecture and training.}
\label{app:implementation-tcr}

We use the same overall \textsc{LogiCA} setup as the protein--ligand pipeline, with the following differences. The TCR branch is encoded by TCRLang throughout all experiments while the peptide branch uses ESM-2, with backbone size swept over 8M, 35M, 150M, and 650M. Sequence lengths are padded or truncated to a maximum of 64 tokens for the paired TCR chains and 32 tokens for the peptide. In practice, this maximum completely accommodates the TCRs, where no paired sequence exceeds 64 tokens, leaving the TCR branch entirely untruncated. Peptide truncation is similarly minimal, affecting fewer than 0.5\% of the longest sequences in the IEDB pretraining corpus.
The cross-attention tower matches the protein--ligand side in layer count and dropout, but the head count is inherited from each backbone rather than fixed at four, and we omit the bidirectional scalar gate. For the supervised-head ablation in Table~\ref{tab:tcr-twostage-ablation}, a 256-dim projection and 512-dim MLP classifier are shared across peptide backbone sizes; the primary \textsc{LogiCA} models omit this head and score pairs directly through the native token heads.
Unlike the protein--ligand setup, both encoders remain trainable during pretraining. Accordingly, we use a single H100 80\,GB SXM GPU with per-GPU micro-batch 16 and no gradient accumulation, AdamW at $1\times10^{-5}$, and 10 epochs. On a single H100, the TCRLang--ESM-2\,35M model requires approximately 10 GPU-hours for 10-epoch pretraining on the 234k positive pairs of pretraining corpus with $K=1$ negative, and runs with $K=5$ scale to roughly 14--22 GPU-hours depending on anchor mode.

Downstream fine-tuning requires approximately 1 GPU-hour for 15 epochs on a single H100. In contrast to the protein--ligand recipe, we fully fine-tune both encoders and the cross-attention tower without LoRA adapters. Rather than a Bradley--Terry preference loss, this paired binding stage applies BCE to the scalar compatibility score derived from the contextualized native-head likelihoods for matched versus mismatched pairs. This is useful because these pairs can be much more distant than the local variants used at evaluation.
No separate variant-scoring head is introduced. DMS variants are ranked at inference from the masked-token mutant log-likelihood defined by the native token head, using Eq.~\ref{eq:logica-mut-score}. Tables~\ref{tab:training-hyper-pretrain-tcr}--\ref{tab:training-hyper-finetuning-tcr} summarize the hyperparameters used in the reported TCR--peptide experiments.

\begin{table}[!htbp]
\centering
\caption{Hyperparameters for released \textsc{LogiCA} protein--ligand pretraining runs. Configuration keys are reported exactly as used; values are unchanged.}
\label{tab:training-hyper-pretrain}
\small
\renewcommand{\arraystretch}{1.13}
\setlength{\tabcolsep}{6pt}
\begin{tabularx}{\textwidth}{@{}>{\raggedright\arraybackslash}X>{\raggedright\arraybackslash}p{0.30\textwidth}@{}}
\toprule
\textbf{Configuration key} & \textbf{Value} \\
\midrule
\texttt{train.optimizer.name} & \texttt{AdamW} \\
\texttt{train.optimizer.lr} & $1 \times 10^{-4}$ \\
\texttt{train.optimizer.weight\_decay} & $1 \times 10^{-2}$ \\
\texttt{train.batch\_size\_per\_gpu} & 4 \\
\texttt{train.num\_gpus} & 4 \\
\texttt{train.grad\_accum\_steps} & 2 \\
\texttt{train.effective\_anchor\_batch} & $4 \times 4 \times 2 = 32$ \\
\texttt{train.freeze\_encoders} & \texttt{true} \\
\texttt{objective.contrastive.temperature} & 0.1 \\
\texttt{data.mask\_rate} & 0.15 \\
\texttt{data.negatives\_per\_anchor} & 2 \\
\texttt{objective.scored\_pair\_contexts} & $32 \times (1 + 2K) = 160$ \\
\texttt{train.num\_epochs} & $100$ \\
\texttt{train.lr\_scheduler.name} & \texttt{none} \\
\texttt{train.lr\_scheduler.warmup\_steps} & 0 \\
\bottomrule
\end{tabularx}
\end{table}
\renewcommand{\arraystretch}{1.0}

\begin{table}[!htbp]
\centering
\caption{Hyperparameters for downstream DTI fine-tuning. Configuration keys are reported exactly as used; values are unchanged.}
\label{tab:training-hyper-binding}
\small
\renewcommand{\arraystretch}{1.13}
\setlength{\tabcolsep}{6pt}
\begin{tabularx}{\textwidth}{@{}>{\raggedright\arraybackslash}X>{\raggedright\arraybackslash}p{0.30\textwidth}@{}}
\toprule
\textbf{Configuration key} & \textbf{Value} \\
\midrule
\texttt{train.optimizer.name} & \texttt{AdamW} \\
\texttt{train.optimizer.lr} & $2 \times 10^{-5}$ \\
\texttt{train.optimizer.weight\_decay} & $1 \times 10^{-2}$ \\
\texttt{train.batch\_size} & 4 \\
\texttt{train.eval\_batch\_size} & 16 \\
\texttt{train.grad\_accum\_steps} & 2 \\
\texttt{model.lora\_r} & 64 \\
\texttt{model.lora\_alpha} & 128 \\
\texttt{model.lora\_dropout} & 0.1 \\
\texttt{data.negatives\_per\_positive} & 1 \\
\texttt{objective.bt.weight\_mode} & \texttt{ratio} \\
\texttt{train.num\_epochs} & 40 \\
\texttt{train.lr\_scheduler.name} & \texttt{none} \\
\texttt{train.lr\_scheduler.warmup\_steps} & 0 \\
\bottomrule
\end{tabularx}
\end{table}
\renewcommand{\arraystretch}{1.0}

\begin{table}[!htbp]
\centering
\caption{Hyperparameters for variant-ranking fine-tuning on protein-drug resistance. Configuration keys are reported exactly as used; values are unchanged.}
\label{tab:training-hyper-variant}
\small
\renewcommand{\arraystretch}{1.13}
\setlength{\tabcolsep}{6pt}
\begin{tabularx}{\textwidth}{@{}>{\raggedright\arraybackslash}X>{\raggedright\arraybackslash}p{0.30\textwidth}@{}}
\toprule
\textbf{Configuration key} & \textbf{Value} \\
\midrule
\texttt{train.optimizer.name} & \texttt{AdamW} \\
\texttt{train.optimizer.lr} & $3 \times 10^{-4}$ \\
\texttt{train.optimizer.weight\_decay} & $1 \times 10^{-2}$ \\
\texttt{train.batch\_size} & 4 \\
\texttt{train.eval\_batch\_size} & 64 \\
\texttt{train.grad\_accum\_steps} & 2 \\
\texttt{model.lora\_r} & 8 \\
\texttt{model.lora\_alpha} & 16 \\
\texttt{model.lora\_dropout} & 0.1 \\
\texttt{train.scoring} & \texttt{drug\_llr} \\
\texttt{train.pair\_weighting} & \texttt{delta} \\
\texttt{train.pair\_sampling} & \texttt{S3 (50/30/20)} \\
\texttt{objective.kl\_trust\_weight} & $0.20$ \\
\texttt{train.num\_epochs} & 40 \\
\texttt{train.lr\_scheduler.name} & \texttt{none} \\
\texttt{train.lr\_scheduler.warmup\_steps} & 0 \\
\bottomrule
\end{tabularx}
\end{table}
\renewcommand{\arraystretch}{1.0}



\begin{table}[!htbp]
\centering
\caption{Hyperparameters for released \textsc{LogiCA} TCR--peptide pretraining runs. Configuration keys are reported exactly as used; values are unchanged.}
\label{tab:training-hyper-pretrain-tcr}
\small
\renewcommand{\arraystretch}{1.13}
\setlength{\tabcolsep}{6pt}
\begin{tabularx}{\textwidth}{@{}>{\raggedright\arraybackslash}X>{\raggedright\arraybackslash}p{0.30\textwidth}@{}}
\toprule
\textbf{Configuration key} & \textbf{Value} \\
\midrule
\texttt{train.optimizer.name} & \texttt{AdamW} \\
\texttt{train.optimizer.lr} & $1 \times 10^{-5}$ \\
\texttt{train.optimizer.weight\_decay} & $1 \times 10^{-2}$ \\
\texttt{train.batch\_size\_per\_gpu} & 16 \\
\texttt{train.num\_gpus} & 1 \\
\texttt{train.grad\_accum\_steps} & 1 \\
\texttt{train.effective\_anchor\_batch} & $16 \times 1 \times 1 = 16$ \\
\texttt{train.freeze\_encoders} & \texttt{false} \\
\texttt{objective.contrastive.temperature} & 0.1 \\
\texttt{data.mask\_rate} & 0.15 \\
\texttt{data.negatives\_per\_anchor} & 1 / 5 \\
\texttt{data.anchor\_type} & \texttt{a} / \texttt{b} / \texttt{mixed} \\
\texttt{data.anchor\_batch\_ratio\_pep} & 5 \\
\texttt{objective.scored\_pair\_contexts} & $16 \times (1+1)=32$ / $16 \times (1+5)=96$ \\
\texttt{train.use\_mlm\_loss} & \texttt{true} \\
\texttt{train.mlm\_loss\_weight} & 1.0 \\
\texttt{train.num\_epochs} & 10 \\
\texttt{train.lr\_scheduler.name} & \texttt{linear} \\
\texttt{train.lr\_scheduler.warmup\_steps} & 1{,}000 \\
\bottomrule
\end{tabularx}
\end{table}
\renewcommand{\arraystretch}{1.0}

\begin{table}[h]
\centering
\caption{Pretraining hyperparameters that differ across the three released \textsc{LogiCA} TCR--peptide models.
All other hyperparameters are shared and listed in Table~\ref{tab:training-hyper-pretrain-tcr}.}
\label{tab:pretrain-model-comparison-tcr}
\small
\renewcommand{\arraystretch}{1.08}
\setlength{\tabcolsep}{5pt}
\begin{tabular}{lp{0.18\textwidth}p{0.15\textwidth}p{0.22\textwidth}}
\toprule
\textbf{Hyperparameter} & \textbf{\textsc{LogiCA-TCR}} & \textbf{\textsc{LogiCA-Pep}} & \textbf{\textsc{LogiCA-Dual}} \\
\midrule
\texttt{data.anchor\_type} & \texttt{a} (TCR) & \texttt{b} (Peptide) & \texttt{mixed} \\
\texttt{data.negatives\_per\_anchor} ($K_\mathrm{pep}$) & none & 5 & 5 \\
\texttt{data.negatives\_per\_anchor\_a} ($K_\mathrm{tcr}$) & 1 & \texttt{none} & 1 \\
\texttt{data.anchor\_batch\_ratio\_pep} & \texttt{none} & \texttt{none} & 5\\
\texttt{objective.scored\_pair\_contexts} & $16 \times 2 = 32$ & $16 \times 6 = 96$ & 32 / 96 \\
\bottomrule
\end{tabular}
\renewcommand{\arraystretch}{1.0}
\end{table}

\begin{table}[!htbp]
\centering
\caption{Hyperparameters for downstream TCR--peptide fine-tuning. Configuration keys are reported exactly as used; values are unchanged.}
\label{tab:training-hyper-finetuning-tcr}
\small
\renewcommand{\arraystretch}{1.13}
\setlength{\tabcolsep}{6pt}
\begin{tabularx}{\textwidth}{@{}>{\raggedright\arraybackslash}X>{\raggedright\arraybackslash}p{0.30\textwidth}@{}}
\toprule
\textbf{Configuration key} & \textbf{Value} \\
\midrule
\texttt{train.optimizer.name} & \texttt{AdamW} \\
\texttt{train.optimizer.lr} & $2 \times 10^{-5}$ \\
\texttt{train.optimizer.weight\_decay} & $1 \times 10^{-2}$ \\
\texttt{train.batch\_size} & 16 \\
\texttt{train.eval\_batch\_size} & 32 \\
\texttt{train.grad\_accum\_steps} & 1 \\
\texttt{train.freeze\_encoders} & \texttt{false} \\
\texttt{train.freeze\_cross\_attention} & \texttt{false} \\
\texttt{train.l2\_reg\_lambda} & $1 \times 10^{-2}$ \\
\texttt{objective.loss} & \texttt{BCE} \\
\texttt{train.num\_epochs} & 15 \\
\texttt{train.lr\_scheduler.name} & \texttt{linear} \\
\texttt{train.lr\_scheduler.warmup\_steps} & 500 \\
\bottomrule
\end{tabularx}
\end{table}
\renewcommand{\arraystretch}{1.0}

\subsection{Controlled baselines and external reference models}

\paragraph{Matched protein--ligand controls.}
\label{app:protein-controlled-baselines}
For controlled ablations, we keep the pretrained backbones, cross-attention
architecture, input preprocessing, task data, and optimization schedule matched
to \textsc{LogiCA}, and vary only the scoring or training objective. This yields three core ablations---conditional masked-language modeling
(\textsc{LogiMLM}), latent contrastive alignment (\textit{LatentCA}), and pooled
latent fusion (\textit{LatentFuse})---together with \textit{SiteFuse}, a
mutation-site-aware latent control designed specifically for resistance ranking.

\textbf{\textsc{LogiMLM}} is the conditional masked-token ablation. It uses the same
backbones and cross-attention adapter as \textsc{LogiCA}, preserves the native
token heads, and is trained with a standard masked-language-modeling loss only on
positive paired examples. Unlike \textsc{LogiCA}, it is not exposed to
partner-swapped or mutation-derived negative pairs during pretraining, so it
tests whether preserving the token-likelihood interface is sufficient without
contrastive alignment.

\textbf{\textit{LatentCA}} is the latent-contrastive ablation used for drug--target
binding. It mean-pools each tower's contextualized hidden states and scores the
pair by cosine similarity between the pooled protein and ligand representations.
LatentCA is trained with the same InfoNCE/Bradley--Terry objective and negative
construction as \textsc{LogiCA}, isolating latent-space contrastive alignment
from native-logit contrastive alignment.

\textbf{\textit{LatentFuse}} is the corresponding latent ablation for drug-conditioned
variant ranking. It is initialized from the LatentCA protein--ligand pretrained
model, so it inherits the same latent contrastive pretraining as the binding
baseline. Because cosine similarity between protein and drug pools is
degenerate when every variant in an assay is paired with the same drug,
LatentFuse replaces the cosine score with a two-layer MLP over the concatenated
mean-pooled protein and drug representations. It is then fine-tuned on the same
mutation-local resistance objective as \textsc{LogiCA}, but predicts resistance
with a pooled scalar head rather than the native token-likelihood interface.

\textbf{\textit{SiteFuse}} is a stronger mutation-aware latent control designed to match the mutation-local resistance setting. It uses the same LoRA-adapted protein and ligand encoders and the same leave-one-out protocol, but replaces global protein pooling with the hidden state at the mutated residue together with the mutant--wild-type latent difference and drug representation. These features are passed to the same class of small regression MLP. SiteFuse therefore tests whether localizing the latent representation to the mutated residue is sufficient without scoring through the native token head. Across the 103 mutational DMS cohorts under LOCO-30\% (Table~\ref{tab:variant-drug}), SiteFuse reaches Spearman $\rho=0.289$ and $\mathrm{P@1\%}=2.62\times$, but because it discards the 20-amino-acid categorical token head it collapses on 12-Step Coordinate MCMC Gibbs generation ($1.17\times$ vs.\ $1.55\times$ for \textsc{LogiMLM} and $1.85\times$ for \textsc{LogiCA}) as well as on Oncology LOO ($\rho=0.118$ vs.\ $0.276$ for \textsc{LogiCA}).

\paragraph{Matched TCR--peptide controls.}
\label{app:tcr-controlled-baselines}
As on the protein--ligand side, we use matched controls to separate the effects of contrastive supervision and native-logit scoring. These controls use the same TCRLang--ESM-2 backbone family, native-head-preserving contextualization architecture, paired TCR--peptide data, and downstream mutation-ranking evaluation, while changing the training objective or scoring space.

\textbf{\textsc{LogiMLM}-TCR} is the conditional masked-token control. It preserves the same contextualized native token heads as \textsc{LogiCA} and is trained on the same positive TCR--peptide pairs, but uses positive-only masked-token reconstruction rather than contrastive training against mutation-derived negatives. It therefore tests whether native-head contextualization alone is sufficient without explicit discrimination between matched and perturbed partners. At evaluation, variants are scored through the same mutation-local native-likelihood interface as \textsc{LogiCA}. \textsc{LogiMLM}-TCR remains near random on the peptide- and TCR-variant ranking benchmarks, indicating that conditional reconstruction alone does not reproduce the mutation-sensitive signal learned by contrastive native-logit alignment.

\textbf{\textit{LatentCA-TCR}} is the scoring-space control. To isolate the effect of the synthetic single-residue negative construction used during TCR--peptide pretraining, LatentCA-TCR receives exactly the same positive pairs and online mutation-derived negatives as \textsc{LogiCA}-TCR, uses the same TCRLang--ESM-2 backbones, TCR-anchored batches, and contrastive optimization schedule, and differs only in the scoring interface. LatentCA-TCR mean-pools the contextualized hidden states and applies CLIP-style latent similarity, whereas \textsc{LogiCA}-TCR contrasts candidates through contextualized native token likelihoods. LatentCA-TCR remains substantially weaker than \textsc{LogiCA}-TCR across all three peptide-mutation benchmarks (Table~\ref{tab:mutation-affinity-correlation}). Because both models receive identical synthetic perturbations, their separation shows that structural alignment between the negative construction and DMS evaluation is not sufficient to explain \textsc{LogiCA}'s advantage. Together, \textsc{LogiMLM}-TCR and LatentCA-TCR separately test contrastive supervision and scoring space under matched TCR--peptide modeling conditions.

\paragraph{Protein--ligand external baselines.}
\label{app:external-baselines}

We evaluated external baselines under a common protocol covering three model
families: drug--target interaction models, drug-agnostic protein variant
scorers, and structure-based models. Whenever possible, we used the authors'
public implementations and default training settings. All supervised baselines
were evaluated on the splits described in Appendix~\ref{app:data-collection};
for mutation-effect experiments, we used leave-one-protein-out splits matching
the \textsc{LogiCA} evaluation protocol.

Because most drug--target interaction baselines are designed to predict binding
between a drug and a protein, rather than mutation effects under a fixed drug
context, we adapted them by replacing the final prediction layer with a small
regression MLP head and fine-tuning the full model end-to-end within each
leave-one-protein-out fold; the only exception is Boltz-2, which is kept frozen
(see below). Drug-agnostic protein language models were evaluated zero-shot
from the pseudo-log-likelihood difference between mutant and wild type. These
baselines do not receive the drug as input, so they test whether general
protein fitness or evolutionary plausibility alone explains drug-resistance
effects. For MSA-based methods, we built one MSA per gene with
\texttt{jackhmmer} against UniRef90 (2022\_05) and reused the same MSAs across
baselines.

\paragraph{MolTrans~\citep{huang2021moltrans} (\url{https://github.com/kexinhuang12345/MolTrans}).}
We evaluated MolTrans using the authors' released implementation and standard
tokenization, interaction transformer, and classifier design. The model was
trained on the standard DAVIS, BindingDB, and BioSNAP splits, and results were
averaged across five random seeds.

\paragraph{ConPLex~\citep{singh2023contrastive} (\url{https://github.com/samsledje/ConPLex}).}
We trained ConPLex with its default protein and molecule featurizers and the
released contrastive co-embedding objective. Results were averaged over five
seeds.

\paragraph{DrugBAN~\citep{bai2023drugban} (\url{https://github.com/peizhenbai/DrugBAN}).}
We followed the authors' default configuration, including the graph-based drug
encoder, protein encoder, bilinear co-attention module, and prediction head.
The model was trained on the same splits and averaged across five
seeds. For the mutation-effect task, we replaced the final classifier with a
small regression head and fine-tuned the full model end-to-end within each
leave-one-protein-out fold.

\paragraph{DrugCLIP~\citep{jia2026drugclip} (\url{https://github.com/bowen-gao/DrugCLIP}).}
We initialized DrugCLIP from the released BindingDB-pretrained checkpoint and
evaluated it on each downstream split using the authors' encoder architecture.
For the mutation-effect setting, we attach a small regression head to the
drug and pocket encoders and fine-tune the full model end-to-end within each
leave-one-protein-out fold.

\paragraph{SP-DTI~\citep{liu2025spdti} (\url{https://github.com/Steven51516/SP-DTI}).}
SP-DTI requires protein pocket information, so we used ESMFold-derived structures~\citep{lin2023evolutionary} and
extracted pockets before training. We then evaluated the released 
pipeline with the authors' default model configuration and averaged results
over five seeds.

\paragraph{GS-DTI~\citep{yu2025graph} (\url{https://github.com/purvavideha/GSDTI}).}
GS-DTI uses both molecular graph features and protein structural information.
We precomputed the required KPGT~\citep{li2022kpgt} drug features
(\url{https://github.com/lihan97/KPGT}) and ESMFold-derived~\citep{lin2023evolutionary} protein features,
then trained the released model with its default configuration and averaged
results over five seeds.

\paragraph{3DProtDTA~\citep{voitsitskyi2023}.}
3DProtDTA represents targets as residue-level graphs derived from predicted protein structures. For the low-homology benchmark we use the same family-disjoint train/validation/test assignment and ESMFold-derived structures supplied to the structure-informed baselines, averaging results across five matched runs.

\paragraph{ESM-1v~\citep{meier2021language} (\url{https://github.com/facebookresearch/esm}).}
We evaluated ESM-1v zero-shot using the masked-marginal pseudo-log-likelihood
ratio between mutant and wild type, averaged over the released five-model
ensemble. Since ESM-1v has no drug input, the resulting score is independent of
drug context.

\paragraph{ESM-2~\citep{lin2023evolutionary} (\url{https://github.com/facebookresearch/esm}).}
We evaluated ESM-2 in the same zero-shot manner as ESM-1v, using the 35M and
150M checkpoints. These scores provide drug-agnostic estimates of variant
plausibility.

\paragraph{EVE~\citep{frazer2021disease} (\url{https://github.com/OATML-Markslab/EVE}).}
We trained one EVE model per gene using the corresponding MSA and the authors'
default training procedure. Variants were scored zero-shot by comparing the
model likelihood of the mutant and wild-type sequences.

\paragraph{Tranception~\citep{notin2022tranception} (\url{https://github.com/OATML-Markslab/Tranception}).}
We evaluated Tranception-Large with the authors' released scoring script and
the same per-gene MSAs used for EVE. As with EVE and ESM, Tranception does not
condition on the drug.

\paragraph{Boltz-2~\citep{passaro2025boltz2} (\url{https://github.com/jwohlwend/boltz}).}
We exclude Boltz-2 from the low-homology BindingDB comparison
(Table~\ref{tab:dti-main}): its released pretraining corpus overlaps
89.4\% of the held-out MMseqs2 test clusters, preventing a family-held-out
comparison without retraining. Co-folding all 32{,}578 test pairs would
also be costly (approximately 48\,s per complex, including MSA generation).

For DMS ranking (Table~\ref{tab:variant-drug}), we instead use frozen
Boltz-2 representations with an LOO-trained MLP regression head. Mutant
and wild-type trunk representations are nearly identical (cosine
similarity $0.99994$), and this readout yields $\rho=0.038$ on LOCO-30\%
and $\rho=0.015$ on Oncology LOO. These results specifically characterize
the frozen-representation baseline rather than a fully mutation-adapted
co-folding model.

\paragraph{TCR--peptide external baselines.}
\label{app:tcr-external-baselines}

\paragraph{Additional released classifier baselines.}
NetTCR-2.0~\citep{montemurro2021nettcr} and pMTnet~\citep{lu2021pmtnet} are included as additional released end-to-end classifier references and are evaluated on the same four variant-ranking test sets. Their Spearman and Pearson correlations are reported in
Tables~\ref{tab:mutation-affinity-correlation} and
\ref{tab:mutation-affinity-pearson}, respectively.

\paragraph{ePytope Benchmark TCR-Predictors ~\citep{drost2025benchmarking}}. iTCep~\citep{zhang2023itcep}, TULIP-TCR~\citep{meynard2024tulip}, TEIM ~\citep{peng2023teim}, and ERGO-II~\citep{springer2021ergoii} with its McPAS variant, and ImRex~\citep{moris2021imrex} with its full variant were evaluated through the ePytope benchmarking pipeline (\url{https://github.com/SchubertLab/benchmark_TCRprediction}). Each model was loaded with its publicly released pretrained checkpoints and no retraining was performed. In addition to the DMS dataset provided by ePytope, we integrated three additional benchmarks into the pipeline: \textsc{BATCAVE}, \textsc{ATLAS-Pep}, and \textsc{ATLAS-TCR}. Across all four test sets, we uniformly report Pearson and Spearman correlations between predicted binding scores and measured binding activity or affinity.

\paragraph{EPACT~\citep{zhang2024epitope}.} EPACT was evaluated outside the ePytope pipeline using the official repository (\url{https://github.com/zhangyumeng1sjtu/EPACT}). We ran inference with all five publicly released cross-validation checkpoints and averaged predictions across folds prior to computing metrics.

\paragraph{TCR-T5~\citep{karthikeyan2025tcrt5}.} TCR-T5 was also evaluated using the official repository (\url{https://github.com/pirl-unc/tcr_translate}). We loaded the publicly released fine-tuned checkpoint and scored each TCR--pMHC pair as the conditional log-likelihood of the CDR3$\beta$ sequence conditioned on the epitope and MHC allele.

\clearpage
\section{Additional Protein--Ligand Results and Ablations}
\label{app:prot-lig-extra}

The following experiments extend the protein--ligand evaluation in
Section~\ref{sec:exp-dti}. We first report conventional binding benchmarks and
variant-ranking robustness, then examine cross-modal contact localization and
context-guided sequence design. Controlled scoring and optimization ablations
are presented last. Methods and sampling definitions appear in
Appendices~\ref{app:logica-capabilities} and~\ref{app:data-collection}.

\subsection{Binding prediction and backbone-size sensitivity}
\label{app:standard-dti}

We complement the primary low-homology result with the published DAVIS,
BindingDB-test, and BioSNAP splits~\citep{davis2011comprehensive,
huang2021moltrans,zitnik2018biosnap}. Unlike the family-held-out benchmark,
these conventional splits do not enforce protein-family separation and are
included for comparison with prior DTI studies. \textsc{LogiCA} has the
strongest sequence-only AUROC or AUPR in five of the six reported metrics;
structure-informed models, which use additional inputs, are shown separately
(Table~\ref{tab:dti-standard}).

\begin{table*}[h]
    \centering
    \caption{\textbf{\textsc{LogiCA} gives the strongest sequence-only result on five of six conventional DTI metrics.}
Five-run mean $\pm$ s.d. for DAVIS, BindingDB-test, and BioSNAP. Blue rows are matched contextualized models; gray rows are external sequence-only comparators; orange rows use additional structure or pocket information ($\ddagger$). Bold and underline mark the best and second-best \emph{sequence-only} values; $\dagger$ denotes a native token head.}
    \label{tab:dti-standard}
    \small
    \renewcommand{\arraystretch}{1.08}
    \resizebox{0.9\textwidth}{!}{
    \begin{tabular}{lcccccc}
    \toprule
    \multirow[c]{2}{*}{\textbf{Method}}
    & \multicolumn{2}{c}{\textbf{DAVIS}}
    & \multicolumn{2}{c}{\textbf{BindingDB (Test)}}
    & \multicolumn{2}{c}{\textbf{BioSNAP}} \\
    \cmidrule(lr){2-3}
    \cmidrule(lr){4-5}
    \cmidrule(lr){6-7}
     & \multicolumn{1}{c}{AUC}
     & \multicolumn{1}{c}{AUPR}
     & \multicolumn{1}{c}{AUC}
     & \multicolumn{1}{c}{AUPR}
     & \multicolumn{1}{c}{AUC}
     & \multicolumn{1}{c}{AUPR} \\
    \midrule

    \rowcolor{blue!6}
    \multicolumn{7}{l}{\textit{Contextualized backbone (ours)}} \\
    \midrule
    \methodcell{w/ LatentCA}{latent contrastive}
    & \stdcell{0.904}{0.007} & \stdcell{0.371}{0.012}
    & \stdcell{0.812}{0.040} & \stdcell{0.481}{0.061}
    & \understdcell{0.907}{0.012} & \stdcell{0.915}{0.013} \\

    \methodcell{w/ \textsc{LogiMLM}\genmark}{logit-level MLM}
    & \stdcell{0.739}{0.001} & \stdcell{0.197}{0.007}
    & \stdcell{0.754}{0.001} & \stdcell{0.332}{0.004}
    & \stdcell{0.815}{0.001} & \stdcell{0.846}{0.002} \\

    \rowcolor{blue!10}
    \methodcell{\textbf{w/ \textsc{LogiCA}\genmark}}{\textbf{native-logit contrastive}}
    & \beststdcell{0.924}{0.005} & \beststdcell{0.446}{0.019}
    & \understdcell{0.906}{0.003} & \beststdcell{0.635}{0.014}
    & \beststdcell{0.921}{0.003} & \beststdcell{0.929}{0.004} \\

    \midrule
    \rowcolor{gray!8}
    \multicolumn{7}{l}{\textit{Sequence-only baselines}} \\
    \midrule
    \methodcell{MolTrans~\citep{huang2021moltrans}}{end-to-end classifier}
    & \stdcell{0.901}{0.006} & \stdcell{0.343}{0.024}
    & \stdcell{0.904}{0.005} & \stdcell{0.584}{0.010}
    & \stdcell{0.886}{0.007} & \stdcell{0.894}{0.009} \\

    \methodcell{ConPLex~\citep{singh2023contrastive}}{latent contrastive}
    & \stdcell{0.887}{0.007} & \understdcell{0.441}{0.033}
    & \stdcell{0.854}{0.004} & \stdcell{0.614}{0.009}
    & \stdcell{0.866}{0.006} & \stdcell{0.888}{0.005} \\

    \methodcell{DrugBAN~\citep{bai2023drugban}}{end-to-end classifier}
    & \stdcell{0.883}{0.006} & \stdcell{0.349}{0.022}
    & \beststdcell{0.912}{0.002} & \understdcell{0.617}{0.006}
    & \understdcell{0.907}{0.002} & \understdcell{0.916}{0.004} \\

    \midrule
    \rowcolor{orange!8}
    \multicolumn{7}{l}{\textit{Structure-informed baselines}} \\
    \midrule
    \rowcolor{orange!4}
    \methodcell{DrugCLIP\structmark~\citep{jia2026drugclip}}{latent contrastive}
    & \stdcell{0.925}{0.009} & \stdcell{0.473}{0.025}
    & \stdcell{0.863}{0.010} & \stdcell{0.581}{0.025}
    & \stdcell{0.788}{0.006} & \stdcell{0.821}{0.007} \\

    \rowcolor{orange!4}
    \methodcell{SP-DTI\structmark~\citep{liu2025spdti}}{end-to-end classifier}
    & \stdcell{0.924}{0.002} & \stdcell{0.424}{0.023}
    & \stdcell{0.921}{0.001} & \stdcell{0.645}{0.014}
    & \stdcell{0.924}{0.003} & \stdcell{0.923}{0.005} \\

    \rowcolor{orange!4}
    \methodcell{GS-DTI\structmark~\citep{yu2025graph}}{latent contrastive}
    & \stdcell{0.916}{0.012} & \stdcell{0.430}{0.038}
    & \stdcell{0.920}{0.004} & \stdcell{0.669}{0.012}
    & \stdcell{0.934}{0.003} & \stdcell{0.934}{0.004} \\

    \bottomrule
    \end{tabular}
    }
\end{table*}
\renewcommand{\arraystretch}{1.0}

We next vary encoder size while keeping the fine-tuning protocol fixed.
The 35M backbone already recovers most of the 150M performance, with modest
additional gains on DAVIS and BindingDB-test but essentially unchanged
BioSNAP results. Increasing the LoRA rank at fixed backbone size offers
smaller improvements still (Table~\ref{tab:dti-ablation}). These findings
indicate that the benefits of native-logit alignment do not depend on a
particularly large protein encoder.

\begin{table*}[h]
    \centering
    \caption{Protein backbone and LoRA-rank ablations on conventional DTI splits.
The highlighted 150M configuration also underlies the primary low-homology
benchmark, but the present results use the conventional splits. Cells are
five-seed means; standard deviations for the primary low-homology benchmark
are reported in Table~\ref{tab:dti-main}.}
    \label{tab:dti-ablation}
    \small
    \renewcommand{\arraystretch}{1.08}
    \resizebox{0.8\textwidth}{!}{%
    \begin{tabular}{lcccccc}
    \toprule
    \multirow[c]{2}{*}{\textbf{Method}}
    & \multicolumn{2}{c}{\textbf{DAVIS}}
    & \multicolumn{2}{c}{\textbf{BindingDB (Test)}}
    & \multicolumn{2}{c}{\textbf{BioSNAP}} \\
    \cmidrule(lr){2-3}
    \cmidrule(lr){4-5}
    \cmidrule(lr){6-7}
     & \multicolumn{1}{c}{AUC}
     & \multicolumn{1}{c}{AUPR}
     & \multicolumn{1}{c}{AUC}
     & \multicolumn{1}{c}{AUPR}
     & \multicolumn{1}{c}{AUC}
     & \multicolumn{1}{c}{AUPR} \\
    \midrule
    \rowcolor{blue!6}
    \multicolumn{7}{l}{\textit{Backbone scaling at fixed LoRA rank $r{=}64$}} \\
    \midrule
    w/ \textsc{LogiCA}\genmark (8M)
    & 0.920 & 0.425 & 0.891 & 0.614 & 0.914 & 0.921 \\
    w/ \textsc{LogiCA}\genmark (35M)
    & 0.919 & 0.433 & 0.894 & 0.608 & 0.920 & 0.927 \\
    \rowcolor{blue!10}
    \textbf{w/ \textsc{LogiCA}\genmark (150M)}
    & \textbf{0.924} & \textbf{0.446} & \textbf{0.906} & \textbf{0.635} & \textbf{0.921} & \textbf{0.929} \\
    \midrule
    \rowcolor{blue!6}
    \multicolumn{7}{l}{\textit{LoRA rank ablation at fixed 35M backbone}} \\
    \midrule
    w/ \textsc{LogiCA}\genmark (35M, $r{=}8$)
    & 0.916 & 0.426 & 0.886 & 0.595 & 0.913 & 0.920 \\
    w/ \textsc{LogiCA}\genmark (35M, $r{=}32$)
    & 0.919 & 0.430 & 0.888 & 0.600 & 0.916 & 0.924 \\
    w/ \textsc{LogiCA}\genmark (35M, $r{=}64$)
    & 0.919 & 0.433 & 0.894 & 0.608 & 0.920 & 0.927 \\
    \bottomrule
    \end{tabular}%
    }
\end{table*}
\renewcommand{\arraystretch}{1.0}


\subsection{Mutation ranking and data efficiency}
\label{app:scaling-regime}

We examine how \textsc{LogiCA} scales along three complementary axes: protein-backbone capacity, downstream adaptation data, and the amount of paired supervision used to learn the contextual objective.

\paragraph{Backbone scaling and likelihood margins.}
We first vary ESM-2 from 8M to 650M parameters, keeping SELFormer and the
training recipe fixed. We measure how well each backbone distinguishes
matched protein--ligand pairs from sampled mismatches through the symmetric
conditional-likelihood margin
\begin{equation}
  \widehat{\bar{\gamma}}
  =
  \frac{1}{2}
  \left[
  \ell(x\mid y)-\frac{1}{K}\sum_{k=1}^K \ell(x\mid y_k^-)
  +
  \ell(y\mid x)-\frac{1}{K}\sum_{k=1}^K \ell(y\mid x_k^-)
  \right],
  \label{eq:likelihood-margin}
\end{equation}
averaged over held-out pairs. The peak margin increases with backbone
capacity, approximately as $\widehat{\bar\gamma}^{\star}\propto N^{0.16}$
over the evaluated $80\times$ parameter range (Figure~\ref{fig:scaling}A--B).
The mutation-ranking advantage is already present at moderate scale: for
example, 35M \textsc{LogiCA} reaches $\rho=0.256$ on the Coelho screen,
compared with $0.214$ for matched \textsc{LogiMLM}; the corresponding 150M
results are $0.271$ and $0.258$. Thus increasing capacity improves the
scores without changing the qualitative conclusion.

\paragraph{Target-specific adaptation: learning from the test cohort.}
We first vary the fraction of labeled variants available \emph{from the
held-out target itself}, $f_{\mathrm{target}}\in\{0,1,2,5,10,25,50\}\%$,
and evaluate on the remaining variants. All models generally improve as
more of the target landscape is observed (Figure~\ref{fig:scaling}C;
Table~\ref{tab:variant-fewshot}). LogiCA's advantage is largest with little
target supervision and narrows as these labels become abundant. At 50\%
target adaptation, the LOCO-30\% Spearman advantage over LogiMLM falls
from $+0.0191$ to $+0.0019$, while its top-1\% enrichment advantage falls
from $+1.52\times$ to $+0.32\times$. This experiment assesses how much
\emph{target-specific} information is needed to close the gap; it is
distinct from the strictly unseen-cohort training-fraction study below.

\begin{table*}[!htbp]
    \centering
    \caption{\textbf{Additional target labels progressively narrow the difference between \textsc{LogiCA} and matched controls.}
Across 103 mutational DMS cohorts, the indicated fraction of labeled variants from each target is used for adaptation and the remainder for evaluation. Metrics are reported in separate blocks; bold marks the highest value within each block.}
    \label{tab:variant-fewshot}
    \small
    \renewcommand{\arraystretch}{1.06}
    \setlength{\tabcolsep}{4.8pt}
    \resizebox{\textwidth}{!}{%
    \begin{tabular}{lccccccc}
    \toprule
    \textbf{Method (150M)} & \textbf{0\% (LOCO)} & \textbf{1\%} & \textbf{2\%} & \textbf{5\%} & \textbf{10\%} & \textbf{25\%} & \textbf{50\%} \\
    \midrule
    \rowcolor{gray!10}
    \multicolumn{8}{l}{\textbf{Spearman $\rho$} {\footnotesize (103 mutational DMS cohorts)}} \\
    w/ SiteFuse (site MLP) & $0.289$ & $0.320$ & $0.349$ & $0.401$ & $0.465$ & $0.541$ & $0.583$ \\
    w/ \textsc{LogiMLM}\genmark (cond.\ MLM) & $0.293$ & $0.327$ & $0.358$ & $0.415$ & $0.483$ & $0.558$ & $0.601$ \\
    \rowcolor{blue!10}
    \textbf{w/ \textsc{LogiCA}\genmark (contrastive)} & \textbf{0.313} & \textbf{0.346} & \textbf{0.375} & \textbf{0.429} & \textbf{0.493} & \textbf{0.563} & \textbf{0.603} \\
    \midrule
    \rowcolor{gray!10}
    \multicolumn{8}{l}{\textbf{Binary AUROC} {\footnotesize (103 mutational DMS cohorts)}} \\
    w/ \textsc{LogiMLM}\genmark (cond.\ MLM) & $0.678$ & $0.697$ & $0.715$ & $0.747$ & $0.785$ & $0.827$ & $0.851$ \\
    \rowcolor{blue!10}
    \textbf{w/ \textsc{LogiCA}\genmark (contrastive)} & \textbf{0.691} & \textbf{0.709} & \textbf{0.726} & \textbf{0.757} & \textbf{0.793} & \textbf{0.831} & \textbf{0.853} \\
    \midrule
    \rowcolor{gray!10}
    \multicolumn{8}{l}{\textbf{Precision@$1\%$ ($\mathrm{P@1\%}$)} {\footnotesize (103 mutational DMS cohorts)}} \\
    w/ SiteFuse (site MLP) & $2.62\times$ & $2.92\times$ & $3.25\times$ & $3.86\times$ & $4.58\times$ & $5.34\times$ & $5.82\times$ \\
    w/ \textsc{LogiMLM}\genmark (cond.\ MLM) & $2.58\times$ & $2.96\times$ & $3.38\times$ & $4.08\times$ & $4.89\times$ & $5.72\times$ & $6.21\times$ \\
    \rowcolor{blue!10}
    \textbf{w/ \textsc{LogiCA}\genmark (contrastive)} & \textbf{4.06}$\times$ & \textbf{4.31}$\times$ & \textbf{4.58}$\times$ & \textbf{5.02}$\times$ & \textbf{5.54}$\times$ & \textbf{6.08}$\times$ & \textbf{6.44}$\times$ \\
    \midrule
    \rowcolor{gray!10}
    \multicolumn{8}{l}{\textbf{12-Step Coordinate MCMC Gibbs Enrichment} {\footnotesize (103 mutational DMS cohorts)}} \\
    w/ SiteFuse (site MLP) & $1.17\times$ & $1.21\times$ & $1.26\times$ & $1.34\times$ & $1.45\times$ & $1.58\times$ & $1.66\times$ \\
    w/ \textsc{LogiMLM}\genmark (cond.\ MLM) & $1.55\times$ & $1.68\times$ & $1.81\times$ & $2.05\times$ & $2.34\times$ & $2.68\times$ & $2.89\times$ \\
    \rowcolor{blue!10}
    \textbf{w/ \textsc{LogiCA}\genmark (contrastive)} & \textbf{1.85}$\times$ & \textbf{1.96}$\times$ & \textbf{2.07}$\times$ & \textbf{2.26}$\times$ & \textbf{2.50}$\times$ & \textbf{2.77}$\times$ & \textbf{2.93}$\times$ \\
    \bottomrule
    \end{tabular}%
    }
\end{table*}
\renewcommand{\arraystretch}{1.0}

\begin{table}[!htbp]
\centering
\caption{\textbf{Target-label adaptation on the 11-oncogene resistance benchmark.}
Fractions correspond to labels drawn from the evaluated target, unlike the strictly unseen-target experiment in Table~\ref{tab:contrastive-data-scale}. Each cell reports Spearman $\rho$ / binary AUROC, with the actual target-label fractions shown in the column headers.}
\label{tab:variant-fewshot-oncology}
\small
\renewcommand{\arraystretch}{1.16}
\setlength{\tabcolsep}{6pt}
\begin{tabularx}{\textwidth}{@{}>{\raggedright\arraybackslash}Xcccc@{}}
\toprule
\textbf{Model} & \textbf{0\%} & \textbf{5\%} & \textbf{10\%} & \textbf{15\%} \\
\midrule
\textsc{LogiMLM} (8M) & $0.279/0.624$ & $0.346/0.659$ & $0.427/0.699$ & $0.471/0.713$ \\
\textbf{\textsc{LogiCA} (8M)} & $\mathbf{0.297}/\mathbf{0.629}$ & $\mathbf{0.377}/\mathbf{0.670}$ & $\mathbf{0.443}/\mathbf{0.709}$ & $\mathbf{0.489}/\mathbf{0.728}$ \\
\textsc{LogiMLM} (35M) & $0.271/0.615$ & $0.372/0.661$ & $0.412/0.666$ & $0.478/0.693$ \\
\rowcolor{blue!10}
\textbf{\textsc{LogiCA} (35M)} & $\mathbf{0.296}/\mathbf{0.636}$ & $\mathbf{0.381}/\mathbf{0.669}$ & $\mathbf{0.460}/\mathbf{0.719}$ & $\mathbf{0.501}/\mathbf{0.741}$ \\
\bottomrule
\end{tabularx}
\end{table}
\renewcommand{\arraystretch}{1.0}


\paragraph{Cross-target data efficiency: keeping the test cohort unseen.}
We now hold the target cohort entirely out of training and instead vary the
fraction $f_{\mathrm{train}}\in\{1,2,5,10,25,50,100\}\%$ of \emph{other}
DMS supervision used for adaptation. This experiment tests transfer to
unseen targets, rather than adaptation to labeled variants of the test
target (Figure~\ref{fig:fewshot-gap-convergence-wide};
Table~\ref{tab:contrastive-data-scale}).

LogiCA is especially data-efficient in this setting. With only 1--5\% of
held-in DMS supervision, it exceeds LogiMLM trained on 100\% of that data
in both rank correlation and top-1\% enrichment. At 1\% held-in data,
LogiCA achieves $\rho=0.310$ and $\mathrm{P@1\%}=3.47\times$ on
LOCO-30\%, compared with $0.293$ and $2.58\times$ for fully trained
LogiMLM. Overall correlation gaps narrow with increasing held-in data,
whereas the top-hit enrichment advantage persists. Similar behavior is
observed on LOLO and Oncology LOO. A separate BindingDB pretraining-fraction
ablation shows the same low-data tendency on Coelho resistance, with AUROC
$0.590$ for LogiCA versus $0.530$ for LogiMLM at 0.1\% of paired
pretraining data (Table~\ref{tab:contrastive-data-scale}). These results are
consistent with, rather than proof of, a benefit from explicit hard-negative
comparisons when prioritizing the most functional variants.

\begin{table}[!htbp]
    \centering
    \caption{\textbf{Contrastive alignment remains data-efficient when target cohorts are entirely unseen.}
Top: changing the fraction $f_{\mathrm{train}}$ of held-in DMS supervision under LOCO-30\% evaluation on 103 mutational cohorts, with $0\%$ labels from the test cohort. Bottom: changing paired pretraining data on the Coelho and EGFR benchmarks. Each entry reports \textbf{\textsc{LogiCA}} / \textsc{LogiMLM}.}
    \label{tab:contrastive-data-scale}
    \small
    \renewcommand{\arraystretch}{1.08}
    \setlength{\tabcolsep}{4.2pt}
    \resizebox{\textwidth}{!}{%
    \begin{tabular}{lccccccc}
        \toprule
        \textbf{Metric (\textbf{\textsc{LogiCA}} / \textsc{LogiMLM})} & \textbf{1\%} & \textbf{2\%} & \textbf{5\%} & \textbf{10\%} & \textbf{25\%} & \textbf{50\%} & \textbf{100\%} \\
        \midrule
        Spearman $\rho$ (103 DMS) & \textbf{0.310}/$0.236$ & \textbf{0.312}/$0.251$ & \textbf{0.312}/$0.277$ & \textbf{0.313}/$0.284$ & \textbf{0.314}/$0.289$ & \textbf{0.315}/$0.291$ & \textbf{0.313}/$0.293$ \\
        $\mathrm{P@1\%}$ ($\times$, 103 DMS) & \textbf{3.47}/$2.10$ & \textbf{3.60}/$2.48$ & \textbf{3.70}/$2.59$ & \textbf{4.02}/$2.67$ & \textbf{4.04}/$2.61$ & \textbf{4.02}/$2.60$ & \textbf{4.06}/$2.58$ \\
        12-Step Gibbs ($\times$, 103 DMS) & \textbf{1.81}/$1.41$ & \textbf{1.82}/$1.48$ & \textbf{1.83}/$1.54$ & \textbf{1.84}/$1.56$ & \textbf{1.85}/$1.56$ & \textbf{1.85}/$1.55$ & \textbf{1.85}/$1.55$ \\
        \midrule
        \textbf{Pretraining Data Fraction} & \multicolumn{2}{c}{\textbf{0.1\%}} & \multicolumn{2}{c}{\textbf{1\%}} & \multicolumn{2}{c}{\textbf{10\%}} & \textbf{100\%} \\
        \midrule
        Coelho Cancer AUROC & \multicolumn{2}{c}{\textbf{0.590} / $0.530$} & \multicolumn{2}{c}{\textbf{0.608} / $0.558$} & \multicolumn{2}{c}{\textbf{0.622} / $0.585$} & \textbf{0.636} / $0.613$ \\
        Kim EGFR AUROC & \multicolumn{2}{c}{\textbf{0.588} / $0.532$} & \multicolumn{2}{c}{\textbf{0.602} / $0.560$} & \multicolumn{2}{c}{\textbf{0.618} / $0.592$} & \textbf{0.630} / $0.620$ \\
        \bottomrule
    \end{tabular}%
    }
\end{table}

\begin{figure}[!htbp]
\centering
\includegraphics[width=\linewidth]{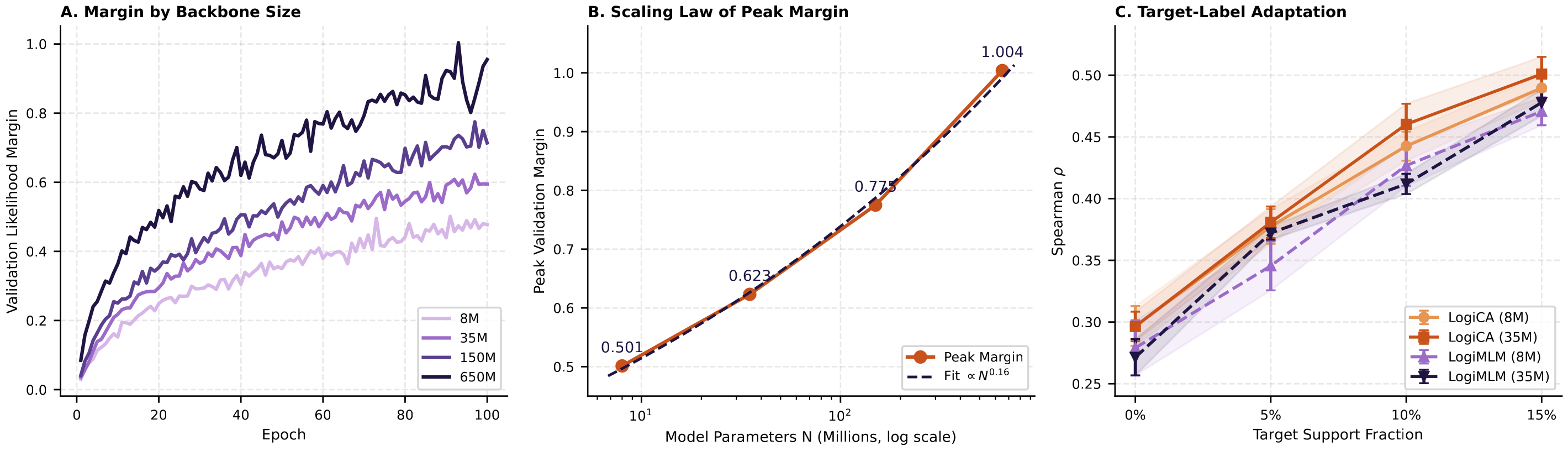}
\caption{\textbf{Protein-backbone size and target-specific labels influence different stages of \textsc{LogiCA} performance.}
\textbf{(A)} Held-out matched-versus-mismatched token-likelihood margins during pretraining with 8M, 35M, 150M, and 650M ESM-2 backbones.
\textbf{(B)} Peak margin as a function of backbone parameter count; the dashed line is a descriptive log--log fit, $\widehat{\bar\gamma}^{\star}\propto N^{0.16}$, over the evaluated sizes.
\textbf{(C)} Mutation-ranking performance as increasing fractions of labeled variants from the target cohort are used for adaptation, with the remaining variants held out for evaluation. This target-adaptation protocol is distinct from the unseen-target training-fraction study in Table~\ref{tab:contrastive-data-scale}.}
\label{fig:scaling}
\end{figure}

\subsection{Statistical robustness and per-gene heterogeneity}
\label{app:protein-statistical-robustness}

{ 

\paragraph{Accounting for correlated assays.}
Because multiple DMS cohorts assay the same gene or related proteins,
assay-level tests may overstate precision. We therefore aggregate paired
model differences by gene ($G=42$) and by $<30\%$-identity cluster
($K=32$), and check robustness with paired nonparametric tests,
a gene-block bootstrap, cluster-robust standard errors, and mixed-effects modeling.

At gene level, \textsc{LogiCA} improves Spearman correlation over
\textsc{LogiMLM} by $+0.0234\pm0.0064$ ($p=6.98\times10^{-4}$).
Ranking and 12-step sampling remain significant after Holm--Bonferroni
correction across the eight comparisons ($\alpha=0.01$), while top-1\%
enrichment is less robust to the choice of test and multiplicity correction.
Substantial within-gene dependence ($\mathrm{ICC}=0.826$) motivates this
analysis. Tables~\ref{tab:statistical-audit} and
\ref{tab:statistical-audit-full} give compact and full test results;
Figure~\ref{fig:app-statistical-audit} gives the corresponding visualization.

}

\begin{table}[!htbp]
\centering
\caption{\textbf{Ranking and prioritization gains remain detectable after gene- and cluster-level aggregation.}
Each gene-level paired effect is the mean difference $\pm$ s.e. across 42 genes. Gene and cluster $p$-values are paired $t$-tests across 42 genes and 32 protein clusters, respectively. Additional dependence-aware tests appear in Table~\ref{tab:statistical-audit-full}.}
\label{tab:statistical-audit}
\small
\renewcommand{\arraystretch}{1.13}
\setlength{\tabcolsep}{4pt}
\begin{tabularx}{\textwidth}{@{}>{\raggedright\arraybackslash}Xlccc@{}}
\toprule
\textbf{Metric} & \textbf{Comparison} & \textbf{Gene effect $\pm$ s.e.} & \textbf{Gene $p$} & \textbf{Cluster $p$} \\
\midrule
Spearman $\rho$ & LogiCA $-$ LogiMLM & $+0.0234\pm0.0064$ & $6.98\times10^{-4}$ & $3.30\times10^{-4}$ \\
 & LogiCA $-$ SiteFuse & $+0.0170\pm0.0056$ & $4.31\times10^{-3}$ & $3.86\times10^{-3}$ \\
 & LogiCA $-$ ESM-2 & $+0.0315\pm0.0067$ & $2.83\times10^{-5}$ & $4.21\times10^{-5}$ \\
\midrule
P@1\% enrichment & LogiCA $-$ LogiMLM & $+1.145\pm0.506$ & $0.0290$ & $0.0325$ \\
 & LogiCA $-$ ESM-2 & $+0.892\pm0.422$ & $0.0405$ & $0.0441$ \\
\midrule
12-step sampling & LogiCA $-$ LogiMLM & $+0.262\pm0.077$ & $1.55\times10^{-3}$ & $1.72\times10^{-3}$ \\
 & LogiCA $-$ SiteFuse & $+0.517\pm0.068$ & $2.08\times10^{-9}$ & $5.12\times10^{-8}$ \\
 & LogiMLM $-$ SiteFuse & $+0.255\pm0.084$ & $4.15\times10^{-3}$ & $4.88\times10^{-3}$ \\
\bottomrule
\end{tabularx}
\end{table}

\begin{table*}[!htbp]
\centering
\caption{Full dependence-aware statistical audit. Paired effects are shown
at the assay, gene, and homology-cluster levels, with gene-level Wilcoxon
and mixed-effects results retained for completeness. The compact primary
effect estimates are in Table~\ref{tab:statistical-audit}.}
\label{tab:statistical-audit-full}
\small
\renewcommand{\arraystretch}{1.14}
\setlength{\tabcolsep}{4.5pt}
\resizebox{\textwidth}{!}{%
\begin{tabular}{lcccccc}
\toprule
\textbf{Paired Comparison} & \shortstack{\textbf{Assay $\bar{\Delta}$}\\\textbf{($p_t$, $N{=}103$)}} & \shortstack{\textbf{Gene $\bar{\Delta}_g$}\\\textbf{($\pm\mathrm{SE}$, $G{=}42$)}} & \shortstack{\textbf{Gene $t_{41}$}\\\textbf{($p_t$)}} & \shortstack{\textbf{Gene}\\\textbf{Wilcox.\ $p_W$}} & \shortstack{\textbf{Cluster $\bar{\Delta}_k$}\\\textbf{($p_t$, $K{=}32$)}} & \shortstack{\textbf{LMM Wald $z$}\\\textbf{($p_{\mathrm{LMM}}$)}} \\
\midrule
\rowcolor{blue!6}
\multicolumn{7}{l}{\textbf{Spearman $\rho$}} \\
\textsc{LogiCA} vs.\ \textsc{LogiMLM} & \shortstack{$+0.0203$\\{\scriptsize $(1.55{\times}10^{-5})$}} & \shortstack{$+0.0234$\\{\scriptsize $\pm 0.0064$}} & \shortstack{$t{=}3.67$\\{\scriptsize $(6.98{\times}10^{-4})$}} & $2.21{\times}10^{-3}$ & \shortstack{$+0.0307$\\{\scriptsize $(3.30{\times}10^{-4})$}} & \shortstack{$z{=}4.22$\\{\scriptsize $(2.41{\times}10^{-5})$}} \\
\textsc{LogiCA} vs.\ SiteFuse & \shortstack{$+0.0241$\\{\scriptsize $(1.24{\times}10^{-4})$}} & \shortstack{$+0.0170$\\{\scriptsize $\pm 0.0056$}} & \shortstack{$t{=}3.02$\\{\scriptsize $(4.31{\times}10^{-3})$}} & $4.57{\times}10^{-3}$ & \shortstack{$+0.0217$\\{\scriptsize $(3.86{\times}10^{-3})$}} & \shortstack{$z{=}3.70$\\{\scriptsize $(2.13{\times}10^{-4})$}} \\
\textsc{LogiCA} vs.\ ESM-2 & \shortstack{$+0.0402$\\{\scriptsize $(2.33{\times}10^{-6})$}} & \shortstack{$+0.0315$\\{\scriptsize $\pm 0.0067$}} & \shortstack{$t{=}4.71$\\{\scriptsize $(2.83{\times}10^{-5})$}} & $4.52{\times}10^{-5}$ & \shortstack{$+0.0397$\\{\scriptsize $(4.21{\times}10^{-5})$}} & \shortstack{$z{=}4.71$\\{\scriptsize $(2.49{\times}10^{-6})$}} \\
\midrule
\rowcolor{blue!6}
\multicolumn{7}{l}{\textbf{Precision@$1\%$ ($\mathrm{P@1\%}$, $\times$)}} \\
\textsc{LogiCA} vs.\ \textsc{LogiMLM} & \shortstack{$+1.474\times$\\{\scriptsize $(1.84{\times}10^{-4})$}} & \shortstack{$+1.145\times$\\{\scriptsize $\pm 0.506\times$}} & \shortstack{$t{=}2.26$\\{\scriptsize $(2.90{\times}10^{-2})$}} & $1.60{\times}10^{-2}$ & \shortstack{$+1.464\times$\\{\scriptsize $(3.25{\times}10^{-2})$}} & \shortstack{$z{=}3.49$\\{\scriptsize $(4.83{\times}10^{-4})$}} \\
\textsc{LogiCA} vs.\ ESM-2 & \shortstack{$+1.326\times$\\{\scriptsize $(3.12{\times}10^{-4})$}} & \shortstack{$+0.892\times$\\{\scriptsize $\pm 0.422\times$}} & \shortstack{$t{=}2.12$\\{\scriptsize $(4.05{\times}10^{-2})$}} & $3.89{\times}10^{-3}$ & \shortstack{$+1.148\times$\\{\scriptsize $(4.41{\times}10^{-2})$}} & \shortstack{$z{=}3.18$\\{\scriptsize $(1.48{\times}10^{-3})$}} \\
\midrule
\rowcolor{blue!6}
\multicolumn{7}{l}{\textbf{12-Step Coordinate MCMC Gibbs Enrichment ($\times$)}} \\
\textsc{LogiCA} vs.\ \textsc{LogiMLM} & \shortstack{$+0.299\times$\\{\scriptsize $(4.12{\times}10^{-5})$}} & \shortstack{$+0.262\times$\\{\scriptsize $\pm 0.077\times$}} & \shortstack{$t{=}3.39$\\{\scriptsize $(1.55{\times}10^{-3})$}} & $3.55{\times}10^{-3}$ & \shortstack{$+0.328\times$\\{\scriptsize $(1.72{\times}10^{-3})$}} & \shortstack{$z{=}4.22$\\{\scriptsize $(2.42{\times}10^{-5})$}} \\
\textsc{LogiCA} vs.\ SiteFuse & \shortstack{$+0.678\times$\\{\scriptsize $(1.08{\times}10^{-9})$}} & \shortstack{$+0.517\times$\\{\scriptsize $\pm 0.068\times$}} & \shortstack{$t{=}7.64$\\{\scriptsize $(2.08{\times}10^{-9})$}} & $2.18{\times}10^{-7}$ & \shortstack{$+0.614\times$\\{\scriptsize $(5.12{\times}10^{-8})$}} & \shortstack{$z{=}10.17$\\{\scriptsize $(2.65{\times}10^{-24})$}} \\
\textsc{LogiMLM} vs.\ SiteFuse & \shortstack{$+0.379\times$\\{\scriptsize $(8.40{\times}10^{-5})$}} & \shortstack{$+0.255\times$\\{\scriptsize $\pm 0.084\times$}} & \shortstack{$t{=}3.04$\\{\scriptsize $(4.15{\times}10^{-3})$}} & $1.98{\times}10^{-3}$ & \shortstack{$+0.286\times$\\{\scriptsize $(4.88{\times}10^{-3})$}} & \shortstack{$z{=}4.08$\\{\scriptsize $(4.50{\times}10^{-5})$}} \\
\bottomrule
\end{tabular}%
}
\end{table*}

\begin{figure}[!htbp]
\centering
\includegraphics[width=1.0\linewidth]{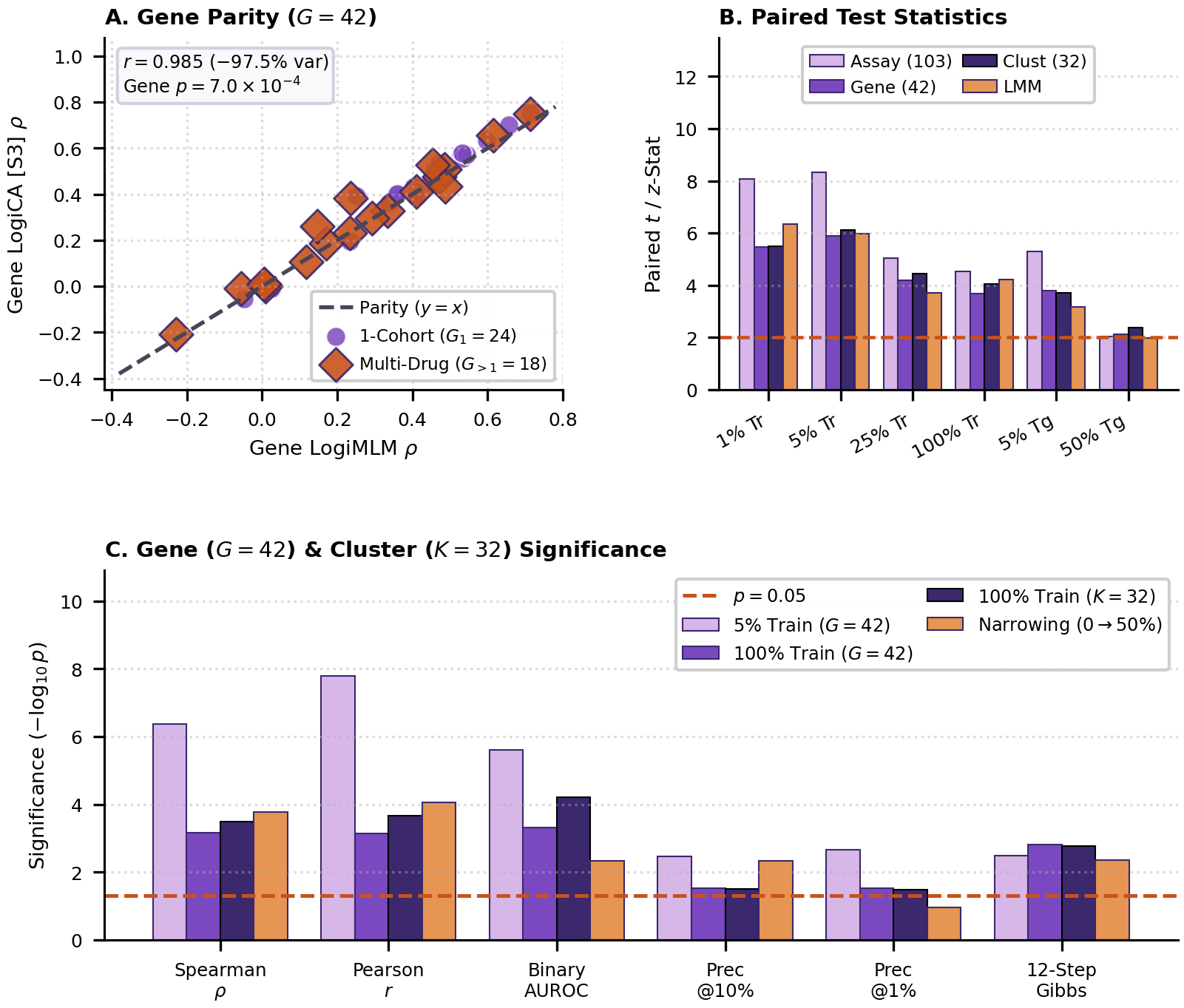}
\caption{\textbf{Mutation-ranking improvements persist after accounting for related protein targets.}
The audit includes 103 DMS cohorts grouped into 42 genes and 32 homology-disjoint protein clusters.
\textbf{(A)} Gene-level paired performance of \textsc{LogiCA} and \textsc{LogiMLM}; the strong correlation between methods motivates paired rather than independent-sample comparisons.
\textbf{(B)} Test statistics from assay-, gene-, and cluster-level analyses and a linear mixed-effects model.
\textbf{(C)} Significance across mutation ranking, top-hit prioritization, and coordinate-search enrichment, reported as $-\log_{10}p$ for gene- and cluster-level comparisons. Exact effects and tests are given in Tables~\ref{tab:statistical-audit} and~\ref{tab:statistical-audit-full}.}
\label{fig:app-statistical-audit}
\end{figure}

\paragraph{Variation across oncogenes.}
Gene-level analysis of the 11-oncogene resistance panel
(100 gene--drug pairs) reveals heterogeneous performance
across targets. Although \textsc{LogiCA} (150M) achieves
the highest overall mean Spearman correlation ($0.276$)
and AUROC ($0.645$), improvements are not uniform:
\textsc{LogiMLM} performs better on PIK3CA
($\rho=0.343$ vs.\ $0.264$), while EVE achieves a
higher correlation on PARP2 ($0.502$ vs.\ $0.224$).
These exceptions highlight target-dependent variation
beyond the aggregate results.

\clearpage

{

\subsection{Cross-modal dependency and binding-site localization}
\label{app:biolip-dependency}

We ask whether the native token distributions identify ligand-contacting
residues without 3D pocket coordinates or contact-label supervision. The
held-out BioLiP evaluation comprises 175 crystal complexes from five
protein-disjoint targets, with 44,337 labeled residue observations and
1,060 contacts within $5\,\text{\AA}$. These five UniProt targets do not
overlap the 137 DMS cohorts used for variant-ranking adaptation.
For reproducibility, the benchmark contains 231 evaluation windows,
69 distinct (protein-window, ligand-SELFIES) model inputs, nine non-overlapping
510-residue full-span windows, and 56 ligands. Target accessions are
\texttt{A4D1P6}, \texttt{O15178}, \texttt{Q96HS1}, \texttt{Q9BYX4},
and \texttt{Q9Y6K1}.

\paragraph{Comparison with matched controls.}
For models retaining a token head, we perturb residues and measure the shift
in the partner's predicted categorical distribution. The principal BioLiP
implementation probes protein-to-ligand dependencies, reducing over ligand
tokens as $\max_j D^{x\to y}_{ji}$. The 19-amino-acid perturbation
implementation agrees closely with the reference score
($r=0.9971$ across residue observations). SiteFuse instead probes its
predicted protein-site distribution by perturbing ligand tokens;
unconditional ESM/ProteinMPNN baselines use within-sequence dependencies.
For scalar-only methods, which lack a native token distribution, we report
a separate residue-occlusion score for completeness. These readouts are
related but not identical; the matched and external comparisons are listed
with their respective protocols in Table~\ref{tab:biolip-dependency-full}.

At the 150M scale, DMS-LOO LogiCA reaches interaction contact AUROC
$0.614$, compared with $0.553$ for positive-only LogiMLM and $0.522$
for SiteFuse. These results indicate that the
cross-modal token probabilities carry localized information beyond what
is recovered by the matched reconstruction and mutation-site latent controls.

\paragraph{Ligand specificity and backbone capacity.}
Unconditional ESM-2 150M achieves interaction AUROC $0.607$ by recovering
within-sequence dependencies, but its map cannot change with the ligand.
LogiCA produces ligand-conditioned maps and improves protein-macro and
pooled-residue localization under the reported aggregation schemes
(Figure~\ref{fig:biolip-dependency}). The distinction matters especially
when different ligands occupy different sites on the same protein.

The scaling trend is not monotonic: zero-shot contact localization peaks
at 35M ($0.677$ interaction AUROC), above the evaluated 150M and 650M
zero-shot checkpoints. This could reflect a changing balance between
unimodal predictions and cross-modal updates, although the benchmark
alone does not establish the mechanism. Perturbation choice matters as
well: averaging 19 amino-acid substitutions generally improves
localization over a single mask. The complete backbone, initialization,
and perturbation controls are separated in
Table~\ref{tab:biolip-dependency-ablations}.

\begin{table}[!p]
\centering
\caption{\textbf{Cross-modal dependency scores identify ligand-contact signals on unseen proteins.}
Principal comparisons across 175 BioLiP complexes from five protein-disjoint targets. Interaction and protein-level AUROC/AUPR are means $\pm$ s.d.; pooled AUROC aggregates residue observations. Paired $p$-values compare with 150M \textsc{LogiMLM}. No contact labels are used during model training.}
\label{tab:biolip-dependency-full}
\footnotesize
\renewcommand{\arraystretch}{1.16}
\setlength{\tabcolsep}{3pt}
\resizebox{\textwidth}{!}{%
\begin{tabular}{@{}lccccc@{}}
\toprule
\textbf{Model / evaluation} & \textbf{Interaction AUROC} & \textbf{Interaction AUPR} & \textbf{Protein AUROC} & \textbf{Pooled AUROC} & \textbf{Paired $p$} \\
\midrule
\textsc{LogiCA}, 19-AA & $0.585\pm0.218$ & $0.084\pm0.105$ & \textbf{$0.603\pm0.144$} & \underline{$0.629$} & $1.76\times10^{-2}$ \\
\textsc{SiteFuse}, DMS-LOO & $0.522\pm0.198$ & $0.053\pm0.047$ & $0.487\pm0.087$ & $0.451$ & $2.40\times10^{-1}$ \\
\textsc{LogiMLM}, DMS-LOO & $0.553\pm0.206$ & $0.066\pm0.114$ & $0.568\pm0.146$ & $0.582$ & --- \\
\rowcolor{blue!8}
\textbf{\textsc{LogiCA}, DMS-LOO} & \textbf{$0.614\pm0.195$} & $0.094\pm0.118$ & \underline{$0.599\pm0.178$} & \textbf{0.637} & $5.82\times10^{-6}$ \\
ESM-2 150M (within-sequence) & \underline{$0.607\pm0.225$} & \textbf{$0.115\pm0.192$} & $0.544\pm0.167$ & $0.568$ & $1.25\times10^{-5}$ \\
\textsc{LatentFuse} (150M) & $0.530\pm0.178$ & $0.074\pm0.106$ & $0.494\pm0.129$ & $0.558$ & $1.06\times10^{-1}$ \\
DrugBAN (LOO) & $0.489\pm0.192$ & $0.071\pm0.109$ & $0.459\pm0.083$ & $0.492$ & $1.63\times10^{-6}$ \\
Boltz-2\structmark\ (frozen trunk) & $0.523\pm0.171$ & $0.072\pm0.102$ & $0.489\pm0.119$ & $0.553$ & $3.62\times10^{-2}$ \\
DrugCLIP\structmark\ (LOO) & $0.576\pm0.202$ & \underline{$0.096\pm0.155$} & $0.538\pm0.133$ & $0.586$ & $1.05\times10^{-1}$ \\
\bottomrule
\end{tabular}%
}
\end{table}
\renewcommand{\arraystretch}{1.0}

\begin{table}[!p]
\centering
\caption{\textbf{Localization quality varies with encoder size and perturbation strategy.}
The supplementary BioLiP controls vary backbone capacity, random initialization, and single-mask versus amino-acid perturbations. Evaluation definitions and the paired-$p$ reference match Table~\ref{tab:biolip-dependency-full}.}
\label{tab:biolip-dependency-ablations}
\footnotesize
\renewcommand{\arraystretch}{1.16}
\setlength{\tabcolsep}{3pt}
\resizebox{\textwidth}{!}{%
\begin{tabular}{@{}lccccc@{}}
\toprule
\textbf{Model / evaluation} & \textbf{Interaction AUROC} & \textbf{Interaction AUPR} & \textbf{Protein AUROC} & \textbf{Pooled AUROC} & \textbf{Paired $p$} \\
\midrule
\textsc{LogiCA}, random-init & $0.529\pm0.196$ & $0.077\pm0.104$ & $0.473\pm0.082$ & $0.453$ & $6.78\times10^{-2}$ \\
\textsc{LogiCA}, mask & $0.568\pm0.196$ & $0.077\pm0.098$ & $0.567\pm0.135$ & $0.614$ & $1.27\times10^{-1}$ \\
\textsc{LogiCA}, 8M & \underline{$0.637\pm0.197$} & \textbf{$0.138\pm0.160$} & \underline{$0.603\pm0.121$} & \underline{$0.678$} & $2.71\times10^{-6}$ \\
\textbf{\textsc{LogiCA}, 35M} & \textbf{$0.677\pm0.201$} & \underline{$0.128\pm0.122$} & \textbf{$0.635\pm0.106$} & \textbf{0.707} & $9.36\times10^{-13}$ \\
\textsc{LogiCA}, 150M (live mask) & $0.577\pm0.207$ & $0.081\pm0.106$ & $0.569\pm0.133$ & $0.617$ & $1.34\times10^{-2}$ \\
\textsc{LogiCA}, 650M & $0.547\pm0.170$ & $0.054\pm0.043$ & $0.562\pm0.099$ & $0.568$ & $7.38\times10^{-1}$ \\
ESM-1v (within-sequence) & $0.566\pm0.219$ & $0.071\pm0.091$ & $0.517\pm0.043$ & $0.514$ & $3.37\times10^{-1}$ \\
Thermo / ProteinMPNN\structmark\ (3D AF2) & $0.557\pm0.183$ & $0.063\pm0.053$ & $0.513\pm0.089$ & $0.560$ & $8.25\times10^{-1}$ \\
\bottomrule
\end{tabular}%
}
\end{table}
\renewcommand{\arraystretch}{1.0}

\begin{figure}[!htbp]
    \centering
    \includegraphics[width=1.0\linewidth]{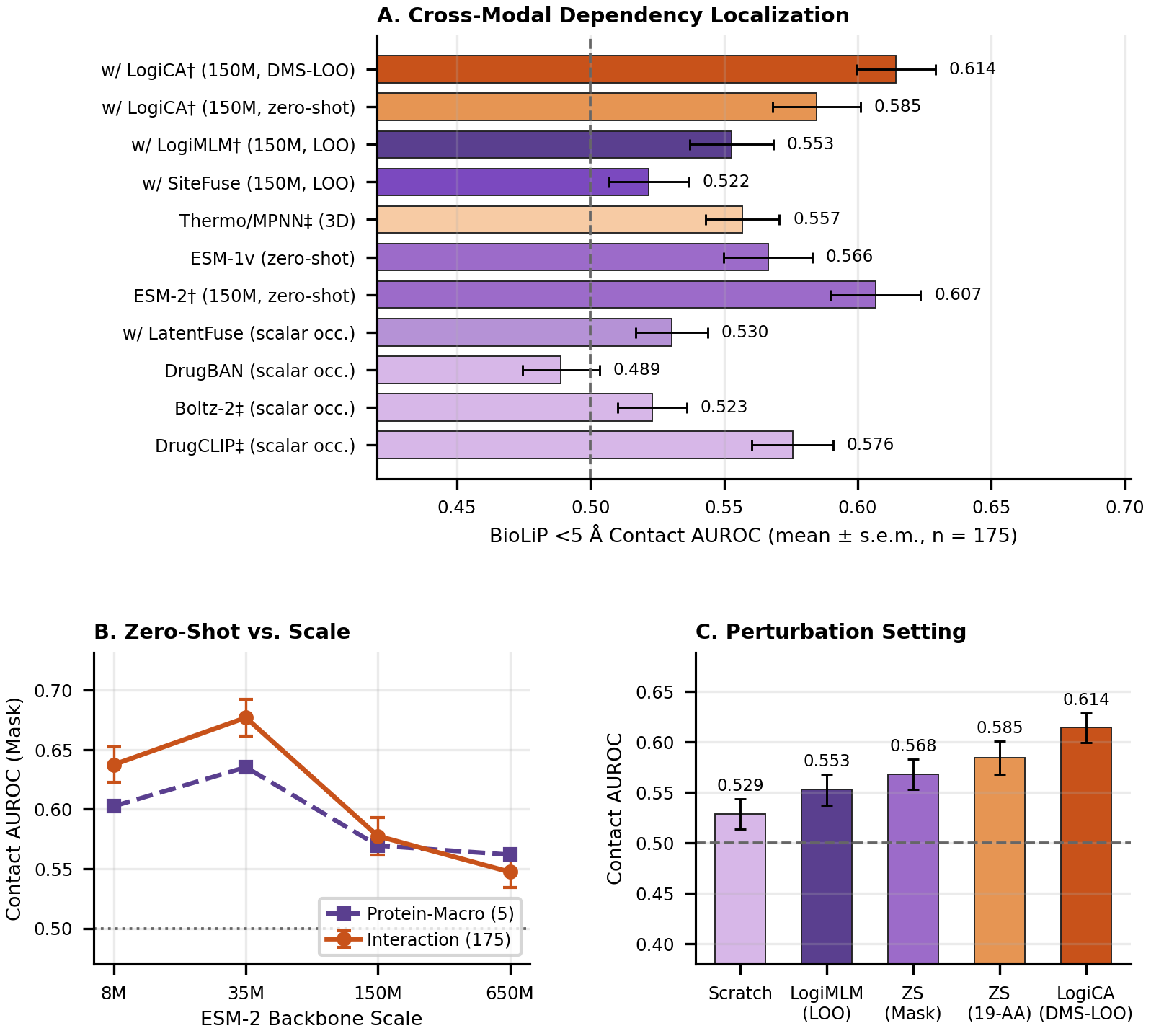}
    \caption{\textbf{Context-conditioned token perturbations localize ligand contacts without contact-label supervision.}
Evaluation uses 175 protein--ligand crystal structures from five held-out protein targets and 44,337 labeled residue observations; contacts are defined by a heavy-atom distance below $5\,\text{\AA}$.
\textbf{(A)} Mean per-interaction contact AUROC compares matched native-token models, unconditional protein baselines, and scalar-score occlusion controls; error bars represent s.e.m.
\textbf{(B)} Zero-shot localization varies non-monotonically with ESM-2 backbone size and peaks at 35M parameters ($0.677$ mean per-interaction AUROC).
\textbf{(C)} Initialization, fine-tuning, and perturbation comparisons distinguish single-mask probing from 19-amino-acid substitution averaging. Per-interaction and protein-aggregated metrics are tabulated separately in Tables~\ref{tab:biolip-dependency-full} and~\ref{tab:biolip-dependency-ablations}.}
    \label{fig:biolip-dependency}
\end{figure}

\subsection{Context-conditioned sequence generation and design}
\label{app:logica-generation-results}

We next evaluate whether the native token distributions used for discriminative
ranking can also guide sequence design without separately training a decoder or
a generative objective. The sampling and score-guided search procedures are
defined in Appendix~\ref{app:logica-generation}. We first examine broad
retrospective enrichment across the mutational DMS cohorts and then study
combinatorial multi-mutant optimization in an experimentally characterized system.

\paragraph{Evaluation across mutational DMS cohorts.}
We compare the within-site conditional, finite-library Boltzmann, and
coordinate score-guided search procedures over the corresponding measured
variant libraries. Ranking performance and sampling enrichment are shown
side by side, but quantify different readouts of the native token interface.
Table~\ref{tab:app-gibbs-all} reports ranking and finite-budget
prioritization, while Table~\ref{tab:app-gibbs-sampling}
reports sampling. The coordinate-chain
evaluation is restricted to the mutational DMS subset.

\begin{table}[!htbp]
\centering
\caption{Retrospective variant prioritization across 137 public protein--ligand cohorts. The three P@ thresholds report fold enrichment over uniform selection. Coordinate sampling, which requires mutational DMS, is reported separately in Table~\ref{tab:app-gibbs-sampling}.}
\label{tab:app-gibbs-all}
\small
\renewcommand{\arraystretch}{1.13}
\setlength{\tabcolsep}{5pt}
\begin{tabularx}{\textwidth}{@{}>{\raggedright\arraybackslash}Xcccc@{}}
\toprule
\textbf{Model} & Spearman $\rho$ & P@1\% & P@5\% & P@10\% \\
\midrule
Random Uniform Baseline & $0.000$ & $1.00\times$ & $1.00\times$ & $1.00\times$ \\
Unconditioned ESM-2 (650M) & $0.251$ & $2.73\times$ & $2.22\times$ & $2.03\times$ \\
\textsc{LogiCA} (150M, zero-shot) & $0.274$ & $3.38\times$ & $2.52\times$ & $2.18\times$ \\
\textsc{LogiCA} (650M, zero-shot) & $0.289$ & $3.78\times$ & $2.69\times$ & $2.31\times$ \\
\textsc{LogiCA} Guided ($\alpha=0.5$) & $0.294$ & $3.95\times$ & $2.81\times$ & $2.39\times$ \\
\midrule
\multicolumn{5}{@{}l}{\emph{Unseen LOO adaptation: 103 mutational DMS cohorts}} \\
w/ SiteFuse (150M, LOCO-30\%) & $0.289$ & $2.62\times$ & $2.35\times$ & $2.12\times$ \\
w/ \textsc{LogiMLM}\genmark (150M, LOCO-30\%) & $0.293$ & $2.58\times$ & $2.16\times$ & $2.05\times$ \\
\rowcolor{blue!8}
\textbf{w/ \textsc{LogiCA}\genmark (150M, LOCO-30\%)} & \textbf{0.313} & \textbf{4.06}$\times$ & \textbf{2.65}$\times$ & \textbf{2.22}$\times$ \\
\rowcolor{blue!8}
\textbf{w/ \textsc{LogiCA}\genmark (150M, LOLO)} & \textbf{0.354} & \textbf{4.62}$\times$ & \textbf{2.94}$\times$ & \textbf{2.45}$\times$ \\
\bottomrule
\end{tabularx}
\end{table}
\renewcommand{\arraystretch}{1.0}

\begin{table}[!htbp]
\centering
\caption{\textbf{Retrospective enrichment under three distinct sequence-selection protocols.}
Within-site conditional sampling and finite-library Boltzmann selection use available measured libraries; 12-step score-guided coordinate search uses the 103 mutational DMS cohorts. All values are fold enrichment over uniform random selection ($1.00\times$). These are different sampling procedures, not three equivalent Gibbs samplers.}
\label{tab:app-gibbs-sampling}
\small
\renewcommand{\arraystretch}{1.13}
\setlength{\tabcolsep}{5pt}
\begin{tabularx}{\textwidth}{@{}>{\raggedright\arraybackslash}Xccc@{}}
\toprule
\textbf{Model} & Within-site & Finite-library Boltzmann & 12-step coordinate \\
\midrule
Random Uniform Baseline & $1.00\times$ & $1.00\times$ & $1.00\times$ \\
Unconditioned ESM-2 (650M) & $1.23\times$ & $2.20\times$ & $1.43\times$ \\
\textsc{LogiCA} (150M, zero-shot) & $1.31\times$ & $2.51\times$ & $1.53\times$ \\
\textsc{LogiCA} (650M, zero-shot) & $1.36\times$ & $2.74\times$ & $1.60\times$ \\
\textsc{LogiCA} Guided ($\alpha=0.5$) & $1.41\times$ & $2.86\times$ & $1.68\times$ \\
\midrule
\multicolumn{4}{@{}l}{\emph{Unseen LOO adaptation: 103 mutational DMS cohorts}} \\
w/ SiteFuse (150M, LOCO-30\%) & $1.06\times$ & $1.41\times$ & $1.17\times$ \\
w/ \textsc{LogiMLM}\genmark (150M, LOCO-30\%) & $1.28\times$ & $2.36\times$ & $1.55\times$ \\
\rowcolor{blue!8}
\textbf{w/ \textsc{LogiCA}\genmark (150M, LOCO-30\%)} & \textbf{1.44}$\times$ & \textbf{2.92}$\times$ & \textbf{1.85}$\times$ \\
\rowcolor{blue!8}
\textbf{w/ \textsc{LogiCA}\genmark (150M, LOLO)} & \textbf{1.56}$\times$ & \textbf{3.28}$\times$ & \textbf{2.12}$\times$ \\
\bottomrule
\end{tabularx}
\end{table}
\renewcommand{\arraystretch}{1.0}

\paragraph{Combinatorial optimization in TtgR.}
We next consider the Sensor-seq screen of the multidrug transcription factor
TtgR~\citep{nishikawa2024allosteric}, comprising $15{,}216$ combinatorial
variants with one to four substitutions at positions
$\{137,141,167,175\}$ under naringenin and phloretin. Starting from
deleterious quadruple mutants, 15-step coordinate search improves the
measured fitness of sampled variants. At step 15, \textsc{LogiCA} reaches
mean experimental fitness $Z=+1.289\sigma$ and top-$10\%$ enrichment
$1.92\times$, including mean fitness $+0.980\sigma$ for variants with
$d\ge 2$ substitutions. The corresponding mean fitness values are
$+0.674\sigma$ ($1.56\times$ enrichment) for ESM-2 and
$+0.184\sigma$ ($1.15\times$ enrichment) for SiteFuse
(Figure~\ref{fig:app-ttgr-gibbs}A,C).

\paragraph{Ligand-dependent site selection in TtgR.}
Ranking protein positions by the symmetrized context-dependent KL score
in Eq.~\ref{eq:app-kl-site-selection}, without using structural input,
identifies all four experimentally characterized pocket positions among the
top four candidate sites. Sampling on those sites approaches the oracle-pocket enrichment
($1.84\times$ versus $1.92\times$), compared with $1.14\times$ under
random site selection
(Figure~\ref{fig:app-ttgr-gibbs}B). This experiment connects the
perturbation-based dependency analysis with mutation-site selection for
generative search.

\begin{figure}[!htbp]
\centering
\includegraphics[width=1.0\linewidth]{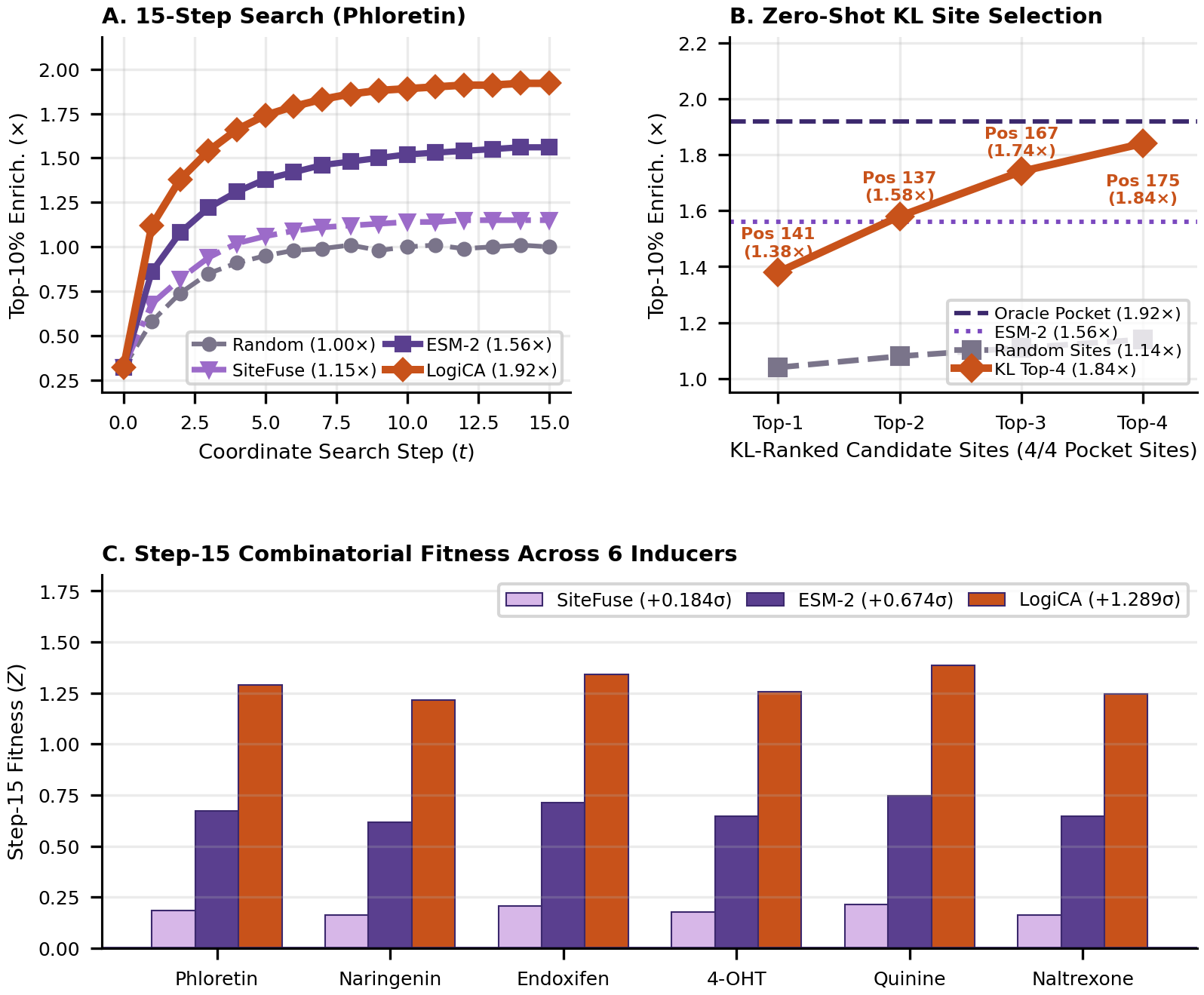}
\caption{\textbf{Conditioned search improves combinatorial fitness and identifies candidate pocket sites in TtgR.}
The retrospective Sensor-seq analysis uses measured multi-mutant variants from~\citet{nishikawa2024allosteric}.
\textbf{(A)} Fifteen-step coordinate-search trajectories in the phloretin-conditioned combinatorial library of 15,216 variants.
\textbf{(B)} Ligand-dependent site sensitivity, computed by symmetrized token-distribution KL divergence without structural input, ranks all four experimentally characterized pocket positions among the top four candidates.
\textbf{(C)} Step-15 experimental fitness ($Z$-score) across the six inducer conditions included in the broader analysis; panel A illustrates the phloretin-specific trajectory.}
\label{fig:app-ttgr-gibbs}
\end{figure}

}

\clearpage

\subsection{Scoring and training-objective ablations}

We close the protein--ligand evaluations with controlled analyses of the
bidirectional scoring rule, variant-pair selection, and regularization. These
experiments isolate modeling choices that are distinct from the architectural
and probabilistic properties established in Appendix~\ref{app:logica-theory}.

\paragraph{Directional mixture and entropy calibration.}
\label{app:alpha-ablation}

The main experiments use an equally weighted bidirectional score in Equation~\ref{eq:logica-bidirectional-score}. Here, we examine whether the
relative contribution of the two conditional directions should instead be
learned or initialized using their predictive uncertainty. We generalize the bidirectional \textsc{LogiCA} score to
\begin{equation}
  s_\alpha(x,y)
  =
  \alpha\,\ell_{A_x}(x\mid y)
  +
  (1-\alpha)\,\ell_{A_y}(y\mid x),
  \qquad
  \alpha\in[0,1],
  \label{eq:logica-bidirectional-score-alpha}
\end{equation}
where \(\ell_{A_x}(x\mid y)\) and \(\ell_{A_y}(y\mid x)\) denote the two
conditional logit-space scores. Because the modalities can differ substantially
in vocabulary size, tokenization, and predictive entropy, their raw conditional
scores need not have identical calibration.

We therefore consider an entropy-based initialization of \(\alpha\). For each
modality \(m\in\{x,y\}\), we estimate its mean token-level conditional entropy
on a small calibration subset \(\mathcal B\):
\begin{equation}
  H_m
  =
  \frac{1}{|\mathcal B|}
  \sum_{(x,y)\in\mathcal B}
  \frac{1}{|A_m|}
  \sum_{i\in A_m}
  \left[
    -\sum_{v\in\mathcal V_m}
    \pi_\theta^m(v\mid \cdot)
    \log \pi_\theta^m(v\mid \cdot)
  \right].
  \label{eq:modality-entropy}
\end{equation}
The entropies are computed using the pretrained or warm-started model before
full contrastive training and are subsequently treated as fixed calibration
statistics.

The entropy-calibrated initialization assigns greater weight to the direction
with lower predictive entropy:
\begin{equation}
  \alpha_0^{\mathrm{ent}}
  =
  \frac{H_y}{H_x+H_y}.
  \label{eq:entropy-alpha-initialization}
\end{equation}
Indeed, when \(H_x<H_y\), the coefficient of
\(\ell_{A_x}(x\mid y)\) is larger. To allow the ranking objective to modify this
initial preference, we parameterize
\begin{equation}
  \alpha
  =
  \sigma\!\left(
    \logit(\alpha_0)+b
  \right),
  \label{eq:entropy-scaled-alpha}
\end{equation}
where \(\alpha_0\) is either the symmetric initialization
\(\alpha_0^{\mathrm{sym}}=0.5\) or the entropy-calibrated initialization
\(\alpha_0^{\mathrm{ent}}\). In the fixed variants, \(b=0\) throughout
training; in the learned variants, \(b\) is optimized jointly with the model.
This produces four controlled settings: symmetric-fixed, symmetric-learned,
entropy-fixed, and entropy-learned.

We evaluate these variants on the unified low-homology BindingDB split, using
identical data, architectures, optimization settings, and five matched random
seeds. Separate masked forward passes give conditional entropies
\(H_{\mathrm{protein}}=2.522\) and
\(H_{\mathrm{ligand}}=0.0683\), corresponding to an entropy-calibrated
initialization of \(\alpha_0^{\mathrm{ent}}=0.0264\).

Table~\ref{tab:alpha-hard-ablation} shows no material improvement from
learning the directional weight or initializing it by token entropy.
The two symmetric variants perform almost identically, while entropy-based
variants are slightly lower and within seed-to-seed variation. We therefore
retain the simpler fixed $\alpha=\tfrac12$ for the main experiments.

\begin{table*}[h]
    \centering
    \caption{Directional-weight ablation on the low-homology BindingDB split. Cells report five-seed mean $\pm$ standard deviation. Learned variants optimize the directional mixture; fixed variants keep $\alpha$ at its initialization.}
    \label{tab:alpha-hard-ablation}
    \small
    \renewcommand{\arraystretch}{1.08}
    \setlength{\tabcolsep}{8pt}
    \begin{tabular}{lccc}
    \toprule
    \textbf{Initialization}
    & \textbf{Selected $\alpha$}
    & \textbf{AUROC}
    & \textbf{AUPR} \\
    \midrule
    \rowcolor{blue!6}
    \multicolumn{4}{l}{\textit{Learned directional weight}} \\
    \midrule
    \rowcolor{blue!10}
    \methodcell{\textbf{Symmetric}}{$\alpha_0{=}0.5$}
    & \stdcell{0.5049}{0.0020}
    & \stdcell{0.6304}{0.0031}
    & \stdcell{0.6023}{0.0058} \\
    \methodcell{Entropy calibrated}{$\alpha_0{=}0.0264$}
    & \stdcell{0.0278}{0.0002}
    & \stdcell{0.6276}{0.0041}
    & \stdcell{0.5998}{0.0075} \\
    \midrule
    \rowcolor{gray!8}
    \multicolumn{4}{l}{\textit{Fixed-weight controls}} \\
    \midrule
    \methodcell{Symmetric}{$\alpha{=}0.5$}
    & 0.5000
    & \stdcell{0.6305}{0.0017}
    & \stdcell{0.5997}{0.0037} \\
    \methodcell{Entropy calibrated}{$\alpha{=}0.0264$}
    & 0.0264
    & \stdcell{0.6278}{0.0042}
    & \stdcell{0.6005}{0.0046} \\
    \bottomrule
    \end{tabular}
\end{table*}
\renewcommand{\arraystretch}{1.0}


\paragraph{Contrastive pair selection and KL regularization.}
We examine whether the selection of variant pairs, rather than the contrastive
objective alone, controls ranking quality (Table~\ref{tab:app-loo-ablation}).
Uniform sampling ($S_1$) is dominated by neutral variants and performs
poorly. Selecting pairs with large experimental-fitness differences ($S_2$)
improves global correlation, but provides less improvement in top-hit
enrichment. The default stratified strategy ($S_3$) instead includes
upper-tail comparisons (50\%), within-site alternatives (30\%), and
large-margin global pairs (20\%). It reaches $\rho=0.313$ and
$\mathrm{P@1\%}=4.06\times$ on LOCO-30\%, compared with $0.305$ and
$2.49\times$ for $S_2$.

The KL trust-region penalty is also consequential: removing it lowers
LOCO-30\% correlation from $0.313$ to $0.232$ and enrichment from
$4.06\times$ to $2.01\times$. Together these ablations suggest that
hard upper-tail examples and restrained native-head adaptation both
contribute to the observed transfer advantage.

\begin{table}[!htbp]
\centering
\caption{\textbf{Upper-tail pair selection and KL regularization improve unseen-target variant ranking.}
Pair-sampling strategies $S_1$--$S_3$, KL penalty $\lambda_{\mathrm{KL}}$, and scoring controls are compared on 103 mutational DMS cohorts under LOCO-30\% (42 genes, 32 homology-disjoint protein clusters) and on 68 scaffold-disjoint multi-ligand cohorts under LOLO. The table retains global ranking and top-hit and coordinate-search metrics.}
\label{tab:app-loo-ablation}
\small
\renewcommand{\arraystretch}{1.08}
\setlength{\tabcolsep}{4.5pt}
\resizebox{\textwidth}{!}{%
\begin{tabular}{lcccccccc}
\toprule
\multirow{2}{*}{\textbf{Model / Ablation Variant}} & \multicolumn{4}{c}{\textbf{LOCO-30\% ($N{=}103$ DMS, $G{=}42$)}} & \multicolumn{4}{c}{\textbf{LOLO ($N{=}68$ multi-ligand)}} \\
\cmidrule(lr){2-5} \cmidrule(lr){6-9}
 & $\rho$ & $\mathrm{P@1\%}$ & $\mathrm{P@10\%}$ & $\mathrm{Gibbs}$ & $\rho$ & $\mathrm{P@1\%}$ & $\mathrm{P@10\%}$ & $\mathrm{Gibbs}$ \\
\midrule
ESM-2 (frozen baseline) & $0.273$ & $2.73\times$ & $2.03\times$ & $1.43\times$ & $0.259$ & $2.65\times$ & $1.98\times$ & $1.39\times$ \\
w/ \textsc{LogiCA}\genmark (zero-shot BindingDB) & $0.285$ & $3.38\times$ & $2.18\times$ & $1.53\times$ & $0.271$ & $3.22\times$ & $2.11\times$ & $1.48\times$ \\
w/ SiteFuse (site latent + MLP) & $0.289$ & $2.62\times$ & $2.12\times$ & $1.17\times$ & $0.282$ & $2.75\times$ & $2.08\times$ & $1.19\times$ \\
w/ \textsc{LogiMLM}\genmark (pos.-only MLM) & $0.293$ & $2.58\times$ & $2.05\times$ & $1.55\times$ & $0.312$ & $2.90\times$ & $2.18\times$ & $1.52\times$ \\
w/ \textsc{LogiCA}\genmark ($S_1$: uniform pairs) & $0.251$ & $2.12\times$ & $1.90\times$ & $1.42\times$ & $0.268$ & $2.34\times$ & $1.96\times$ & $1.45\times$ \\
w/ \textsc{LogiCA}\genmark ($S_2$: large-$|\Delta Z|$) & $0.305$ & $2.49\times$ & $2.08\times$ & $1.59\times$ & $0.338$ & $3.12\times$ & $2.29\times$ & $1.74\times$ \\
w/ \textsc{LogiCA}\genmark ($S_3$, $\lambda_{\mathrm{KL}}{=}0$) & $0.232$ & $2.01\times$ & $1.79\times$ & $1.38\times$ & $0.254$ & $2.28\times$ & $1.88\times$ & $1.41\times$ \\
\rowcolor{blue!10}
\textbf{w/ \textsc{LogiCA}\genmark ($S_3$, $\lambda_{\mathrm{KL}}{=}0.20$)} & \textbf{0.313} & \textbf{4.06}$\times$ & \textbf{2.22}$\times$ & \textbf{1.85}$\times$ & \textbf{0.354} & \textbf{4.62}$\times$ & \textbf{2.45}$\times$ & \textbf{2.12}$\times$ \\
\bottomrule
\end{tabular}%
}
\end{table}

\clearpage
\section{Additional TCR--Peptide Results and Ablations}
\label{app:tcr-extended-ablations}

This appendix expands the TCR--peptide evaluation in
Section~\ref{sec:exp-tcr-epitope}. We report detailed zero-shot ranking
comparisons, upper-tail prioritization, and inter-chain dependency analysis
before assessing anchor direction, contrastive pretraining, and backbone
capacity. All mutation-activity benchmarks are held out from the training
stages described in Appendix~\ref{app:tcrbinding}.

\subsection{Zero-shot peptide and TCR mutation ranking}
\label{app:tcr-per-dataset}

The complete TCR--peptide mutation rankings are presented separately as
Spearman correlations (Table~\ref{tab:mutation-affinity-correlation}) and
Pearson correlations (Table~\ref{tab:mutation-affinity-pearson}). The benchmark
covers 71 peptide-mutation groups (10,064 variants) and 23 TCR-mutation
groups (703 variants). The TCR-anchored configuration attains mean
Spearman $\rho=0.271$ and $0.190$ on the two sides of the interface,
respectively. Reported standard deviations describe heterogeneity between
mutation landscapes, not uncertainty in the mean.

\begin{table}[!p]
\centering
\caption{Zero-shot TCR--peptide variant ranking: Spearman $\rho$ (mean $\pm$ s.d. by benchmark). ePytope, BATCAVE, and ATLAS-PEP are peptide-mutation tests; the remaining groups evaluate TCR mutations. Bold and underline retain the original within-column emphasis. The complementary correlation metric is in Table~\ref{tab:mutation-affinity-pearson}.}
\label{tab:mutation-affinity-correlation}
\small
\renewcommand{\arraystretch}{1.15}
\setlength{\tabcolsep}{4pt}
\resizebox{\textwidth}{!}{%
\begin{tabular}{@{}lcccccc@{}}
\toprule
\textbf{Method} & \shortstack{\textbf{ePytope}\\{\scriptsize(26)}} & \shortstack{\textbf{BATCAVE}\\{\scriptsize(35)}} & \shortstack{\textbf{ATLAS-PEP}\\{\scriptsize(10)}} & \shortstack{\textbf{Affinity sets}\\{\scriptsize(21)}} & \shortstack{\textbf{868-Z11}\\{\scriptsize(2)}} & \shortstack{\textbf{TCR combined}\\{\scriptsize(23)}} \\
\midrule
\rowcolor{gray!6}
\multicolumn{7}{@{}l}{\emph{Matched contextualized models}} \\
\methodcell{w/ \textsc{LogiMLM}\genmark}{continued MLM} & \stdcell{-0.040}{0.216} & \stdcell{0.000}{0.225} & \stdcell{0.527}{0.312} & \stdcell{0.035}{0.429} & \stdcell{0.137}{0.135} & \stdcell{0.044}{0.411} \\
\methodcell{w/ LatentCA-TCR}{pooled latent scoring} & \stdcell{0.078}{0.194} & \stdcell{0.054}{0.207} & \stdcell{0.412}{0.346} & \stdcell{-0.060}{0.430} & \stdcell{-0.049}{0.040} & \stdcell{-0.059}{0.410} \\
\rowcolor{blue!8}
\methodcell{\textbf{w/ \textsc{LogiCA}-TCR\genmark}}{\textbf{TCR-anchored}} & \beststdcell{0.229*}{0.168} & \beststdcell{0.170*}{0.191} & \beststdcell{0.731}{0.238} & \beststdcell{0.194}{0.475} & \understdcell{0.148}{0.167} & \beststdcell{0.190*}{0.455} \\
\rowcolor{blue!8}
\methodcell{\textbf{w/ \textsc{LogiCA}-Pep\genmark}}{\textbf{peptide-anchored}} & \stdcell{0.147}{0.313} & \stdcell{0.109}{0.294} & \stdcell{0.509}{0.468} & \stdcell{0.101}{0.428} & \stdcell{0.146}{0.119} & \stdcell{0.105}{0.409} \\
\rowcolor{blue!8}
\methodcell{\textbf{w/ \textsc{LogiCA}-Dual\genmark}}{\textbf{dual-anchor}} & \stdcell{0.180}{0.314} & \stdcell{0.132}{0.295} & \stdcell{0.654}{0.380} & \stdcell{0.050}{0.471} & \beststdcell{0.154}{0.121} & \stdcell{0.059}{0.451} \\
\midrule
\rowcolor{gray!6}
\multicolumn{7}{@{}l}{\emph{External reference methods}} \\
\methodcell{EPACT~\citep{zhang2024epitope}}{embedding contrastive} & \stdcell{0.047}{0.211} & \stdcell{0.016}{0.220} & \stdcell{0.233}{0.465} & \stdcell{0.125}{0.463} & \stdcell{-0.159}{0.001} & \stdcell{0.100}{0.449} \\
\methodcell{ERGO-II~\citep{springer2021ergoii}}{end-to-end classifier} & \stdcell{-0.031}{0.173} & \stdcell{-0.044}{0.169} & \stdcell{0.185}{0.581} & \stdcell{-0.042}{0.467} & \stdcell{0.005}{0.027} & \stdcell{-0.037}{0.446} \\
\methodcell{ImRex~\citep{moris2021imrex}}{end-to-end classifier} & \stdcell{0.094}{0.182} & \stdcell{0.102}{0.173} & \stdcell{0.543}{0.351} & \understdcell{0.174}{0.405} & \stdcell{0.062}{0.088} & \understdcell{0.164}{0.388} \\
\methodcell{iTCep~\citep{zhang2023itcep}}{end-to-end classifier} & \understdcell{0.207}{0.219} & \understdcell{0.162}{0.222} & \stdcell{0.321}{0.333} & \stdcell{-0.021}{0.354} & \stdcell{0.006}{0.008} & \stdcell{-0.019}{0.338} \\
\methodcell{TEIM~\citep{peng2023teim}}{end-to-end classifier} & \stdcell{0.041}{0.161} & \stdcell{0.053}{0.166} & \stdcell{0.360}{0.563} & \stdcell{-0.003}{0.460} & \stdcell{0.115}{0.162} & \stdcell{0.007}{0.441} \\
\methodcell{TCR-T5\genmark~\citep{karthikeyan2025tcrt5}}{generative MLM} & \stdcell{0.023}{0.139} & \stdcell{0.027}{0.148} & \stdcell{-0.031}{0.541} & \stdcell{0.070}{0.372} & \stdcell{0.020}{0.028} & \stdcell{0.065}{0.355} \\
\methodcell{TULIP-TCR\genmark~\citep{meynard2024tulip}}{generative MLM} & \stdcell{0.147}{0.160} & \stdcell{0.132}{0.150} & \understdcell{0.641}{0.255} & \stdcell{0.099}{0.535} & \stdcell{0.102}{0.023} & \stdcell{0.100}{0.510} \\
\bottomrule
\end{tabular}%
}
\end{table}
\renewcommand{\arraystretch}{1.0}

\begin{table}[!p]
\centering
\caption{Zero-shot TCR--peptide variant ranking: Pearson $r$ (mean $\pm$ s.d. by benchmark). ePytope, BATCAVE, and ATLAS-PEP are peptide-mutation tests; the remaining groups evaluate TCR mutations. Bold and underline retain the original within-column emphasis. The complementary correlation metric is in Table~\ref{tab:mutation-affinity-correlation}.}
\label{tab:mutation-affinity-pearson}
\small
\renewcommand{\arraystretch}{1.15}
\setlength{\tabcolsep}{4pt}
\resizebox{\textwidth}{!}{%
\begin{tabular}{@{}lcccccc@{}}
\toprule
\textbf{Method} & \shortstack{\textbf{ePytope}\\{\scriptsize(26)}} & \shortstack{\textbf{BATCAVE}\\{\scriptsize(35)}} & \shortstack{\textbf{ATLAS-PEP}\\{\scriptsize(10)}} & \shortstack{\textbf{Affinity sets}\\{\scriptsize(21)}} & \shortstack{\textbf{868-Z11}\\{\scriptsize(2)}} & \shortstack{\textbf{TCR combined}\\{\scriptsize(23)}} \\
\midrule
\rowcolor{gray!6}
\multicolumn{7}{@{}l}{\emph{Matched contextualized models}} \\
\methodcell{w/ \textsc{LogiMLM}\genmark}{continued MLM} & \stdcell{0.039}{0.230} & \stdcell{0.063}{0.227} & \stdcell{0.428}{0.438} & \stdcell{0.030}{0.404} & \stdcell{0.121}{0.169} & \stdcell{0.038}{0.387} \\
\methodcell{w/ LatentCA-TCR}{pooled latent scoring} & \stdcell{0.104}{0.205} & \stdcell{0.081}{0.218} & \stdcell{0.465}{0.412} & \stdcell{-0.099}{0.479} & \stdcell{-0.069}{0.028} & \stdcell{-0.096}{0.457} \\
\rowcolor{blue!8}
\methodcell{\textbf{w/ \textsc{LogiCA}-TCR\genmark}}{\textbf{TCR-anchored}} & \beststdcell{0.296*}{0.219} & \beststdcell{0.223*}{0.229} & \beststdcell{0.632}{0.286} & \understdcell{0.170}{0.433} & \understdcell{0.140}{0.193} & \understdcell{0.167}{0.415} \\
\rowcolor{blue!8}
\methodcell{\textbf{w/ \textsc{LogiCA}-Pep\genmark}}{\textbf{peptide-anchored}} & \stdcell{0.176}{0.288} & \stdcell{0.144}{0.268} & \stdcell{0.398}{0.466} & \stdcell{0.084}{0.436} & \understdcell{0.140}{0.128} & \stdcell{0.089}{0.417} \\
\rowcolor{blue!8}
\methodcell{\textbf{w/ \textsc{LogiCA}-Dual\genmark}}{\textbf{dual-anchor}} & \understdcell{0.227}{0.333} & \understdcell{0.177}{0.310} & \stdcell{0.534}{0.425} & \stdcell{0.050}{0.463} & \beststdcell{0.153}{0.127} & \stdcell{0.059}{0.443} \\
\midrule
\rowcolor{gray!6}
\multicolumn{7}{@{}l}{\emph{External reference methods}} \\
\methodcell{EPACT~\citep{zhang2024epitope}}{embedding contrastive} & \stdcell{0.016}{0.180} & \stdcell{0.037}{0.271} & \stdcell{0.357}{0.454} & \stdcell{0.167}{0.440} & \stdcell{-0.136}{0.059} & \stdcell{0.140}{0.429} \\
\methodcell{ERGO-II~\citep{springer2021ergoii}}{end-to-end classifier} & \stdcell{-0.007}{0.201} & \stdcell{-0.029}{0.173} & \stdcell{0.182}{0.732} & \stdcell{0.018}{0.524} & \stdcell{-0.093}{0.063} & \stdcell{0.009}{0.500} \\
\methodcell{ImRex~\citep{moris2021imrex}}{end-to-end classifier} & \stdcell{0.089}{0.183} & \stdcell{0.089}{0.177} & \stdcell{0.535}{0.423} & \beststdcell{0.187}{0.416} & \stdcell{0.006}{0.008} & \beststdcell{0.171}{0.400} \\
\methodcell{iTCep~\citep{zhang2023itcep}}{end-to-end classifier} & \stdcell{0.094}{0.190} & \stdcell{0.070}{0.172} & \stdcell{0.219}{0.475} & \stdcell{0.041}{0.359} & \stdcell{-0.042}{0.059} & \stdcell{0.034}{0.343} \\
\methodcell{TEIM~\citep{peng2023teim}}{end-to-end classifier} & \stdcell{0.051}{0.171} & \stdcell{0.076}{0.179} & \stdcell{0.225}{0.659} & \stdcell{0.004}{0.436} & \stdcell{0.139}{0.196} & \stdcell{0.015}{0.420} \\
\methodcell{TCR-T5\genmark~\citep{karthikeyan2025tcrt5}}{generative MLM} & \stdcell{0.028}{0.129} & \stdcell{0.032}{0.125} & \stdcell{0.047}{0.556} & \stdcell{0.060}{0.392} & \stdcell{0.021}{0.029} & \stdcell{0.056}{0.374} \\
\methodcell{TULIP-TCR\genmark~\citep{meynard2024tulip}}{generative MLM} & \stdcell{0.181}{0.175} & \stdcell{0.170}{0.162} & \understdcell{0.607}{0.347} & \stdcell{0.143}{0.554} & \stdcell{0.122}{0.076} & \stdcell{0.141}{0.528} \\
\bottomrule
\end{tabular}%
}
\end{table}
\renewcommand{\arraystretch}{1.0}

\subsection{Upper-tail prioritization and sequence design}

We evaluate how variant rankings translate into limited-budget hit
prioritization, as well as how the retained token distributions support
iterative sequence search. Table~\ref{tab:tcr-top5} reports limited-budget hit selection, whereas
Table~\ref{tab:tcr-top5-generation} reports coordinate-search enrichment.
These are complementary uses of native token probabilities, rather than
the same evaluation metric. TCR anchoring is strongest for mean mutation
ranking, while peptide anchoring gives the highest reported 12-step
sampling enrichment ($4.51\times$ on peptides and $1.65\times$ on TCRs).

\begin{table}[!p]
\centering
\caption{\textbf{Anchor direction shapes upper-tail prioritization of peptide and TCR variants.}
Evaluation covers 65 peptide and 16 TCR mutation sets with at least five variants. Values are fold enrichment over uniform expectation; Top-5 $>\mathrm{WT}$ indicates prioritization of wild-type-improving agonists. P@10\% is appropriate for short peptide libraries, where P@1\% may select a single variant. $\dagger$: native token head; $*$: $p<0.01$ versus \textsc{LogiMLM}.}
\label{tab:tcr-top5}
\small
\renewcommand{\arraystretch}{1.13}
\setlength{\tabcolsep}{4pt}
\resizebox{\textwidth}{!}{%
\begin{tabular}{@{}lccccccc@{}}
\toprule
\textbf{Method} & \textbf{ePytope P@10\%} & \textbf{BATCAVE P@10\%} & \textbf{Peptide P@10\%} & \textbf{Peptide P@5\%} & \textbf{Top-5 $>\mathrm{WT}$} & \textbf{TCR P@10\%} & \textbf{TCR P@5\%} \\
\midrule
w/ \textsc{LogiMLM}\genmark & $1.96\times$ & $1.70\times$ & $1.82\times$ & $1.30\times$ & $0.53\times$ & $1.11\times$ & $0.83\times$ \\
w/ LatentCA-TCR & $1.15\times$ & $1.04\times$ & $1.10\times$ & $1.62\times$ & $0.99\times$ & $1.23\times$ & $1.11\times$ \\
\rowcolor{blue!8}
\textbf{w/ \textsc{LogiCA}-TCR\genmark} & \textbf{2.95}$^*\times$ & \textbf{2.38}$^*\times$ & \textbf{2.62}$^*\times$ & \underline{3.02}$^*\times$ & \underline{1.86}$\times$ & \textbf{1.60}$\times$ & \textbf{1.66}$\times$ \\
\rowcolor{blue!8}
\textbf{w/ \textsc{LogiCA}-Pep\genmark} & $1.64\times$ & $1.62\times$ & $1.63\times$ & \underline{3.02}$^*\times$ & $1.38\times$ & \underline{1.35}$\times$ & \underline{1.38}$\times$ \\
\rowcolor{blue!8}
\textbf{w/ \textsc{LogiCA}-Dual\genmark} & \underline{2.67}$^*\times$ & $2.25^*\times$ & \underline{2.44}$^*\times$ & \textbf{4.01}$^*\times$ & \textbf{2.12}$^*\times$ & \underline{1.35}$\times$ & \underline{1.38}$\times$ \\
\midrule
TULIP-TCR\genmark & $2.63\times$ & \underline{2.28}$\times$ & $2.44\times$ & $2.85\times$ & $1.50\times$ & $1.11\times$ & \textbf{1.66}$\times$ \\
TCR-T5\genmark & $0.73\times$ & $0.74\times$ & $0.73\times$ & $0.74\times$ & $0.85\times$ & $0.86\times$ & $0.55\times$ \\
iTCep & $2.03\times$ & $1.93\times$ & $1.96\times$ & $1.55\times$ & $1.05\times$ & $0.86\times$ & $0.28\times$ \\
ImRex & $0.77\times$ & $0.74\times$ & $0.75\times$ & $0.53\times$ & $0.78\times$ & $0.86\times$ & $0.83\times$ \\
EPACT & $1.02\times$ & $0.85\times$ & $0.93\times$ & $0.81\times$ & $1.05\times$ & $0.98\times$ & \textbf{1.66}$\times$ \\
\bottomrule
\end{tabular}%
}
\end{table}
\renewcommand{\arraystretch}{1.0}

\begin{table}[!p]
\centering
\caption{\textbf{Peptide-anchored training gives the highest reported coordinate-search enrichment.}
Twelve-step top-10\% enrichment for peptide and TCR mutation sets, expressed relative to uniform sampling. Dashes indicate unavailable scalar-head generation protocols. Finite-budget ranking results are reported separately in Table~\ref{tab:tcr-top5}. $\dagger$: native token head; $*$: $p<0.01$ versus \textsc{LogiMLM}.}
\label{tab:tcr-top5-generation}
\small
\renewcommand{\arraystretch}{1.13}
\setlength{\tabcolsep}{4pt}
\resizebox{0.55\textwidth}{!}{%
\begin{tabular}{@{}lcccc@{}}
\toprule
\textbf{Method} & \textbf{ePytope} & \textbf{BATCAVE} & \textbf{All peptide} & \textbf{TCR} \\
\midrule
w/ \textsc{LogiMLM}\genmark & $1.58\times$ & $1.30\times$ & $1.42\times$ & $0.89\times$ \\
w/ LatentCA-TCR & -- & -- & -- & -- \\
\rowcolor{blue!8}
\textbf{w/ \textsc{LogiCA}-TCR\genmark} & $2.45\times$ & $1.82\times$ & $2.09\times$ & \underline{1.48}$^*\times$ \\
\rowcolor{blue!8}
\textbf{w/ \textsc{LogiCA}-Pep\genmark} & \textbf{4.77}$^*\times$ & \textbf{4.31}$^*\times$ & \textbf{4.51}$^*\times$ & \textbf{1.65}$^*\times$ \\
\rowcolor{blue!8}
\textbf{w/ \textsc{LogiCA}-Dual\genmark} & \underline{4.67}$^*\times$ & \underline{4.23}$^*\times$ & \underline{4.41}$^*\times$ & $1.45^*\times$ \\
\midrule
TULIP-TCR\genmark & $1.42\times$ & $1.38\times$ & $1.40\times$ & $1.31\times$ \\
TCR-T5\genmark & -- & -- & -- & $1.10\times$ \\
iTCep & -- & -- & -- & -- \\
ImRex & -- & -- & -- & -- \\
EPACT & -- & -- & -- & -- \\
\bottomrule
\end{tabular}%
}
\end{table}
\renewcommand{\arraystretch}{1.0}

\subsection{Cross-modal dependency and contact localization}
\label{app:tcr-contacts-binary}

Table~\ref{tab:dependency-map-interchain} reports zero-shot inter-chain contact-localization performance across $n=251$ experimentally resolved TCR--pMHC crystal structures in TCR3d~\citep{gowthaman2019tcr3d} for $\mathrm{CDR3}\alpha$--$\mathrm{CDR3}\beta$, Peptide--$\mathrm{CDR3}\alpha$, and Peptide--$\mathrm{CDR3}\beta$ interfaces. The complementary thresholded ePytope classification analysis is presented in Appendix~\ref{app:epytope-auc}.

\begin{table*}[!htbp]
    \centering
    \caption{Dependency-map prediction performance for inter-chain interactions, evaluated against ground-truth contact maps. Cells show mean AUC or AUPR $\pm$ standard deviation, averaged over $n=251$ TCR-pMHC structures. Best values are shown in \textbf{bold}, and second-best values are \underline{underlined}.}
    \label{tab:dependency-map-interchain}
    \small
    \renewcommand{\arraystretch}{1.04}
    \resizebox{0.9\textwidth}{!}{%
    \begin{tabular}{llcc}
    \toprule
    \textbf{Region} & \textbf{Model} & \textbf{AUC} & \textbf{AUPR} \\
    \midrule

    \multirow{7}{*}{\shortstack[c]{CDR3$\alpha$--CDR3$\beta$ \\ ($n = 251$)}}
      & ESM-2~\citep{lin2023evolutionary}      & 0.558 $\pm$ 0.227 & 0.027 $\pm$ 0.027 \\
      & TULIP-TCR~\citep{meynard2024tulip} & 0.449 $\pm$ 0.190 & 0.017 $\pm$ 0.009 \\
      & TCRlang~\citep{olsen2024addressing}  & 0.502 $\pm$ 0.220 & 0.035 $\pm$ 0.114 \\
      & \cellcolor{blue!10}w/ \textsc{LogiMLM} & \cellcolor{blue!10} \textbf{0.667 $\pm$ 0.214} & \cellcolor{blue!10} \textbf{0.062 $\pm$ 0.088} \\
      & \cellcolor{blue!10}w/ \textsc{LogiCA}-TCR
      & \cellcolor{blue!10}\underline{0.633 $\pm$ 0.190}
      & \cellcolor{blue!10}\underline{0.054 $\pm$ 0.083} \\
      & \cellcolor{blue!10}w/ \textsc{LogiCA}-Pep
      & \cellcolor{blue!10}0.600 $\pm$ 0.195
      & \cellcolor{blue!10}0.031 $\pm$ 0.028 \\
      & \cellcolor{blue!10}w/ \textsc{LogiCA}-Dual
      & \cellcolor{blue!10}0.578 $\pm$ 0.196
      & \cellcolor{blue!10}0.034 $\pm$ 0.060 \\
    \midrule

    \multirow{6}{*}{\shortstack[c]{Peptide--CDR3$\alpha$ \\ ($n = 251$)}}
      & ESM-2~\citep{lin2023evolutionary}      & 0.493 $\pm$ 0.212 & 0.039 $\pm$ 0.055 \\
      & TULIP-TCR~\citep{meynard2024tulip} & 0.531 $\pm$ 0.243 & 0.041 $\pm$ 0.051 \\
      & \cellcolor{blue!10}w/ \textsc{LogiMLM} & \cellcolor{blue!10}\underline{0.553 $\pm$ 0.193} & \cellcolor{blue!10}0.039 $\pm$ 0.054 \\
      & \cellcolor{blue!10}w/ \textsc{LogiCA}-TCR
      & \cellcolor{blue!10}\textbf{0.592 $\pm$ 0.178}
      & \cellcolor{blue!10}\underline{0.047 $\pm$ 0.114} \\
      & \cellcolor{blue!10}w/ \textsc{LogiCA}-Pep
      & \cellcolor{blue!10}0.544 $\pm$ 0.239
      & \cellcolor{blue!10}0.046 $\pm$ 0.113 \\
      & \cellcolor{blue!10}w/ \textsc{LogiCA}-Dual
      & \cellcolor{blue!10}0.503 $\pm$ 0.250
      & \cellcolor{blue!10}\textbf{0.049 $\pm$ 0.126} \\
    \midrule

    \multirow{7}{*}{\shortstack[c]{Peptide--CDR3$\beta$ \\ ($n = 251$)}}
      & ESM-2~\citep{lin2023evolutionary}      & 0.451 $\pm$ 0.224 & 0.031 $\pm$ 0.046 \\
      & TULIP-TCR~\citep{meynard2024tulip} & 0.455 $\pm$ 0.189 & 0.024 $\pm$ 0.021 \\
      & TCR-T5~\citep{karthikeyan2025tcrt5}    & 0.648 $\pm$ 0.243 & \underline{0.077 $\pm$ 0.108} \\
      & \cellcolor{blue!10}w/ \textsc{LogiMLM} & \cellcolor{blue!10}0.717 $\pm$ 0.153 & \cellcolor{blue!10}0.058 $\pm$ 0.065 \\
      & \cellcolor{blue!10}w/ \textsc{LogiCA}-TCR
      & \cellcolor{blue!10}\textbf{0.743 $\pm$ 0.131}
      & \cellcolor{blue!10}0.075 $\pm$ 0.145 \\
      & \cellcolor{blue!10}w/ \textsc{LogiCA}-Pep
      & \cellcolor{blue!10}0.712 $\pm$ 0.169
      & \cellcolor{blue!10}0.063 $\pm$ 0.073 \\
      & \cellcolor{blue!10}w/ \textsc{LogiCA}-Dual
      & \cellcolor{blue!10}\underline{0.733 $\pm$ 0.166}
      & \cellcolor{blue!10}\textbf{0.094 $\pm$ 0.162} \\
    \bottomrule
    \end{tabular}%
    }
\end{table*}
\renewcommand{\arraystretch}{1.0}

{

\subsection{Contrastive pretraining and anchor-direction ablations}

The following controls separate the contribution of anchor direction and
negative-sample budget from that of local mutation-sensitive contrastive
pretraining. Both analyses use the same held-out mutation benchmarks as the
primary TCR experiments.

\paragraph{Anchor direction and negative-budget controls.}
\label{app:tcr-factorial}

The released TCR- and peptide-anchored checkpoints differ in both the
anchor direction and the number of synthetic negatives. To separate these
factors, Table~\ref{tab:tcr-2x2-factorial} evaluates both anchor directions
at $K=1$ and $K=5$, alongside the dual-anchor configuration. At matched
negative count, TCR anchoring gives better or equal aggregate ranking on
the 71 peptide and 23 TCR mutation groups; peptide anchoring retains an
advantage on selected multi-mutant and iterative-design readouts. This
asymmetry is consistent with the much greater diversity of TCRs than
peptides in the training corpus, although the factorial itself does not
establish the causal mechanism. A separate exact epitope-cluster sign-flip
test on the full BATCAVE panel remains significant after accounting for
shared epitopes ($p=0.0023$).

\begin{table}[!htbp]
\centering
\caption{\textbf{Controlled TCR--peptide anchor-direction and negative-budget factorial.}
The five checkpoints vary anchor direction (TCR or peptide) and the number
of synthetic negatives $K\in\{1,5\}$. The table is split into
(A) dataset-specific mutation-ranking correlations and (B) aggregate mutation
ranking and repertoire classification. Parentheses in TCR DMS cells
report the pure holo-LLR correlations; the corresponding released checkpoint identifiers are provided below.}
\label{tab:tcr-2x2-factorial}
\small
\renewcommand{\arraystretch}{1.14}
\setlength{\tabcolsep}{5pt}
\noindent\textbf{A. Mutation ranking by benchmark (Spearman $\rho$)}\par\vspace{2pt}
\begin{tabularx}{\textwidth}{@{}>{\raggedright\arraybackslash}Xcccc@{}}
\toprule
\textbf{Anchor ($K$)} & \textbf{ePytope (26)} & \textbf{BATCAVE (66)} & \textbf{ATLAS-PEP (10)} & \textbf{ATLAS-CDR (8)} \\
\midrule
TCR ($K=1$) & $0.134$ & $0.127$ & $0.573$ & $0.130$ \\
TCR ($K=5$) & $0.113$ & $0.116$ & $0.501$ & $0.172$ \\
Peptide ($K=1$) & \textbf{0.197} & \textbf{0.150} & $0.561$ & $-0.009$ \\
Peptide ($K=5$) & $0.092$ & $0.106$ & \textbf{0.614} & \underline{0.243} \\
\rowcolor{blue!10}\textbf{Dual ($5{:}1$, $K=5$)} & \underline{0.174} & \underline{0.144} & \underline{0.565} & \textbf{0.252} \\
\bottomrule
\end{tabularx}
\par\vspace{9pt}
\noindent\textbf{B. Aggregate mutation ranking and repertoire classification}\par\vspace{2pt}
\begin{tabularx}{\textwidth}{@{}>{\raggedright\arraybackslash}Xcccc@{}}
\toprule
\textbf{Anchor ($K$)} & \shortstack{\textbf{Peptide DMS}\\\textbf{($N=71$)}} & \shortstack{\textbf{TCR DMS}\\\textbf{($N=23$)}} & \shortstack{\textbf{ePytope binary}\\\textbf{AUROC}} & \shortstack{\textbf{ImmRep25}\\\textbf{AUROC}} \\
\midrule
TCR ($K=1$) & \textbf{0.271} & \textbf{0.190 (0.201)} & $0.452$ & $0.515$ \\
TCR ($K=5$) & $0.179$ & \underline{0.156 (0.170)} & $0.453$ & \textbf{0.530} \\
Peptide ($K=1$) & \underline{0.249} & $-0.011$ (0.125) & $0.493$ & $0.509$ \\
Peptide ($K=5$) & $0.179$ & $0.105$ (0.048) & \textbf{0.546} & \underline{0.520} \\
\rowcolor{blue!10}\textbf{Dual ($5{:}1$, $K=5$)} & $0.223$ & $0.059$ (0.040) & \underline{0.520} & $0.511$ \\
\bottomrule
\end{tabularx}
\par\vspace{4pt}
\noindent\footnotesize\emph{Checkpoint identifiers.} In row order:
\texttt{tcr-anchored-neg1}, \texttt{tcr-anchored-neg5},
\texttt{pep-anchored-neg1}, \texttt{pep-anchored-neg5},
and \texttt{dual-anchored-ratio5}.
\end{table}

\paragraph{Two-stage training and local-perturbation pretraining.}
\label{app:tcr-twostage}

We test whether coarse paired-repertoire classification alone yields
mutation-sensitive native token scores (Table~\ref{tab:tcr-twostage-ablation}).
A fine-tuning-only model, trained on shuffled observed peptides but without
Stage~1 local-perturbation pretraining, learns repertoire-level
classification (validation AUROC $0.893$) yet ranks single-residue
functional changes poorly. Adding Stage~1 synthetic-mutation contrasts
increases zero-shot mutation-ranking correlation to $\rho=0.271$ for
peptides and $0.190$ for TCRs, with ePytope binary AUROC $0.672$.
This supports the value of local mutation-sensitive pretraining beyond
supervised binding discrimination alone.

\begin{table}[!htbp]
\centering
\caption{Ablation isolating Stage~1 synthetic-mutation contrastive pretraining vs.\ direct Stage~2 supervised fine-tuning only (shuffled non-cognate peptide negatives) and positive-only \textsc{LogiMLM} across thresholded ePytope binary classification ($26$ groups) and zero-shot mutation ranking ($71$ peptide DMS and $23$ TCR DMS groups).}
\label{tab:tcr-twostage-ablation}
\small
\renewcommand{\arraystretch}{1.12}
\setlength{\tabcolsep}{5.5pt}
\resizebox{\textwidth}{!}{%
\begin{tabular}{lcccc}
\toprule
\textbf{Evaluation Metric} & \textbf{\textsc{LogiMLM}} & \shortstack{\textbf{Fine-tuning only}\\{\scriptsize (supervised head)}} & \shortstack{\textbf{Fine-tuning only}\\{\scriptsize (native logits)}} & \textbf{\textsc{LogiCA}-TCR (two-stage)} \\
\midrule
ePytope binary AUROC ($26$ sets) & $0.524$ & $0.567$ & $0.523$ & \textbf{0.672} \\
ePytope binary AUPR ($26$ sets) & $0.341$ & $0.311$ & $0.330$ & \textbf{0.485} \\
Peptide mutation Spearman $\rho$ ($71$ DMS) & $0.060$ & $0.075$ & $0.115$ & \textbf{0.271} \\
TCR mutation Spearman $\rho$ ($23$ DMS) & $0.044$ & $0.089$ & $0.107$ & \textbf{0.190} \\
\bottomrule
\end{tabular}%
}
\end{table}

}

\subsection{Peptide encoder scaling and supplementary classification}

This final subsection reports threshold-based classification readouts on the
mutation benchmarks and model-capacity sensitivity, complementing the
continuous variant-ranking analyses above.

\paragraph{Thresholded ePytope classification.}

\label{app:epytope-auc}

Table~\ref{app:epytope-mutation-binary} reports a supplementary binary classification analysis on ePytope using the benchmark-provided labels derived from epitope-specific NFAT thresholds. This evaluation tests whether the variant scores used for mutation ranking also separate the thresholded active and inactive TCR--epitope pairs within this mutation dataset. In addition to AUC and AUPR, we report AUC0.1, BEDROC, and logAUC as early-retrieval metrics that emphasize whether active pairs are ranked near the top of the prediction list. \textsc{LogiCA}-TCR performs best on most metrics, including AUC (0.672), AUPR (0.485), AUC0.1 (0.586), and BEDROC (0.541), while \textsc{LogiCA}-Dual achieves the best logAUC (0.093) and remains competitive across the other metrics.

\begin{table*}[!htbp]
    \centering
    \caption{Binary classification performance on the ePytope benchmark. The blue block compares contextualized sequence-model variants, including \textsc{LogiMLM} and \textsc{LogiCA} variants; the gray block lists external baselines. The $\dagger$ marker denotes methods that retain a native-vocabulary generative interface. Performance is reported as mean $\pm$ standard deviation across ePytope mutation sets. Best values are shown in bold, and second-best values are underlined.}
    \label{app:epytope-mutation-binary}
    \small
    \setlength{\tabcolsep}{4pt}
    \renewcommand{\arraystretch}{1.08}
    \resizebox{0.9\textwidth}{!}{%
    \begin{tabular}{lccccc}
    \toprule
    \multirow[c]{2}{*}{\textbf{Method}}
    & \multicolumn{5}{c}{\textbf{ePytope binary classification}} \\
    \cmidrule(lr){2-6}
     & AUC
     & AUPR
     & AUC0.1
     & BEDROC
     & logAUC \\
    \midrule

    \rowcolor{blue!6}
    \multicolumn{6}{l}{\textit{Contextualized backbone}} \\
    \midrule
    \methodcell{w/ \textsc{LogiMLM}\genmark}{continued MLM}
    & \stdcell{0.524}{0.134}
    & \stdcell{0.341}{0.123}
    & \stdcell{0.531}{0.045}
    & \stdcell{0.345}{0.181}
    & \stdcell{0.024}{0.022} \\

    \methodcell{w/ LatentCA-TCR}{pooled latent scoring}
    & \stdcell{0.558}{0.142}
    & \stdcell{0.352}{0.145}
    & \stdcell{0.522}{0.048}
    & \stdcell{0.342}{0.218}
    & \stdcell{0.031}{0.026} \\

    \rowcolor{blue!10}
    \methodcell{\textbf{w/ \textsc{LogiCA}-TCR\genmark}}{\textbf{TCR-anchored}}
    & \beststdcell{0.672}{0.158}
    & \beststdcell{0.485}{0.151}
    & \beststdcell{0.586}{0.083}
    & \beststdcell{0.541}{0.250}
    & \understdcell{0.080}{0.056} \\

    \rowcolor{blue!10}
    \methodcell{\textbf{w/ \textsc{LogiCA}-Pep\genmark}}{\textbf{peptide-anchored}}
    & \stdcell{0.586}{0.169}
    & \stdcell{0.389}{0.133}
    & \stdcell{0.540}{0.047}
    & \stdcell{0.416}{0.205}
    & \stdcell{0.037}{0.027} \\

    \rowcolor{blue!10}
    \methodcell{\textbf{w/ \textsc{LogiCA}-Dual\genmark}}{\textbf{dual-anchor}}
    & \understdcell{0.633}{0.189}
    & \understdcell{0.466}{0.154}
    & \understdcell{0.574}{0.070}
    & \understdcell{0.523}{0.283}
    & \beststdcell{0.093}{0.068} \\

    \midrule
    \rowcolor{gray!8}
    \multicolumn{6}{l}{\textcolor{gray!70!black}{\textit{External baselines}}} \\
    \midrule
    \methodcell{EPACT~\citep{zhang2024epitope}}{embedding contrastive}
    & \stdcell{0.524}{0.140}
    & \stdcell{0.309}{0.185}
    & \stdcell{0.498}{0.030}
    & \stdcell{0.245}{0.229}
    & \stdcell{0.019}{0.021} \\

    \methodcell{ERGO-II~\citep{springer2021ergoii}}{end-to-end classifier}
    & \stdcell{0.506}{0.089}
    & \stdcell{0.317}{0.191}
    & \stdcell{0.509}{0.031}
    & \stdcell{0.311}{0.272}
    & \stdcell{0.020}{0.020} \\

    \methodcell{ImRex~\citep{moris2021imrex}}{end-to-end classifier}
    & \stdcell{0.563}{0.135}
    & \stdcell{0.335}{0.200}
    & \stdcell{0.507}{0.050}
    & \stdcell{0.272}{0.250}
    & \stdcell{0.019}{0.027} \\

    \methodcell{iTCep~\citep{zhang2023itcep}}{end-to-end classifier}
    & \stdcell{0.612}{0.148}
    & \stdcell{0.392}{0.194}
    & \stdcell{0.543}{0.046}
    & \stdcell{0.428}{0.227}
    & \stdcell{0.041}{0.037} \\

    \methodcell{TCR-T5\genmark~\citep{karthikeyan2025tcrt5}}{generative MLM}
    & \stdcell{0.502}{0.121}
    & \stdcell{0.312}{0.180}
    & \stdcell{0.507}{0.032}
    & \stdcell{0.302}{0.216}
    & \stdcell{0.024}{0.022} \\

    \methodcell{TEIM~\citep{peng2023teim}}{end-to-end classifier}
    & \stdcell{0.554}{0.092}
    & \stdcell{0.332}{0.177}
    & \stdcell{0.516}{0.037}
    & \stdcell{0.329}{0.265}
    & \stdcell{0.031}{0.025} \\

    \methodcell{TITAN~\citep{weber2021titan}}{end-to-end classifier}
    & \stdcell{0.567}{0.102}
    & \stdcell{0.327}{0.162}
    & \stdcell{0.502}{0.028}
    & \stdcell{0.279}{0.214}
    & \stdcell{0.017}{0.019} \\

    \methodcell{TULIP-TCR\genmark~\citep{meynard2024tulip}}{generative MLM}
    & \stdcell{0.591}{0.113}
    & \stdcell{0.403}{0.138}
    & \stdcell{0.551}{0.062}
    & \stdcell{0.431}{0.227}
    & \stdcell{0.056}{0.043} \\

    \bottomrule
    \end{tabular}%
    }
\end{table*}
\renewcommand{\arraystretch}{1.0}
\setlength{\tabcolsep}{6pt}

\paragraph{Peptide encoder scaling.}

\label{app:tcr-esm-scaling}

Table~\ref{tab:tcr-esm-scaling} evaluates the effect of ESM-2 peptide encoder size for \textsc{LogiCA}-TCR. All variants use the same TCRLang TCR encoder, pretraining objective, and fine-tuning recipe, and differ only in the ESM-2 peptide encoder size. The 35M-parameter encoder performs best across all three peptide-mutation benchmarks. Larger encoders do not improve performance: the 150M model is competitive but lower than 35M, while the 650M model drops further across benchmarks. This suggests that ESM-2 35M is best matched to the scale of our TCR--peptide data, so we use it as the peptide encoder in the main experiments.

\begin{table*}[!htbp]
    \centering
    \caption{Experiments with different ESM-2 peptide encoder sizes in \textsc{LogiCA}-TCR. All variants use TCRLang as the TCR encoder and differ only in the size of the ESM-2 peptide encoder. Performance is reported as mean Pearson and Spearman correlations with standard deviation across runs on ePytope, BATCAVE, and ATLAS-PEP peptide-mutation benchmarks.}
    \label{tab:tcr-esm-scaling}
    \small
    \renewcommand{\arraystretch}{1.08}
    \resizebox{1\textwidth}{!}{%
    \begin{tabular}{lcccccc}
    \toprule
    \multirow[c]{2}{*}{\textbf{Method}} 
    & \multicolumn{2}{c}{\textbf{ePytope~\citep{drost2025benchmarking}}} 
    & \multicolumn{2}{c}{\textbf{BATCAVE~\citep{banerjee2025comprehensive}}} 
    & \multicolumn{2}{c}{\textbf{ATLAS-PEP~\citep{borrman2017atlas}}} \\
    \cmidrule(lr){2-3} \cmidrule(lr){4-5} \cmidrule(lr){6-7}
    & \multicolumn{1}{c}{Pearson} 
    & \multicolumn{1}{c}{Spearman} 
    & \multicolumn{1}{c}{Pearson} 
    & \multicolumn{1}{c}{Spearman} 
    & \multicolumn{1}{c}{Pearson} 
    & \multicolumn{1}{c}{Spearman} \\
    \midrule
    \rowcolor{blue!6}
    \multicolumn{7}{l}{\textit{Contextualized backbone}} \\
    \midrule
    \methodcell{\textsc{LogiCA}-TCR}{ESM-2 8M peptide encoder}
    & \stdcell{0.098}{0.237} & \stdcell{0.063}{0.230} 
    & \stdcell{0.109}{0.229} & \stdcell{0.074}{0.225} 
    & \stdcell{0.286}{0.683} & \stdcell{0.261}{0.561} \\
    
    \rowcolor{blue!10}
    \methodcell{\textbf{\textsc{LogiCA}-TCR}}{\textbf{ESM-2 35M peptide encoder}}
    & \beststdcell{0.296}{0.219} & \beststdcell{0.229}{0.168} 
    & \beststdcell{0.223}{0.229} & \beststdcell{0.170}{0.191} 
    & \beststdcell{0.632}{0.286} & \beststdcell{0.731}{0.238} \\

    \methodcell{\textsc{LogiCA}-TCR}{ESM-2 150M peptide encoder}
    & \stdcell{0.209}{0.211} & \stdcell{0.194}{0.193} 
    & \stdcell{0.165}{0.199} & \stdcell{0.144}{0.192} 
    & \stdcell{0.527}{0.260} & \stdcell{0.609}{0.272} \\

    \methodcell{\textsc{LogiCA}-TCR}{ESM-2 650M peptide encoder}
    & \stdcell{0.078}{0.207} & \stdcell{0.092}{0.213} 
    & \stdcell{0.064}{0.185} & \stdcell{0.066}{0.192} 
    & \stdcell{0.471}{0.395} & \stdcell{0.533}{0.381} \\
    \bottomrule
    \end{tabular}%
    }
\end{table*}
\renewcommand{\arraystretch}{1.0}


\end{document}